\documentclass{article}

\PassOptionsToPackage{numbers, compress}{natbib}
\usepackage[main,final]{neurips_2026}

\usepackage[utf8]{inputenc}
\usepackage[T1]{fontenc}
\usepackage{hyperref}
\usepackage{url}
\usepackage{booktabs}
\usepackage{amsfonts}
\usepackage{amsmath}
\usepackage{amssymb}
\usepackage{amsthm}
\usepackage{nicefrac}
\usepackage{microtype}
\usepackage{xcolor}
\usepackage{graphicx}
\usepackage{subcaption}
\usepackage{adjustbox}
\usepackage{mathtools}

\usepackage{titlesec}

\titlespacing*{\section}
  {0pt}{*0.4}{*0.4}  
\titlespacing*{\subsection}
  {0pt}{*0.3}{*0.3}

\usepackage{bm}
\usepackage{multirow}
\usepackage{colortbl}
\usepackage{float}
\usepackage{wrapfig}

\newtheorem{theorem}{Theorem}
\newtheorem{proposition}[theorem]{Proposition}
\newtheorem{lemma}[theorem]{Lemma}
\newtheorem{corollary}[theorem]{Corollary}
\newtheorem{definition}{Definition}
\newtheorem{remark}[theorem]{Remark}

\newcommand{\R}{\mathbb{R}}
\newcommand{\Z}{\mathbb{Z}}
\newcommand{\En}{\mathrm{E}(n)}
\newcommand{\On}{\mathrm{O}(n)}
\newcommand{\SOn}{\mathrm{SO}(n)}
\newcommand{\OC}{\mathrm{O}(C)}
\newcommand{\SOC}{\mathrm{SO}(C)}
\newcommand{\Tn}{\mathrm{T}(n)}
\newcommand{\spd}{\mathrm{SPD}}
\newcommand{\tr}{\mathrm{tr}}
\newcommand{\diag}{\mathrm{diag}}
\DeclareMathOperator{\softplus}{softplus}

\title{Learning Discrete Riemannian Metrics for Physical Fields with Cochain-Frame Equivariance}

\author{%
  Dongzhe Zheng$^{1}$,
  Christine Allen-Blanchette$^{1}$ \\
  $^1$Princeton University\\
  \texttt{\{dz5992, ca15\}@princeton.edu}
}

\begin{document}

\maketitle


\begin{abstract}
Physical fields on meshes require a separation between topology and geometry: differential identities should remain exact, while geometry, material response, and anisotropic coupling are learned from data. Existing neural surrogates often mix these roles inside unconstrained message passing. We introduce Riemannian Hodge Message Passing (RHMP), which turns this separation into an architectural principle. RHMP fixes the cellular coboundaries ($d_k$) determined by oriented incidence and learns symmetric positive-definite cochain metrics ($H_k$) for geometry-dependent propagation. The learned metric defines how cochain fields are measured and coupled; its invariance to orthogonal changes of the hidden feature basis yields cochain-frame-equivariant propagation. RHMP implements this principle with metric-weighted Hodge blocks ($d_k^\top H_{k+1}d_k$), yielding exact cochain-complex identities ($d_{k+1}d_k=0$), nonnegative Hodge energies, positive-semidefinite operators, and exact Abelian curvature invariance. Across seven physical benchmarks spanning fluids, electromagnetism, gauge fields, and variable-mesh CFD, RHMP performs strongly across these systems, with the largest gains when topology, learned geometry, and field structure interact.
\end{abstract}
\section{Introduction}
\label{sec:intro}

Physical fields on meshes do not all live naturally as node features. Pressure is naturally vertex- or cell-centered, flux and circulation live on
oriented edges, and field strengths such as vorticity or curvature live on
faces or higher-dimensional cells. Learning such fields therefore requires more than expressive message passing: the model must preserve topological identities that are exact, while adapting to geometric and material responses that vary across domains. Incidence relations between vertices, edges, faces, and higher cells impose exact physical identities, such as the curl-free condition of gradient fields ($\nabla\!\times\!\nabla\phi=0$) and the source-free (or divergence-free) condition of curl fields ($\nabla\!\cdot\!\nabla\!\times\!\mathbf{A}=0$). By contrast, material parameters, anisotropy, curvature, and constitutive laws determine how strongly neighboring quantities interact. This suggests a simple design principle for neural physical surrogates: keep topology fixed and learn geometry.

Hodge theory provides precisely this separation. On a cell complex, scalar quantities are represented as $0$-cochains, fluxes and circulations as $1$-cochains, and surface densities such as vorticity or field strength as $2$-cochains. The coboundary maps $d_k$ move between these spaces and are determined entirely by oriented incidence. They contain no material parameters and satisfy $d_{k+1}d_k=0$ by construction; we write this cochain-complex identity compactly as $d^2=0$. Geometry enters through a positive-definite inner product on each cochain space. We write this inner product as a symmetric positive-definite (SPD) matrix $H_k$ and interpret it as a \emph{discrete Riemannian metric}, the discrete counterpart of the metric-induced Hodge star in $\langle\omega,\eta\rangle_g=\int_M\omega\wedge\star_g\eta$~\cite{desbrun2005discrete,arnold2006finite}. For RHMP's Hodge message operator, learning $H_k$ shapes the energies and spectrum of non-harmonic modes while preserving the harmonic subspace of the fixed cochain complex. For readers less familiar with this dictionary, Appendix~\ref{app:metric_primer} gives a concrete triangle example and the continuous-to-discrete correspondence.

This viewpoint leads to the RHMP principle: fix the topological coboundary map $d_k$, and learn the cochain metrics $H_k$. For an input or hidden $k$-cochain feature $x_k$, the weighted Hodge energy $\|d_k x_k\|_{H_{k+1}}^2$ induces the operator $d_k^\top H_{k+1}d_k$ on $k$-cochains (see Appendix~\ref{app:cochain-dataflow}). The coboundary \(d_k\) determines which quantities are differentiated and which conservation laws remain closed. 
The metric \(H_{k+1}\) determines how the resulting \((k+1)\)-cochain is measured and coupled. Thus, RHMP keeps the cochain complex exact while
concentrating learnable geometric freedom in the metric.

\begin{wrapfigure}{r}{0.49\linewidth}
\vspace{-0.5em}
\centering
\includegraphics[width=\linewidth]{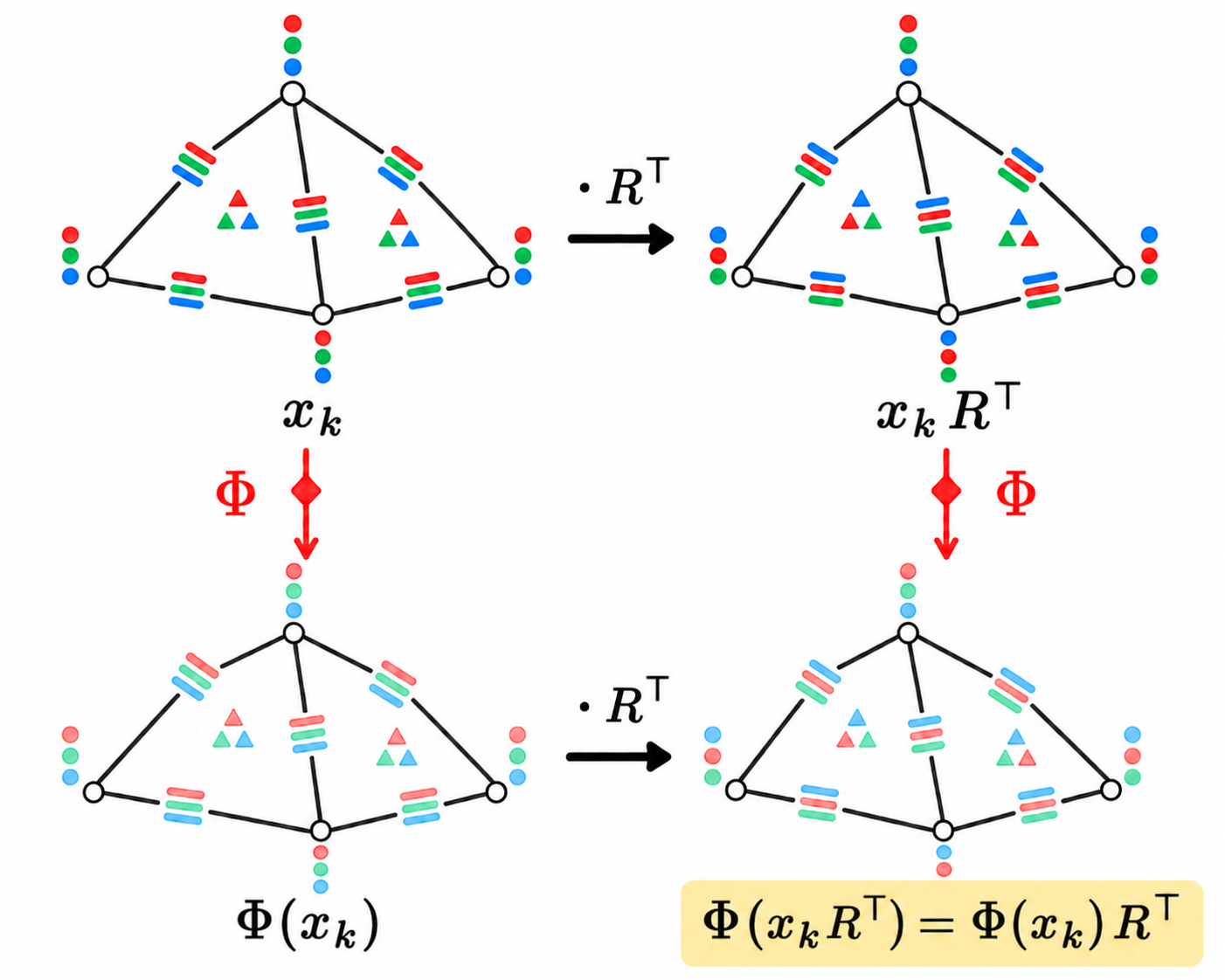}
\caption{\textbf{Metric invariance and cochain-frame equivariance.} Applying an orthogonal change of hidden cochain frame $R \in O(C)$ before or after an RHMP layer $\Phi$ gives the same result. The frame action is on feature channels, while $d_k$ and $H_k$ act on cell indices.}
\label{fig:gauge_equi_overview}
\vspace{-0.7em}
\end{wrapfigure}
Treating $H_k$ as the learned geometric object also gives a natural symmetry requirement. The metric describes propagation independently of the orthonormal basis used for the hidden feature channels. 
In RHMP, the $C$ channels of a hidden
\(k\)-cochain are coordinates in an auxiliary feature space with a standard inner product. Changing this orthonormal
frame by \(R\in O(C)\) acts as
\(\mathbf{x}_k\mapsto \mathbf{x}_k R^\top\) for
\(\mathbf{x}_k\in\mathbb{R}^{n_k\times C}\).
We call a layer \emph{cochain-frame equivariant} if it commutes with this action. RHMP enforces this by predicting the cell-space metric $H_k$ only from $\OC$-invariant channel-statistics, such as norms and inner products, and by
using norm-gated nonlinearities. The resulting propagation is therefore consistent across orthonormal
choices of the hidden cochain frame.
Together with the SPD constraint on $H_k$, Euclidean-invariant geometric
features, and fixed coboundaries, this yields nonnegative Hodge energies, positive-semidefinite Hodge operators, coordinate-free scalar behavior, $\En$-equivariant vector readouts, and exact preservation of $d^2=0$.

RHMP instantiates these principles by assigning features to the appropriate cochain degree, freezing the coboundary operators, predicting $H_k\succ0$ from invariant statistics, and using norm-gated nonlinearities that preserve cochain-frame equivariance. The result is a message-passing operator whose topology and symmetries are built in, with geometry learned through the cochain metric (\href{https://github.com/CAB-Lab-Princeton/Riemannian-Hodge-Message-Passing}{code}).

\textbf{Contributions.}
(1) \emph{Metric-centered formulation.}
We formulate the problem of learning a mesh-based physical field as a problem of learning a discrete Riemannian metric $H_k$ on cochains, with fixed coboundary maps $d_k$. This makes the separation between topology, geometry, and conservation explicit.
(2) \emph{Metric guarantees and constructive architecture.}
We prove that the metric conditions $H_k\succ0$, cochain-frame invariance under $\OC$, and spatial $\En$-invariance yield nonnegative Hodge energies,
positive-semidefinite Hodge operators, preservation of the harmonic kernel with local control of the non-harmonic spectrum, cochain-frame-equivariant message passing, and invariant/equivariant scalar and vector readouts. RHMP implements these conditions directly.
(3) \emph{Exact field structure and empirical validation.} For Abelian gauge fields represented as cochains, curvature observables depending on $dA$ are exactly invariant under $A\mapsto A+d\lambda$, since the additional term $d(d\lambda)$ vanishes by the cochain-complex identity $d\circ d=0$. On seven tasks spanning fluid mechanics, electromagnetism, Abelian and non-Abelian gauge-field benchmarks, and per-sample variable-mesh CFD, RHMP outperforms graph, cell-complex, manifold-gauge, and neural-operator baselines. 


\section{Related Work}
\label{sec:related}

\textbf{Discrete exterior calculus and compatible discretization.}
RHMP is built on the topology--geometry separation that underlies discrete exterior calculus (DEC)~\cite{desbrun2005discrete} and finite element exterior calculus (FEEC)~\cite{arnold2006finite}: coboundaries encode the cochain complex and satisfy $d_{k+1}d_k=0$, while metric-dependent Hodge stars or mass matrices define inner products on discrete forms. This separation also appears in compatible and mimetic discretizations and computational electromagnetism, where incidence matrices encode exact sequences and conservation laws, while mass, Hodge, or constitutive matrices encode metric and material response~\cite{bossavit1998computational,hiptmair2002finite,bochev2006principles}. RHMP turns this separation into a neural architecture, with $d_k$ fixed and the SPD metric factor $H_k$ predicted from data.

\textbf{Cellular, simplicial, and sheaf neural networks.}
Topological signal processing extends graph filtering from node signals to edge-, face-, and higher-order cochains through Hodge Laplacians and Hodge decompositions~\cite{yang2022simplicial_filters}. Simplicial and cellular message-passing methods extend graph neural networks to higher-order cells via incidence maps, Hodge Laplacians, orientation-aware aggregation, and attention, including SNN~\cite{ebli2020simplicial}, MPSN~\cite{bodnar2021weisfeiler_sc}, SCCNN~\cite{yang2022sccnn}, CW Net~\cite{bodnar2021weisfeiler}, SAT~\cite{goh2022simplicial}, Clifford-SMPN~\cite{liu2024clifford}, and the unified framework of Hajij et al.~\cite{hajij2022topological}; sheaf neural networks~\cite{hansen2020sheaf} and neural sheaf diffusion~\cite{bodnar2022neural} equip graphs with learned cellular sheaves and sheaf Laplacians. RHMP keeps the cellular coboundary fixed and learns an SPD cochain metric $H_k$ from $\OC$-invariant statistics, preserving $d^2=0$ while learning geometry, anisotropy, and material response.

\textbf{Learned Hodge operators and mesh spectral geometry.}
HodgeNet~\cite{smirnov2021hodgenet}, HodgeFormer~\cite{nousias2025hodgeformer}, and HSD~\cite{zheng2026hsd} learn Hodge or Hodge-like matrices for mesh spectral or transformer operators. RHMP differs by learning input-dependent SPD cochain metrics inside physical message passing, with fixed coboundaries, nonnegative Hodge energies, and cochain-frame-invariant metric prediction.

\textbf{Gauge-equivariant and surface networks.}
Gauge Equivariant CNNs~\cite{cohen2019gauge} and GEM-CNN~\cite{dehaan2021gem} define convolutions on manifolds or meshes via local tangent-frame parallel transport; Tangent Bundle Networks~\cite{battiloro2023tangent} use connection-Laplacian tangent-bundle filters; DiffusionNet~\cite{sharp2022diffusionnet} learns through discretization-agnostic surface diffusion. Lattice-gauge models instead impose gauge structure on link variables, Wilson-loop features, or flow-based samplers~\cite{favoni2022lattice,luo2021gauge,kanwar2020equivariant}. RHMP predicts its metric from statistics invariant under orthogonal changes of the hidden cochain feature basis, yielding cochain-frame-equivariant Hodge messages. For Abelian connection data, RHMP inherits exact curvature invariance from $d^2=0$; for non-Abelian fields it uses cochain differentials with algebraic coupling.

\textbf{Equivariant graph networks and physical surrogates.}
$\En$-equivariant graph networks such as SchNet~\cite{schutt2017schnet}, EGNN~\cite{satorras2021en}, Tensor Field Networks~\cite{thomas2018tensor}, SE(3)-Transformer~\cite{fuchs2020se3}, PaiNN~\cite{schutt2021equivariant}, and NequIP~\cite{batzner2022e3} build coordinate-aware message passing from distances, relative positions, and higher-order geometric features. Graph Network-based Simulators~\cite{sanchezgonzalez2020learning}, MeshGraphNets~\cite{pfaff2021learning}, and message-passing neural PDE solvers~\cite{brandstetter2022message} learn surrogate dynamics on particles or meshes. These methods encode fields as node, particle, or edge features. RHMP assigns scalar-, flux-, and intensity-like quantities to 0-, 1-, and 2-cochains and uses fixed coboundaries with learned SPD metrics.

\textbf{Neural operators.}
Fourier Neural Operator~\cite{li2021fourier} and DeepONet~\cite{lu2021deeponet} are central neural operators, learning input--output maps at the function-space level. General neural-operator theory formulates such maps through discretization-consistent integral-kernel parameterizations~\cite{kovachki2023neural}. Beyond FNO and DeepONet, graph-kernel and multipole graph neural operators learn nonlocal kernels on irregular samples~\cite{li2020neural,li2020multipole}; Galerkin Transformer uses attention as an operator-learning layer~\cite{cao2021choose}; F-FNO and WNO modify the spectral factorization or basis~\cite{tran2023factorized,tripura2023wavelet}; PINO adds physics residual constraints~\cite{li2024physics}; and Geo-FNO / GINO extend FNO-style models to general or large-scale geometries~\cite{li2023fourier_deformations,li2023geometry}. RHMP is complementary: it builds cochain type, exact $d^2=0$, SPD Hodge energy, and cochain-frame-invariant metric learning into the architecture.

\section{Discrete Geometry on Cell Complexes}
\label{sec:prelim}

A regular CW complex \(K\) is a cell complex whose \(k\)-dimensional cells
generalize vertices, edges, faces, and higher-dimensional elements.  A $C$-channel
$k$-cochain is a matrix $\mathbf{x}_k\in\R^{n_k\times C}$, with one
$C$-dimensional feature vector per $k$-cell.  The coboundary operator
$d_k\in\R^{n_{k+1}\times n_k}$ maps $k$-cochains to $(k+1)$-cochains.
For example, $d_0$ differences vertex quantities along oriented edges and
$d_1$ sums oriented edge quantities around faces.  These are the standard cochain and coboundary operators of DEC and FEEC~\cite{desbrun2005discrete,arnold2006finite}.  The identity
$d_{k+1}d_k=0$ is purely combinatorial: it is the discrete form of
identities such as $\nabla\!\times\!\nabla\phi=0$ and
$\nabla\!\cdot\!\nabla\!\times\!\mathbf{A}=0$.
Appendix~\ref{app:triangle-example} gives a concrete triangle example illustrating $d_0$, $d_1$, and how changing $H_1$ while fixing $d_k$ changes the Hodge step.

\begin{definition}[Discrete Riemannian metric on cochains]
A discrete Riemannian metric on $k$-cochains is an SPD matrix $H_k\in\spd(n_k)$, where $\spd(n_k)$ denotes the cone of symmetric positive-definite $n_k\times n_k$ matrices, defining the inner product
\begin{equation}
  \langle \mathbf{x}_k,\mathbf{y}_k\rangle_{H_k}
  := \tr(\mathbf{x}_k^\top H_k\mathbf{y}_k).
  \label{eq:metric_inner_product}
\end{equation}
Diagonal entries of $H_k$ describe local cell weights such as volumes,
areas, conductivities, or permittivities. Off-diagonal entries describe
anisotropic or nonlocal couplings between cells.  Learning $H_k$ therefore
amounts to learning the discrete geometry through which the cochain field is
measured and propagated.
\end{definition}

\textbf{The metric block.}
The basic object is a metric-weighted differential.  Given an incidence map
$D:C^a(K)\to C^b(K)$ and a predicted positive-definite metric $H_b$ on the
target cochain space, held fixed in the following quadratic form, define
\begin{equation}
  \mathcal{E}_{D,H}(x)=\frac12\|Dx\|_{H_b}^2
  =\frac12\tr\big((Dx)^\top H_b(Dx)\big),
  \qquad
  \nabla_x \mathcal{E}_{D,H}=D^\top H_bD x.
  \label{eq:generic_metric_block}
\end{equation}
Throughout the paper, $d^\top H d$ denotes this metric block: $d$ supplies
the topological derivative and $H$ supplies the learned geometry used to
measure its output.

\textbf{Placement convention.}
Superscripts $\uparrow$, $\downarrow$, and $\mathrm{cross}$ only indicate which adjacent cochain space carries the metric; all such matrices use the same SPD construction.

\textbf{Upper and lower Hodge terms.}
For a $k$-cochain, the upper placement takes $D=d_k$ and measures the
$(k+1)$-cochain $d_k\mathbf{x}_k$ with $H_{k+1}^{\uparrow}$:
\begin{equation}
  \mathcal{E}_{k}^{\uparrow}(\mathbf{x}_k)
  = \frac12\|d_k\mathbf{x}_k\|_{H_{k+1}^{\uparrow}}^2
  = \frac12\tr\big((d_k\mathbf{x}_k)^\top H_{k+1}^{\uparrow}(d_k\mathbf{x}_k)\big),
  \label{eq:upper_energy}
\end{equation}
with gradient $d_k^\top H_{k+1}^{\uparrow}d_k\mathbf{x}_k$.  The lower
placement takes $D=d_{k-1}^\top$ and measures the
$(k-1)$-cochain $d_{k-1}^\top\mathbf{x}_k$ with
$H_{k-1}^{\downarrow}$:
\begin{equation}
  \mathcal{E}_{k}^{\downarrow}(\mathbf{x}_k)
  = \frac12\|d_{k-1}^\top\mathbf{x}_k\|_{H_{k-1}^{\downarrow}}^2,
  \qquad
  \nabla_{\mathbf{x}_k} \mathcal{E}_{k}^{\downarrow}
  = d_{k-1}H_{k-1}^{\downarrow}d_{k-1}^\top\mathbf{x}_k.
  \label{eq:lower_energy}
\end{equation}
The weighted $k$-Hodge message operator is therefore
\begin{equation}
  L_k^H
  = d_{k-1}H_{k-1}^{\downarrow}d_{k-1}^\top
  + d_k^\top H_{k+1}^{\uparrow}d_k.
  \label{eq:hodge_lap}
\end{equation}

Positive definiteness of the placed metrics gives positive semidefiniteness
of the Hodge operator.  For any $\mathbf{x}_k$,
\begin{equation}
  \langle \mathbf{x}_k,L_k^H\mathbf{x}_k\rangle
  = \|d_{k-1}^\top\mathbf{x}_k\|_{H_{k-1}^{\downarrow}}^2
  + \|d_k\mathbf{x}_k\|_{H_{k+1}^{\uparrow}}^2
  \ge 0.
  \label{eq:psd_energy}
\end{equation}
Thus $H_k\succ0$ certifies the learned geometry as an energy metric.

The connection to continuous geometry is direct.  A Riemannian metric $g$ on
a smooth manifold induces the Hodge star $\star_g$ and the $L^2$ inner
product on $k$-forms,
$\langle \omega,\eta\rangle_g=\int_M\omega\wedge\star_g\eta$.
After discretizing $k$-forms by their integrals over oriented $k$-cells, the
matrix representing this inner product is a mass/Hodge matrix
$H_k$~\cite{desbrun2005discrete,arnold2006finite}. RHMP learns a
positive-definite inner product at each cochain degree, with geometric Hodge
matrices and constitutive mass matrices as concrete realizations. The
coboundary $d_k$ remains the discrete counterpart of the exterior derivative.

\textbf{Cochain-frame action.}
For $\mathbf{x}_k\in\R^{n_k\times C}$, a cochain-frame change $R\in\OC$ acts on the hidden feature fiber by $\rho_R(\mathbf{x}_k)=\mathbf{x}_k R^\top$, with the same $R$ applied to all cells and cochain degrees: for $\mathbf{x}=(\mathbf{x}_0,\mathbf{x}_1,\mathbf{x}_2)$ the action is $\rho_R(\mathbf{x})=(\mathbf{x}_0 R^\top,\mathbf{x}_1 R^\top,\mathbf{x}_2 R^\top)$. A layer $\Phi$ is \emph{cochain-frame equivariant} if $\Phi(\rho_R\mathbf{x})=\rho_R\Phi(\mathbf{x})$ for all $R\in\OC$.

The following theorem collects the structural properties of the learned
metric and fixed cochain complex.

\begin{theorem}[Topology, positivity, and symmetry of metric-based propagation]
\label{thm:equivariance}
Let $K$ be a CW complex embedded in $\R^n$, let $\mathbf{x}_k\in\R^{n_k\times C}$ be $C$-channel cochain features, and let $\Phi$ be a layer built from the fixed coboundaries $d_k$, metric-weighted operators of the form \eqref{eq:hodge_lap}, residual sums, and per-cell nonlinearities of the form $\sigma_{\rm gate}(\mathbf{m})=h(\|\mathbf{m}\|)\mathbf{m}$. Assume:
\begin{enumerate}
  \item each predicted metric satisfies $H_j\succ0$;
  \item $H_j(\rho_R\mathbf{x},K)=H_j(\mathbf{x},K)$ for all cochain-frame changes $R\in\OC$;
  \item $H_j$ and the lifting encoder depend on the embedding of $K$ only through $\En$-invariants such as distances, angles, areas, and incidence relations.
\end{enumerate}
Then the following hold.
\textnormal{(i)} $d_{k+1}d_k=0$ is exact and independent of training.
\textnormal{(ii)} $L_k^H$ is positive semidefinite for every input.
\textnormal{(iii)} $\Phi$ is cochain-frame equivariant under $\OC$: $\Phi(\rho_R\mathbf{x})=\rho_R\Phi(\mathbf{x})$.
\textnormal{(iv)} Scalar cochain readouts are $\En$-invariant; vector readouts of the form $\sum_e w_e(\cdot)(\mathbf{r}_j-\mathbf{r}_i)$, with invariant weights $w_e$, are $\En$-equivariant.
\end{theorem}

\begin{proof}[Proof sketch]
(i) is the cellular identity $d^2=0$. (ii) follows from the quadratic form identity \eqref{eq:psd_energy}. For (iii), $d_k$ and $H_k$ act on the cell index, while $R\in\OC$ acts on the channel index; these actions commute. The metric itself is unchanged by assumption, and the norm gate satisfies $\sigma_{\rm gate}(\mathbf{m}R^\top)=\sigma_{\rm gate}(\mathbf{m})R^\top$. For (iv), all scalar quantities entering the layer are functions of Euclidean invariants. Multiplying an invariant scalar by the displacement $\mathbf{r}_j-\mathbf{r}_i$, which transforms equivariantly under $\En$, produces an equivariant vector. The full proof is given in Appendix~\ref{app:proofs}.
\end{proof}

\textbf{Abelian gauge-field observables.}
For Abelian gauge-field tasks, the fixed cochain complex provides an exact invariance at the level of physical observables. Let a $U(1)$ connection be represented additively as a 1-cochain $A$, and define curvature by $F=d_1 A$. For any 0-cochain $\lambda$, the gauge shift $A\mapsto A+d_0\lambda$ leaves curvature unchanged:
\[
  F'=d_1(A+d_0\lambda)=d_1 A+d_1 d_0\lambda=d_1 A,
\]
because $d_1 d_0=0$. Curvature-type observables depending on $A$ through $dA$ are therefore exactly invariant under Abelian gauge shifts; for compact $U(1)$ phases the equality is understood modulo $2\pi$.

\section{Riemannian Hodge Message Passing}
\label{sec:method}

RHMP implements Theorem~\ref{thm:equivariance} directly.  The
architecture lifts input fields to $0$-, $1$-, and $2$-cochains,
then applies metric-weighted Hodge message passing, and finally reads out
scalar or vector quantities according to the task. The learned geometric object inside each Hodge block is the
discrete cochain metric \(H_k\); coboundaries \(d_k\) remain fixed by the
oriented cell complex.
Appendix~\ref{app:cochain-dataflow} gives a schematic
of the cochain-level data flow, including the chain maps, Hodge return loops,
and matrix expressions.


\subsection{Weighted Hodge Message-Passing Layer}
\label{sec:hodge_mp}

Physical quantities are assigned to cochain degrees according to geometric
type: $0$-cochains carry scalar fields such as pressure or temperature,
$1$-cochains carry flux-like fields such as velocity components or gauge
connections, and $2$-cochains carry intensity-like fields such as vorticity
or field strength.  The coboundaries $d_k$ are fixed once the oriented cell
complex is fixed, so learning leaves $d^2=0$ intact.

For a $k$-cochain feature $\mathbf{x}_k^{(\ell)}$, RHMP forms a
same-degree metric Hodge message
\begin{equation}
  \mathbf{m}_{k,\mathrm{self}}^{(\ell)}
  = d_{k-1}H_{k-1}^{\downarrow,(\ell)}d_{k-1}^\top\mathbf{x}_k^{(\ell)}
  + d_k^\top H_{k+1}^{\uparrow,(\ell)}d_k\mathbf{x}_k^{(\ell)}.
  \label{eq:self_hodge_message}
\end{equation}
It also exchanges information across adjacent cochain degrees through
metric-aware transport,
\begin{equation}
  \mathbf{m}_{k,\mathrm{cross}}^{(\ell)}
  = d_{k-1}H_{k-1}^{\mathrm{cross},(\ell)}\mathbf{x}_{k-1}^{(\ell)}
  + d_k^\top H_{k+1}^{\mathrm{cross},(\ell)}\mathbf{x}_{k+1}^{(\ell)},
  \label{eq:cross_hodge_message}
\end{equation}
with boundary terms omitted when the adjacent degree is absent.  The same
invariant metric predictor and SPD constraints are used for
\(H^{\mathrm{cross}}\); the superscript records its use in cross-degree
transport.
The layer is
\begin{equation}
  \begin{aligned}
  \mathbf{m}_k^{(\ell)}
  &= \alpha\,\mathbf{m}_{k,\mathrm{self}}^{(\ell)}
      +(1-\alpha)\,\mathbf{m}_{k,\mathrm{cross}}^{(\ell)}, \\
  \mathbf{x}_k^{(\ell+1)}
  &= \mathbf{x}_k^{(\ell)}
      + g_\theta(\|\mathbf{m}_k^{(\ell)}\|)\,
        \mathrm{RMSNorm}(\mathbf{m}_k^{(\ell)}).
  \end{aligned}
  \label{eq:hodge_mp}
\end{equation}
where $\alpha\in(0,1)$ is learned and $g_\theta$ is a scalar gate
evaluated on the norm of each cell's message. RMSNorm is applied per cell
without learnable affine parameters; both operations preserve channel-basis symmetry.

Eq.~\eqref{eq:hodge_mp} is the architectural form of the paper's topology--geometry
factorization.  The maps $d_k$ determine which cells communicate and preserve
the cochain complex.  The learned SPD metrics $H_k$ determine the geometry of that
communication. The same metrics define two global operators used as first layers
where the physics asks for them: the metric Hodge projection, which splits a $1$-form
input into its gradient and $H_1$-co-closed parts, and the physical solve, which inverts
$d_0^\top H_1d_0$ for a potential driven by a source; both are evaluated by differentiable
linear solves (Appendix~\ref{app:structured-readouts}).

\subsection{Constructing Positive-Definite Cochain Metrics}
\label{sec:symmetry_construction}

The metric predictor receives only statistics that are invariant under a
change of channel basis.  For each $k$-cell we use the self-norm and
neighbor inner products,
\begin{equation}
  \psi(\mathbf{x}_k)_i
  = \Big[\|\mathbf{x}_k(i)\|_2^2,\;
  \mathrm{Agg}_{j\sim i}\,\langle \mathbf{x}_k(i),\mathbf{x}_k(j)\rangle,\;
  \text{cell geometry invariants}\Big],
  \label{eq:metric_features}
\end{equation}
where the aggregation is permutation-invariant over neighboring cells.  Since
norms and inner products are unchanged by $\mathbf{x}\mapsto\mathbf{x}R^\top$,
any deterministic metric predictor built from $\psi$ is cochain-frame invariant.

The simplest instantiation is the diagonal metric
\begin{equation}
  H_k
  = \diag\!\left(\softplus(\mathrm{MLP}_H(\psi(\mathbf{x}_k)))+\epsilon\right),
  \qquad \epsilon>0.
  \label{eq:diagonal_metric}
\end{equation}
This gives a strict positive-definiteness guarantee:
for every nonzero $z\in\R^{n_k}$,
\begin{equation}
  z^\top H_k z
  = \sum_i \big(\softplus(h_i)+\epsilon\big)z_i^2
  \ge \epsilon\|z\|_2^2>0.
  \label{eq:spd_guarantee}
\end{equation}
Thus the learned $H_k$ is SPD at every layer and every training step, and Eq.~\eqref{eq:psd_energy} makes the corresponding Hodge operator positive semidefinite.
The shared $\mathrm{MLP}_H$ applies the same learned rule to local cell statistics on each mesh; its weights are independent of $n_k$, while its outputs scale with the number of cells (Appendix~\ref{app:complexity}).
The diagonal form in Eq.~\eqref{eq:diagonal_metric} gives the simplest SPD
certificate. The benchmark models use bounded diagonal corrections to reference Hodge
metrics retaining known cell geometry. Locally assembled tensor metrics encode
anisotropic material response through the same energy construction
(Appendix~\ref{app:constitutive-metrics}). Appendix~\ref{app:metric-comparison}
compares diagonal, diagonal-plus-low-rank, and full Cholesky parameterizations
under matched conditions.

Spatial behavior is handled by restricting geometric inputs to quantities
invariant under Euclidean motions:
edge lengths, vertex degrees, face
areas, and interior angles.  
Edge orientation is handled by multiplying an odd linear feature by an
even gate. For an oriented edge \(e_{ij}\), a \(1\)-cochain lifting takes the form
\begin{equation}
  \mathbf{x}_1(e_{ij})
  = W(\mathbf{x}_0^j-\mathbf{x}_0^i)\odot
    \bigl[1+\phi_{\rm even}(\mathbf{x}_0^i+\mathbf{x}_0^j,\ell_e,
      \deg_i+\deg_j)\bigr],
  \label{eq:edge_lift}
\end{equation}
where \(\ell_e=\|\mathbf{r}_j-\mathbf{r}_i\|\), $W$ is linear, and the gate
is unchanged when the endpoints are exchanged. 
Higher-degree cochains are constructed by
oriented aggregation over boundary cells together with invariant face
features.
For tasks requiring ambient vector outputs, RHMP reconstructs vectors using an
invariant-scalar-times-displacement readout:
\begin{equation}
  \hat{\mathbf{v}}(v)
  = \frac{1}{\deg(v)}\sum_{e=(i,j)\ni v}
  w(\mathbf{x}_1(e),\|\mathbf{r}_{ij}\|^2,\ldots)
  (\mathbf{r}_j-\mathbf{r}_i).
  \label{eq:vector_reconstruction}
\end{equation}
Here \(w\) is an \(\En\)-invariant scalar and
\(\mathbf{r}_j-\mathbf{r}_i\) is translation-invariant and
\(\On\)-equivariant. Therefore each term, and hence the sum, transforms
equivariantly under \(\En\)~\cite{villar2021scalars}. Gradient readouts
$\hat e=d_0\hat\phi$ satisfy $d_1\hat e=0$ exactly.

\subsection{Numerical Verification, Symmetry Hierarchy, and Spectral Expressivity}
\label{sec:method_verify}
\begin{table}[t]
\centering
\small
\caption{\textbf{Structural verification} (satisfied constraints out of $10$, grouped by family). A test is satisfied if its relative $L_2$ error is below $10^{-5}$ or holds exactly by construction; numeric and \textsc{n/a} entries count as unsatisfied. These tests identify which constraints each architecture enforces, though some baselines do not target all families. See Appendix~\ref{app:symmetry_detail} for per-subgroup errors.}
\label{tab:sym_pass_count}
\setlength{\tabcolsep}{3pt}
\adjustbox{max width=\linewidth}{%
\begin{tabular}{l|cccccccccccc|c}
\toprule
Constraint group
 & GCN & GAT & SchNet & EGNN & MPSN & SCCNN
 & GaugeEqCNN & GEM-CNN & CW~Net & Cliff-SMPN
 & FNO & DeepONet & \textbf{RHMP} \\
\midrule
Spatial $\En$ (3)        &  3 &  3 &  3 &  3 &  3 &  3 &  1 &  1 &  3 &  3 &  3 &  0 & \textbf{3} \\
Topology (2)             &  1 &  1 &  1 &  1 &  1 &  2 &  0 &  0 &  1 &  1 &  1 &  0 & \textbf{2} \\
Cochain-frame $\OC$ (4)  &  0 &  0 &  0 &  0 &  0 &  0 &  0 &  0 &  0 &  0 &  0 &  0 & \textbf{4} \\
Abelian $U(1)$ (1)       &  0 &  0 &  0 &  0 &  1 &  1 &  0 &  0 &  1 &  1 &  0 &  0 & \textbf{1} \\
\midrule
Total (10)               &  4 &  4 &  4 &  4 &  5 &  6 &  1 &  1 &  5 &  5 &  4 &  0 & \textbf{10} \\
\bottomrule
\end{tabular}
}
\vspace{-12pt}
\end{table}

We probe ten structural constraints, grouped into four families: spatial $\En$ subgroups (3 tests), topology checks (2; including $d^2\!=\!0$), cochain-frame $\OC$ subgroups (4; including channel rotation, permutation, and sign flips), and Abelian $U(1)$ curvature invariance under exact shifts (1). Table~\ref{tab:sym_pass_count} reports per-group pass counts: RHMP satisfies all ten, including the cochain-frame tests; other architectures satisfy complementary subsets according to their design. The full protocol, mesh and seed details, and per-subgroup numerical errors are in Appendix~\ref{app:symmetry_detail}.

\textbf{Metric geometry and harmonic structure.}
For the coefficient-space operator $L_k(A,B)$ in
Eq.~\eqref{eq:expressivity_operator}, positive-definite metrics preserve
\[
  \ker L_k(A,B)=\ker d_{k-1}^\top\cap\ker d_k.
\]
At $(A,B)=(I,I)$, every simple positive eigenvalue admits a nonzero
first-order response to a diagonal metric perturbation.
Proposition~\ref{prop:expressivity} quantifies this local spectral response;
Corollary~\ref{cor:harmonic} proves the fixed harmonic kernel.
Scalar calibration and anisotropic response give complementary empirical
characterizations of the metric (Appendix~\ref{app:learned-metric}).

\section{Experiments}
\label{sec:experiments}
\begin{figure}[b]
\vspace{-1.3em}
    \centering
        \includegraphics[width=.95\textwidth]{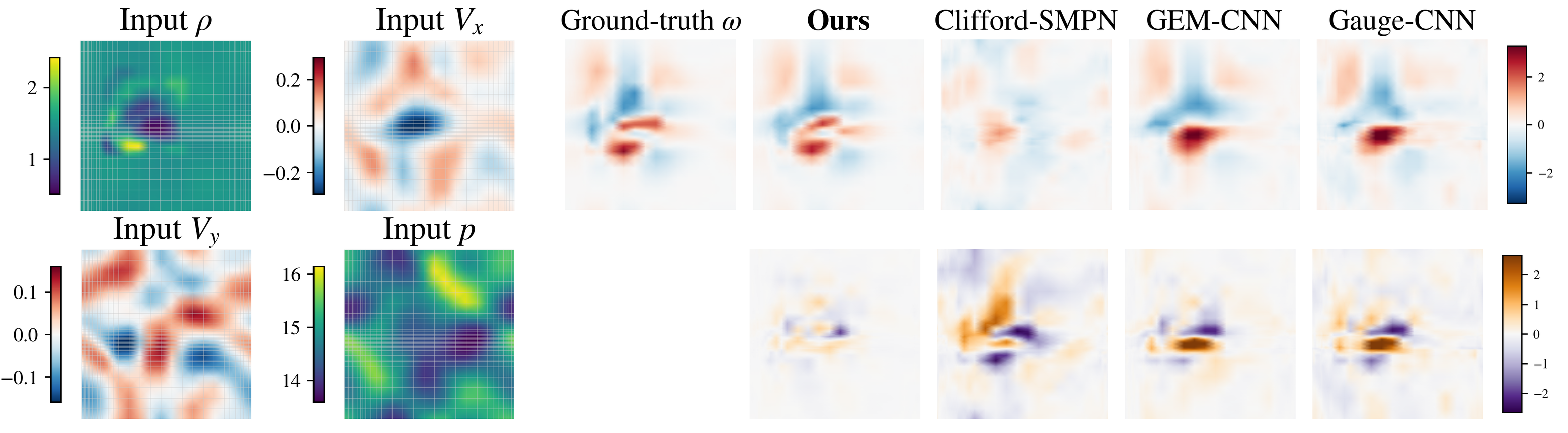}
    \caption{\textbf{Navier--Stokes vorticity.}
Left blocks summarize the inputs; right blocks compare ground truth, RHMP, and representative baselines. The bottom row shows the error maps.}
\vspace{-1.3em}
\label{fig:main_ns}
\end{figure}

\begin{table}[!t]
\centering
\caption{Quantitative results on the six fixed-mesh tasks, three metrics, and 11 baselines whose native definition admits irregular meshes.
SSIM ($\uparrow$), Pearson ($\uparrow$), NRMSE ($\downarrow$).
\textbf{Bold} = best, \underline{underline} = second best among node-target methods.
All entries use seed $42$ and the first 100 test samples. On Maxwell, RHMP uses the physical solve layer (Appendix~\ref{app:structured-readouts}). MPSN, SCCNN, and Clifford-SMPN use
edge-averaged targets (Appendix~\ref{app:tasks}).
Comparison against FNO on the structured-grid setting (NS Vorticity) is discussed inline in §\ref{sec:main_results}; variable-mesh generalization (AirfRANS) is reported separately in Table~\ref{tab:airfoil}.
Baselines grouped by family: \colorbox{blue!6}{Graph}, \colorbox{orange!8}{Topological}, \colorbox{green!8}{Geometric}, \colorbox{purple!8}{Operator}.}
\label{tab:main}
\resizebox{\linewidth}{!}{%
\begin{tabular}{ll cccc cc cccc c c}
\toprule
& & \multicolumn{4}{c}{\cellcolor{blue!6} Graph}
& \multicolumn{2}{c}{\cellcolor{orange!8} Topological}
& \multicolumn{4}{c}{\cellcolor{green!8} Geometric}
& \multicolumn{1}{c}{\cellcolor{purple!8} Operator}
& \\
\cmidrule(lr){3-6} \cmidrule(lr){7-8} \cmidrule(lr){9-12} \cmidrule(lr){13-13}
Task & Metric
& \cellcolor{blue!6}GCN & \cellcolor{blue!6}GAT & \cellcolor{blue!6}SchNet & \cellcolor{blue!6}EGNN
& \cellcolor{orange!8}MPSN & \cellcolor{orange!8}SCCNN
& \cellcolor{green!8}GaugeEqCNN & \cellcolor{green!8}GEM-CNN & \cellcolor{green!8}CW Net & \cellcolor{green!8}Cliff-SMPN
& \cellcolor{purple!8}DeepONet
& \textbf{RHMP} \\
\midrule
\multirow{3}{*}{NS Vort}
 & SSIM & \cellcolor{blue!6}.617 & \cellcolor{blue!6}.591 & \cellcolor{blue!6}.625 & \cellcolor{blue!6}.616 & \cellcolor{orange!8}.617 & \cellcolor{orange!8}.610 & \cellcolor{green!8}.944 & \cellcolor{green!8}\underline{.966} & \cellcolor{green!8}.921 & \cellcolor{green!8}.737 & \cellcolor{purple!8}.625 & \textbf{.994} \\
 & Pearson & \cellcolor{blue!6}.159 & \cellcolor{blue!6}.095 & \cellcolor{blue!6}.082 & \cellcolor{blue!6}.092 & \cellcolor{orange!8}.134 & \cellcolor{orange!8}.028 & \cellcolor{green!8}.900 & \cellcolor{green!8}\underline{.941} & \cellcolor{green!8}.880 & \cellcolor{green!8}.548 & \cellcolor{purple!8}.226 & \textbf{.991} \\
 & NRMSE & \cellcolor{blue!6}.0350 & \cellcolor{blue!6}.0356 & \cellcolor{blue!6}.0349 & \cellcolor{blue!6}.0350 & \cellcolor{orange!8}.0336 & \cellcolor{orange!8}.0336 & \cellcolor{green!8}.0126 & \cellcolor{green!8}\underline{.0096} & \cellcolor{green!8}.0140 & \cellcolor{green!8}.0271 & \cellcolor{purple!8}.0349 & \textbf{.0054} \\
\cmidrule(lr){1-14}
\multirow{3}{*}{Torus Adv}
 & SSIM & \cellcolor{blue!6}.777 & \cellcolor{blue!6}.863 & \cellcolor{blue!6}.793 & \cellcolor{blue!6}.780 & \cellcolor{orange!8}.766 & \cellcolor{orange!8}.763 & \cellcolor{green!8}.940 & \cellcolor{green!8}\underline{.961} & \cellcolor{green!8}.932 & \cellcolor{green!8}.893 & \cellcolor{purple!8}.289 & \textbf{.990} \\
 & Pearson & \cellcolor{blue!6}.785 & \cellcolor{blue!6}.898 & \cellcolor{blue!6}.784 & \cellcolor{blue!6}.789 & \cellcolor{orange!8}.786 & \cellcolor{orange!8}.785 & \cellcolor{green!8}.959 & \cellcolor{green!8}\underline{.972} & \cellcolor{green!8}.934 & \cellcolor{green!8}.894 & \cellcolor{purple!8}.131 & \textbf{.991} \\
 & NRMSE & \cellcolor{blue!6}.0360 & \cellcolor{blue!6}.0260 & \cellcolor{blue!6}.0360 & \cellcolor{blue!6}.0358 & \cellcolor{orange!8}.0362 & \cellcolor{orange!8}.0361 & \cellcolor{green!8}.0161 & \cellcolor{green!8}\underline{.0133} & \cellcolor{green!8}.0207 & \cellcolor{green!8}.0258 & \cellcolor{purple!8}.0628 & \textbf{.0076} \\
\cmidrule(lr){1-14}
\multirow{3}{*}{Ellips Flow}
 & SSIM & \cellcolor{blue!6}.210 & \cellcolor{blue!6}.209 & \cellcolor{blue!6}.209 & \cellcolor{blue!6}.209 & \cellcolor{orange!8}.152 & \cellcolor{orange!8}.156 & \cellcolor{green!8}.968 & \cellcolor{green!8}\underline{.972} & \cellcolor{green!8}.483 & \cellcolor{green!8}.167 & \cellcolor{purple!8}.242 & \textbf{1.0000} \\
 & Pearson & \cellcolor{blue!6}.020 & \cellcolor{blue!6}.033 & \cellcolor{blue!6}.045 & \cellcolor{blue!6}.046 & \cellcolor{orange!8}.046 & \cellcolor{orange!8}.075 & \cellcolor{green!8}\underline{.974} & \cellcolor{green!8}\underline{.974} & \cellcolor{green!8}.503 & \cellcolor{green!8}.091 & \cellcolor{purple!8}.189 & \textbf{1.0000} \\
 & NRMSE & \cellcolor{blue!6}.0613 & \cellcolor{blue!6}.0614 & \cellcolor{blue!6}.0614 & \cellcolor{blue!6}.0614 & \cellcolor{orange!8}.0698 & \cellcolor{orange!8}.0697 & \cellcolor{green!8}.0141 & \cellcolor{green!8}\underline{.0139} & \cellcolor{green!8}.0532 & \cellcolor{green!8}.0696 & \cellcolor{purple!8}.0602 & $\bm{1.45\!\times\!10^{-4}}$ \\
\cmidrule(lr){1-14}
\multirow{3}{*}{Maxwell}
 & SSIM & \cellcolor{blue!6}.356 & \cellcolor{blue!6}.392 & \cellcolor{blue!6}.348 & \cellcolor{blue!6}.409 & \cellcolor{orange!8}.435 & \cellcolor{orange!8}.471 & \cellcolor{green!8}.606 & \cellcolor{green!8}\underline{.610} & \cellcolor{green!8}.548 & \cellcolor{green!8}.446 & \cellcolor{purple!8}.348 & \textbf{.979} \\
 & Pearson & \cellcolor{blue!6}.135 & \cellcolor{blue!6}.233 & \cellcolor{blue!6}.047 & \cellcolor{blue!6}.266 & \cellcolor{orange!8}.311 & \cellcolor{orange!8}.372 & \cellcolor{green!8}.599 & \cellcolor{green!8}\underline{.601} & \cellcolor{green!8}.510 & \cellcolor{green!8}.288 & \cellcolor{purple!8}.122 & \textbf{.989} \\
 & NRMSE & \cellcolor{blue!6}.0640 & \cellcolor{blue!6}.0629 & \cellcolor{blue!6}.0642 & \cellcolor{blue!6}.0624 & \cellcolor{orange!8}.0579 & \cellcolor{orange!8}.0559 & \cellcolor{green!8}\underline{.0519} & \cellcolor{green!8}\underline{.0519} & \cellcolor{green!8}.0561 & \cellcolor{green!8}.0576 & \cellcolor{purple!8}.0642 & \textbf{.0150} \\
\cmidrule(lr){1-14}
\multirow{3}{*}{Wilson}
 & SSIM & \cellcolor{blue!6}.797 & \cellcolor{blue!6}.908 & \cellcolor{blue!6}.978 & \cellcolor{blue!6}.956 & \cellcolor{orange!8}.790 & \cellcolor{orange!8}.856 & \cellcolor{green!8}.995 & \cellcolor{green!8}\underline{.998} & \cellcolor{green!8}.974 & \cellcolor{green!8}.909 & \cellcolor{purple!8}.842 & \textbf{.99994} \\
 & Pearson & \cellcolor{blue!6}.111 & \cellcolor{blue!6}.571 & \cellcolor{blue!6}.862 & \cellcolor{blue!6}.765 & \cellcolor{orange!8}.271 & \cellcolor{orange!8}.444 & \cellcolor{green!8}.949 & \cellcolor{green!8}\underline{.976} & \cellcolor{green!8}.843 & \cellcolor{green!8}.640 & \cellcolor{purple!8}.286 & \textbf{.99991} \\
 & NRMSE & \cellcolor{blue!6}.0149 & \cellcolor{blue!6}.0102 & \cellcolor{blue!6}.0054 & \cellcolor{blue!6}.0072 & \cellcolor{orange!8}.0159 & \cellcolor{orange!8}.0135 & \cellcolor{green!8}.0027 & \cellcolor{green!8}\underline{.0016} & \cellcolor{green!8}.0056 & \cellcolor{green!8}.0107 & \cellcolor{purple!8}.0138 & $\bm{8.93\!\times\!10^{-5}}$ \\
\cmidrule(lr){1-14}
\multirow{3}{*}{Yang--Mills}
 & SSIM & \cellcolor{blue!6}.441 & \cellcolor{blue!6}.472 & \cellcolor{blue!6}.504 & \cellcolor{blue!6}.432 & \cellcolor{orange!8}.518 & \cellcolor{orange!8}.519 & \cellcolor{green!8}.826 & \cellcolor{green!8}\underline{.868} & \cellcolor{green!8}.684 & \cellcolor{green!8}.514 & \cellcolor{purple!8}.406 & \textbf{.99999} \\
 & Pearson & \cellcolor{blue!6}.186 & \cellcolor{blue!6}.262 & \cellcolor{blue!6}.328 & \cellcolor{blue!6}.161 & \cellcolor{orange!8}.337 & \cellcolor{orange!8}.339 & \cellcolor{green!8}.778 & \cellcolor{green!8}\underline{.829} & \cellcolor{green!8}.597 & \cellcolor{green!8}.358 & \cellcolor{purple!8}.028 & \textbf{.99999} \\
 & NRMSE & \cellcolor{blue!6}.0355 & \cellcolor{blue!6}.0348 & \cellcolor{blue!6}.0340 & \cellcolor{blue!6}.0356 & \cellcolor{orange!8}.0331 & \cellcolor{orange!8}.0330 & \cellcolor{green!8}.0227 & \cellcolor{green!8}\underline{.0202} & \cellcolor{green!8}.0290 & \cellcolor{green!8}.0331 & \cellcolor{purple!8}.0362 & $\bm{1.67\!\times\!10^{-4}}$ \\
\bottomrule
\end{tabular}%
}
\vspace{-0.5em}
\end{table}

We evaluate RHMP on seven physical systems chosen to probe complementary structural requirements: stretched-grid vorticity prediction,
anisotropic scalar transport on a torus, equivariant surface-vector prediction on an
ellipsoid, nonlinear Maxwell magnetostatics, Abelian \(U(1)\) curvature
prediction, non-Abelian \(SU(2)\) field-strength prediction, and variable-mesh
AirfRANS pressure prediction. 
The first six tasks share a fixed-mesh evaluation protocol and form the main quantitative comparison (Table~\ref{tab:main}, §\ref{sec:main_results}). AirfRANS uses an independent unstructured mesh for each sample and is reported
separately in Section~\ref{sec:airfoil}.
Full task definitions, parameter
distributions, qualitative examples, and conservation-residual checks are given
in Appendix~\ref{app:tasks} and Appendix~\ref{app:qualitative}.

Each task is selected to isolate a different structural demand. \textbf{Navier--Stokes vorticity prediction} (PDEBench~\cite{takamoto2022pdebench}) tests whether the Hodge factorization remains useful even
on a stretched structured grid.
\textbf{Torus convection--diffusion} tests anisotropic scalar transport on a nontrivial topology.
%
\textbf{Ellipsoid surface flow} requires $\En$-equivariant vector readout~\cite{bhatia2013helmholtz} .
\textbf{Maxwell magnetostatics} tests divergence-free flux prediction in saturating,
grain-oriented steel.
%
\textbf{$U(1)$ Wilson loop}~\cite{wilson1974confinement} tests exact Abelian curvature invariance under $A\mapsto A+d\lambda$.
The \textbf{$SU(2)$ Yang--Mills}~\cite{creutz1980monte} task tests cochain differentials together with the nonlinear commutator term $F=dA+[A,A]$.
\textbf{AirfRANS airfoil pressure}~\cite{bonnet2022airfrans} tests transfer across per-sample variable meshes and mesh sizes.

We compare with 12 baselines, covering graph methods
(GCN \cite{kipf2017semi}, GAT \cite{velickovic2018graph},
SchNet \cite{schutt2017schnet}, EGNN \cite{satorras2021en}),
topological methods (MPSN \cite{bodnar2021weisfeiler_sc},
SCCNN \cite{yang2022sccnn}),
geometric methods (GaugeEquivCNN \cite{cohen2019gauge},
GEM-CNN \cite{dehaan2021gem},
CW Net \cite{bodnar2021weisfeiler},
Clifford-SMPN \cite{liu2024clifford}),
and operator learning methods
(FNO \cite{li2021fourier}, DeepONet \cite{lu2021deeponet}).
Benchmark models use Adam with cosine annealing ($10^{-3}\to10^{-5}$)
for 100 epochs on the splits in Appendix~\ref{app:tasks}.
All entries use seed $42$. Baseline widths are selected against the RHMP parameter
budget for their input representation, with $+20\%$ headroom.
Across seeds $\{42,1,2\}$, RHMP's NRMSE standard
deviation is at most $0.0009$ on the six fixed-mesh tasks and below $0.0001$ on
AirfRANS. Appendices~\ref{app:bootstrap} and~\ref{app:seed-variance} give the
corresponding $1{,}000$-resample bootstrap and cross-seed analyses.

\subsection{Main Fixed-Mesh Results}
\label{sec:main_results}

Table~\ref{tab:main} reports results on the six fixed-mesh tasks.
RHMP performs strongly across all six tasks, with clear gains on surface-vector and gauge-curvature prediction.
For objectivity, each baseline is evaluated in the form proposed in its original paper.
We discuss the results below, grouped by structural regime. 

\textbf{Structured grids.}
Navier--Stokes vorticity is the only structured-grid task and therefore the only
fixed-mesh benchmark where the spectral-grid baseline FNO is naturally
applicable. The periodic grid is stretched, with cell aspect ratios up to $8$. FNO achieves SSIM~0.872 / NRMSE~0.023 and RHMP achieves SSIM~0.994 / NRMSE~0.0054 on the quadrilateral cell complex, suggesting that the Hodge factorization remains useful even on structured grids.





\begin{figure}[t]
\vspace{-1.3em}
    \centering
        \includegraphics[width=\textwidth]{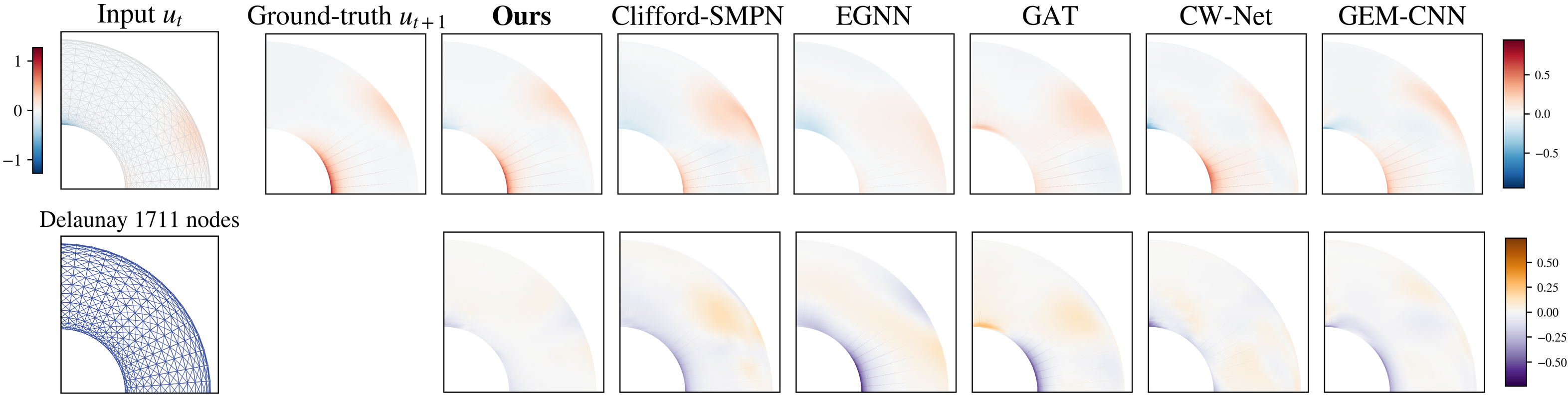}
    \caption{\textbf{Torus advection--diffusion.} Left blocks summarize the inputs; right blocks compare ground truth, RHMP, and representative baselines. The bottom row shows the error maps.}
\vspace{-1.3em}
    \label{fig:main_torus}
\end{figure}

\begin{figure*}[b]
\vspace{-1.2em}
\centering
\begin{minipage}[b]{0.43\linewidth}
\centering
\captionof{table}{Variable-mesh generalization on 100 AirfRANS test simulations (pressure, seed $42$). SSIM ($\uparrow$), Pearson ($\uparrow$), NRMSE ($\downarrow$) at the 1.9K-node training size; $R^2$ ($\uparrow$) without retraining on 0.55K- and 6.9K-node meshes. All methods are in Appendix~\ref{app:airfrans}. \textbf{Bold} = best, \underline{underline} = second best.}
\label{tab:airfoil}
\begin{sc}
\footnotesize
\resizebox{\linewidth}{!}{
\begin{tabular}{lccccc}
\toprule
Method & SSIM & Pearson & NRMSE & $R^2_{0.55\mathrm{K}}$ & $R^2_{6.9\mathrm{K}}$ \\
\midrule
EGNN & .697 & .252 & .0227 & .017 & $-$.001 \\
MPSN & .676 & .149 & .0229 & .017 & .013 \\
GaugeEqCNN & .921 & .857 & .0098 & \underline{.757} & \underline{.688} \\
GEM-CNN & \underline{.947} & \underline{.903} & \underline{.0077} & .728 & .677 \\
CW Net & .693 & .219 & .0226 & .023 & $-$.015 \\
Cliff-SMPN & .684 & .182 & .0228 & .016 & .020 \\
\textbf{RHMP} & \textbf{.998} & \textbf{.997} & \textbf{.0016} & \textbf{.969} & \textbf{.982} \\
\bottomrule 
\end{tabular}
}
\end{sc}
\end{minipage}\hfill
\begin{minipage}[b]{0.55\linewidth}
\centering
\includegraphics[width=\linewidth]{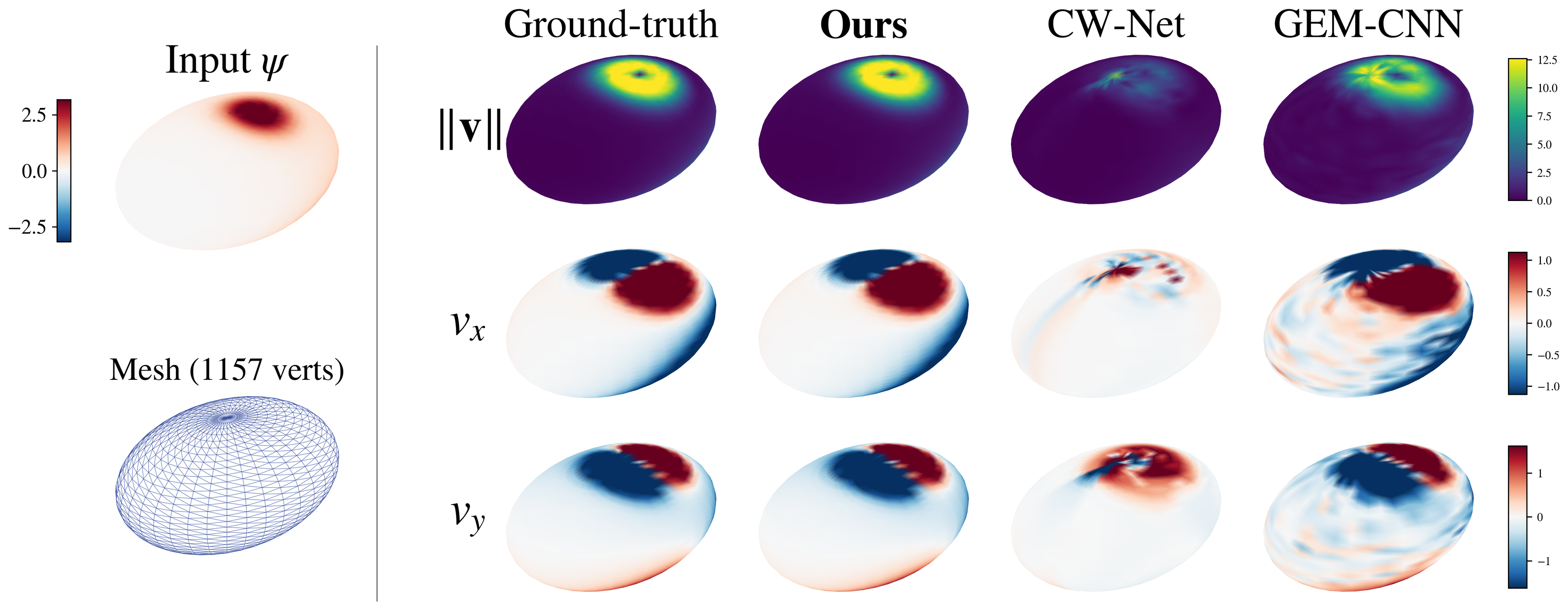}
\captionof{figure}{\textbf{Ellipsoid surface flow.} Left blocks summarize the inputs; right blocks compare ground truth, RHMP, and representative baselines. Additional qualitative comparisons in Appendix~\ref{app:qualitative}.}
\label{fig:real_world}
\end{minipage}
\vspace{-1.3em}
\end{figure*}


\textbf{Cochain support and topology.}
Node-based graph baselines perform poorly on tasks whose target quantities live
naturally on edges or faces, such as vorticity, surface flow, and gauge
curvature. Encoding these quantities as ordinary node features leaves cochain
type, incidence structure, and identities such as \(d^2=0\) to be inferred from
data. Among cell-complex baselines, CW Net improves substantially on vorticity,
surface flow, and Yang--Mills curvature, confirming the value of representing
fields on the correct geometric support.


\textbf{Metric-weighted propagation and field structure.}
RHMP combines higher-order topology with a learned metric \(H_k\) controlling
geometry-dependent propagation. On the \(U(1)\) Wilson-loop task, CW Net
reaches SSIM~0.974, while RHMP reaches SSIM~0.99994. The curvature lifting
retains the exact coboundary \(d_1\theta\) of the input connection.
On \(SU(2)\) Yang--Mills, RHMP reaches NRMSE~\(1.7\times10^{-4}\) while capturing both
the differential term \(dA\) and the nonlinear commutator \([A,A]\).
On Maxwell magnetostatics, RHMP solves for the vector potential with a field-dependent
tensor metric and reads out the flux \(\hat b=d_0\hat a\) with \(d_1\hat b=0\), reaching
SSIM~0.979. The solve and direct face supervision are described in Appendix~\ref{app:structured-readouts}.

\textbf{Frame symmetries.}
Gauge-equivariant surface methods such as GaugeEquivCNN and GEM-CNN are the
strongest baselines across the fixed-mesh tasks. RHMP remains more accurate on every task and is unchanged under
rigid rotations of the test meshes (Appendix~\ref{app:zeroshot}). RHMP combines fixed cochain topology with
cochain-frame-invariant metric prediction, separating manifold-frame symmetry
from the hidden cochain-frame symmetry introduced in this work.

\subsection{Mesh Generalization and Scalability}
\label{sec:airfoil}


AirfRANS~\cite{bonnet2022airfrans} tests a regime in which each sample carries its own unstructured mesh. Methods that are tied to a fixed spectrum, or fixed mesh are not directly applicable, so we compare against those whose native formulation supports per-sample variable meshes. Models predict the pressure on a near-field mesh of each airfoil at inlet speeds above the training range and are evaluated without retraining on coarser and finer meshes (Appendix~\ref{app:airfrans}). RHMP achieves the highest SSIM and lowest NRMSE
in Table~\ref{tab:airfoil} and keeps $R^2\ge0.968$ at both transfer sizes. The shared metric predictor applies the same learned rule to local cell statistics on each mesh, allowing the model to operate across meshes of differing size and connectivity (see Figure~\ref{fig:qual_app_airfoil}). The fixed-mesh models transfer in the same way to new meshes, geometries, and gauges without retraining (Appendix~\ref{app:zeroshot}). Larger-scale profiling and a 100K-cell comparison with CW Net are in Appendix~\ref{app:complexity}.


\begin{figure}[t]
\vspace{-1.3em}
    \centering
        \includegraphics[width=.93\textwidth]{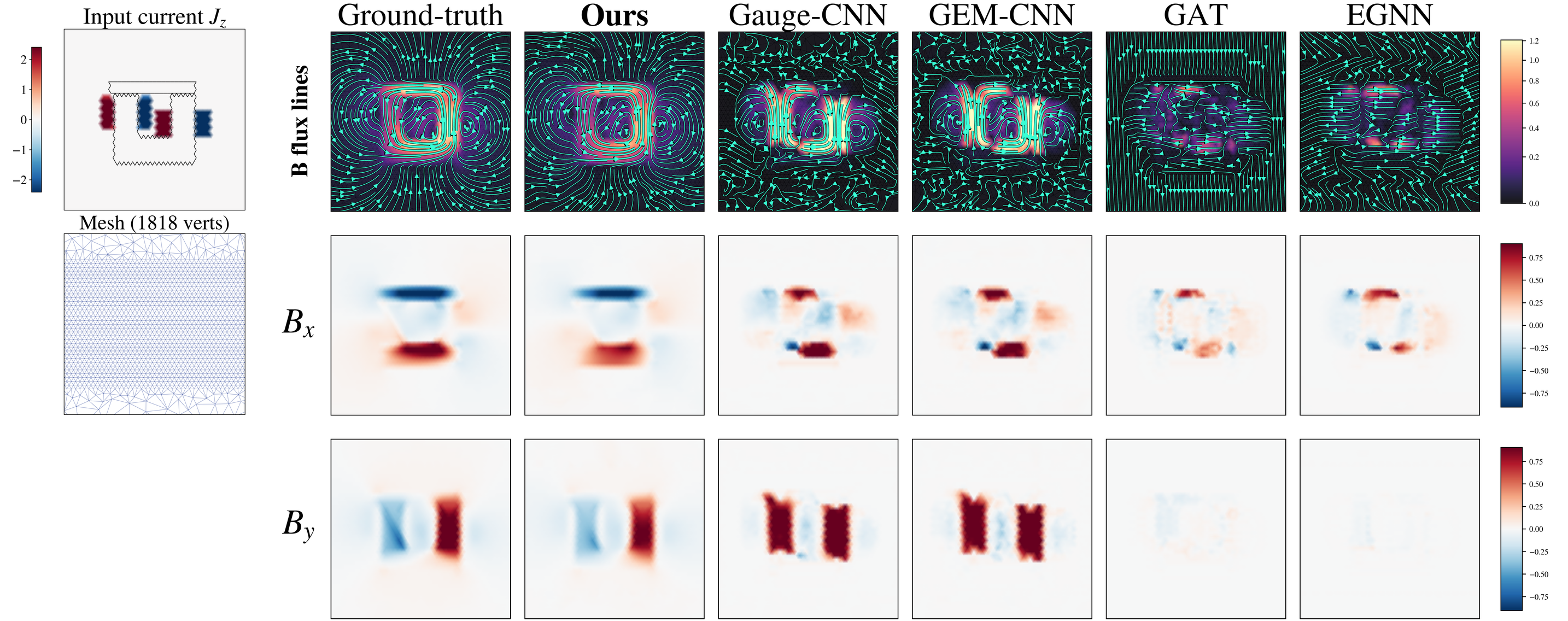}
    \caption{\textbf{Maxwell magnetostatics.} Left blocks summarize the inputs; right blocks compare ground truth, RHMP, and representative baselines.}
    \vspace{-1.3em}
    \label{fig:main_maxwell}
\end{figure}

\subsection{Ablation Studies}
\label{sec:ablation}

We ablate four design choices of the diagonal-plus-low-rank RHMP configuration ($r=8$) on four tasks: CNS vorticity, ellipsoid surface flow, $U(1)$ Wilson loop, and $SU(2)$ Yang--Mills. Each variant changes one component at a time and is retrained for $50$ epochs with the same split and parameter budget. Full protocols and tables are in Appendix~\ref{app:ablation}.  

The ablations separate performance-driving components from structure-preserving components. Setting $H_k=I$ or disabling cross-dimensional transport ($\alpha=1$) sharply reduces $R^2$ on gauge-curvature and structured-grid tasks. For example, Wilson loop drops from $0.96$ to $0.50$ with $H_k=I$. Ellipsoid surface flow is unaffected ($\Delta R^2\le0.05$), consistent with its one-step coexact target. By contrast, learning scalar $d_k$ or replacing the norm-gated nonlinearity with element-wise ReLU changes $R^2$ by at most $0.04$, but breaks structural diagnostics: $\|d_1d_0\|_F$ grows from $0$ to $\sim10^1$, and cochain-frame equivariance error from $\sim10^{-6}$ to $\sim10^0$. Thus, fixed $d_k$ and the $\OC$-equivariant nonlinearity are essential for topology and cochain-frame symmetry
preservation.

\newpage
\begin{figure}[t]
\vspace{-1.2em}
    \centering
        \includegraphics[width=\textwidth]{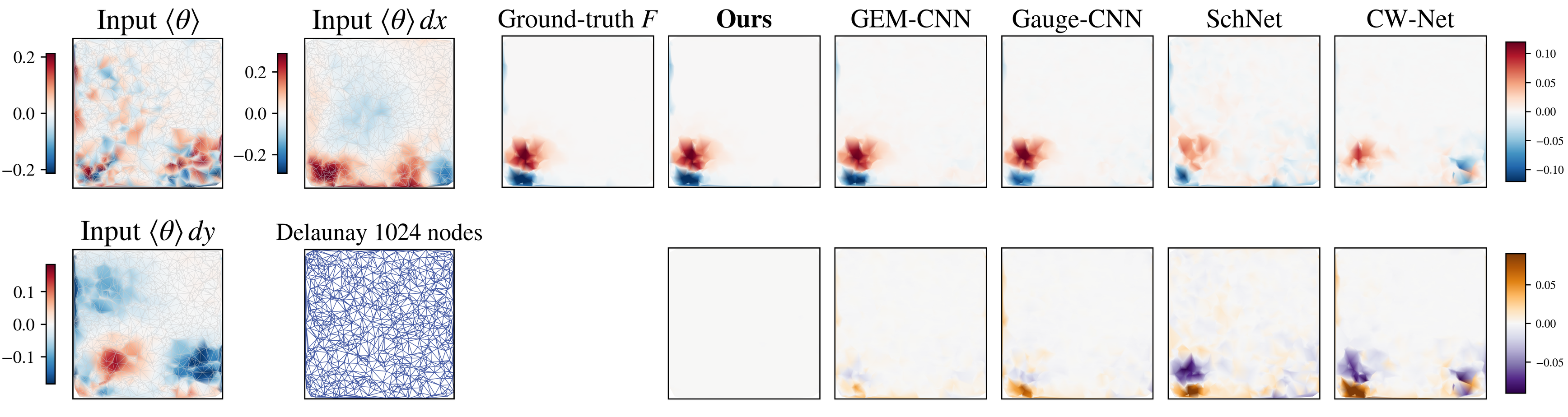}
    \caption{\textbf{U(1) Wilson-loop curvature.} Left blocks summarize the inputs; right blocks compare ground truth, RHMP, and representative baselines. The bottom row shows the error maps.}
    \vspace{-1em}
    \label{fig:main_wilson}
\end{figure}

\section{Conclusion and Limitations}
\label{sec:conclusion}

\textbf{Conclusion.}
RHMP is a metric-centered architecture for mesh fields: fixed coboundaries preserve the cochain complex, while learned SPD cochain metrics control geometry-dependent propagation. This yields positive semidefinite (PSD) Hodge operators, cochain-frame equivariant layers, and exact Abelian curvature invariance. Across seven benchmarks, these constraints bring consistent gains, especially when conservation, learned geometry, topology, and gauge structure interact.

\textbf{Limitations.}
RHMP shows that learning SPD cochain metrics with fixed coboundaries can capture geometry, anisotropy, variable-mesh transfer, and gauge-field structure. Natural next steps are to specialize the metric layer further: group-valued transport to move beyond fitting SU(2) field strength toward full non-Abelian gauge covariance, better-conditioned SPD parameterizations for highly heterogeneous media (Appendix.~\ref{app:metric-comparison}), and compressed or basis-tied per-cell metric activations for larger simulations (Appendix.~\ref{app:complexity}). 

\newpage


{\small
\bibliographystyle{plainnat}
\bibliography{references}
}

\appendix

\section{Broader Impacts}
\label{app:broader_impacts}

This work is directly aimed at surrogate modeling for physical simulation.
The most immediate application is reducing the computational cost of CFD simulation in engineering design scenarios (e.g., airfoil optimization, thermal management),
as well as accelerating routine computations in materials science (property inference for inhomogeneous media)
and computational physics (surrogate solvers for lattice gauge theory).
The method itself addresses the approximation of physical operators.
It involves no personal data, user behavior, or decision recommendations,
and introduces no surveillance or automated decision-making capabilities.
If surrogate models are used in safety-critical engineering simulations
(aerospace structures, nuclear reactors, etc.),
we recommend cross-validation against high-fidelity solvers prior to deployment,
together with quantified uncertainty estimates for the predictions.

\section{Primer: From Continuous Metrics to Discrete $k$-Form Metrics}
\label{app:metric_primer}

This appendix gives the geometric dictionary behind the main text and
clarifies why $H_k$ is the natural object to learn.

\subsection{Worked Example: One Triangle}
\label{app:triangle-example}

We illustrate the constructions of Section~\ref{sec:prelim} on the simplest nontrivial cell complex: a single triangle $K$ with three vertices $v_0, v_1, v_2$, three oriented edges $e_1=(v_0\!\to\!v_1)$, $e_2=(v_1\!\to\!v_2)$, $e_3=(v_0\!\to\!v_2)$, and one face $f_1$ oriented counterclockwise, so that $(n_0, n_1, n_2) = (3, 3, 1)$.

\paragraph{Cochains and coboundaries.}
A 0-cochain $x_0 \in \R^3$ assigns one scalar to each vertex; a 1-cochain $x_1 \in \R^3$ assigns one scalar to each edge; a 2-cochain $x_2 \in \R$ assigns a scalar to the face. The coboundary $d_0 \in \R^{3\times 3}$ maps vertex values to oriented edge differences: the row corresponding to edge $e=(v_i\!\to\!v_j)$ has $-1$ in column $i$ and $+1$ in column $j$, so $d_0$ acts as a discrete gradient. The coboundary $d_1 \in \R^{1\times 3}$ maps edge values to the face by summing oriented edge values around $\partial f_1$, so $d_1$ acts as a discrete curl. The matrices and their correspondence to the oriented mesh are shown in Figure~\ref{fig:triangle_coboundary}.

The identity $d_1 d_0 = 0$ holds combinatorially: each vertex on $\partial f_1$ appears in two boundary edges with opposite signs, so its contributions cancel. This is the discrete counterpart of $\nabla\!\times\!\nabla\varphi = 0$ and holds for every input $x_0$.

\begin{figure}[h!]
\centering
\includegraphics[width=0.65\linewidth]{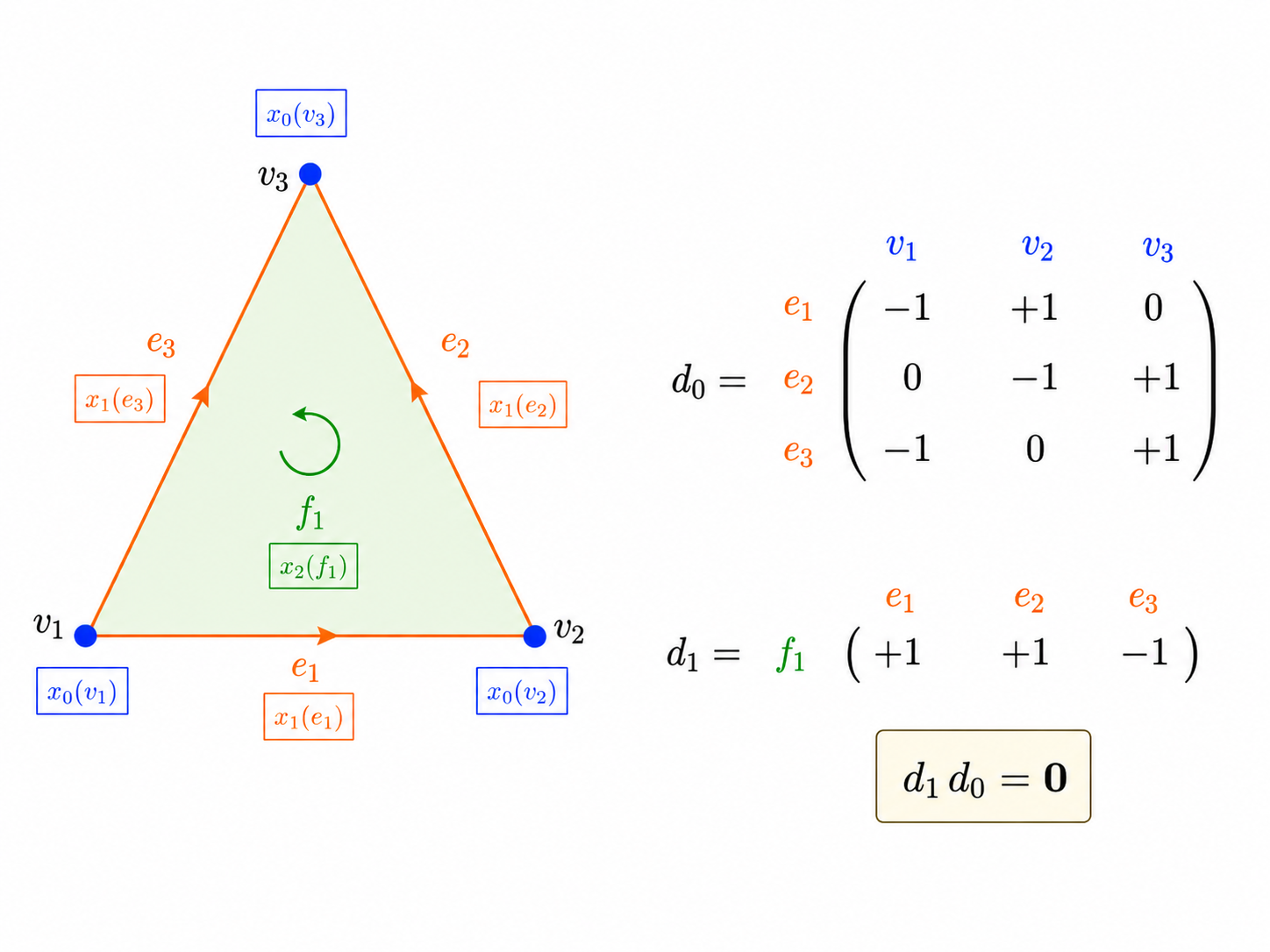}
\caption{The triangle complex $K$ and its coboundary matrices. Left: three vertices (blue), three oriented edges (orange, arrows indicate orientation), and one face (green, counterclockwise). Right: the coboundary matrices $d_0$ (discrete gradient, $3\times 3$) and $d_1$ (discrete curl, $1\times 3$), with row/column labels matching the mesh elements. The identity $d_1 d_0 = 0$ holds by construction: every vertex on the face boundary appears in exactly two boundary edges with opposite signs.}
\label{fig:triangle_coboundary}
\end{figure}

\paragraph{A numerical Hodge step.}
We evaluate one upper Hodge step $d_0^\top H_1 d_0 x_0$ on $x_0 = (2.0, 5.0, 3.0)$ (Figure~\ref{fig:triangle_hodge_step}).

\textit{(i) Coboundary $d_0$.} The coboundary computes oriented differences along each edge; for example, $e_1=(v_0\!\to\!v_1)$ yields $x_0(v_1) - x_0(v_0) = 5.0 - 2.0 = 3.0$. This step is determined by the mesh and contains no learnable parameters.

\textit{(ii) Metric weighting $H_1$.} The diagonal metric $H_1 = \diag(h_1, h_2, h_3)$ scales each edge's gradient by a per-edge weight. With $h_1 = 1.5$, $h_2 = 0.3$, $h_3 = 2.1$, the weighted gradient on $e_1$ is $1.5 \times 3.0 = 4.5$. Each $h_i$ acts as a local constitutive weight (e.g., conductivity or permeability weighted by cell geometry) along edge $e_i$, and $H_1$ is the only learnable component of the operator.

\textit{(iii) Adjoint $d_0^\top$.} The adjoint coboundary aggregates the weighted gradients back to vertices, with signs determined by $d_0^\top$: each edge contributes $-1$ to its source vertex and $+1$ to its target. The result is $d_0^\top H_1 d_0 x_0 = (-6.6,\ 5.1,\ 1.5)$.

The composite $d_0^\top H_1 d_0$ is a weighted graph Laplacian whose connectivity pattern is determined by $d_0$ and whose edge weights are determined by $H_1$. Replacing $H_1$ by the permuted variant $H_1' = \diag(0.3, 2.1, 1.5)$ on the same complex yields $(-2.4,\ 5.1,\ -2.7)$ (inset of Figure~\ref{fig:triangle_hodge_step}); $d_0$ and the identity $d_1 d_0 = 0$ are unchanged, so the change in output is attributable entirely to $H_1$.

\begin{figure}[h!]
\centering
\includegraphics[width=0.95\linewidth]{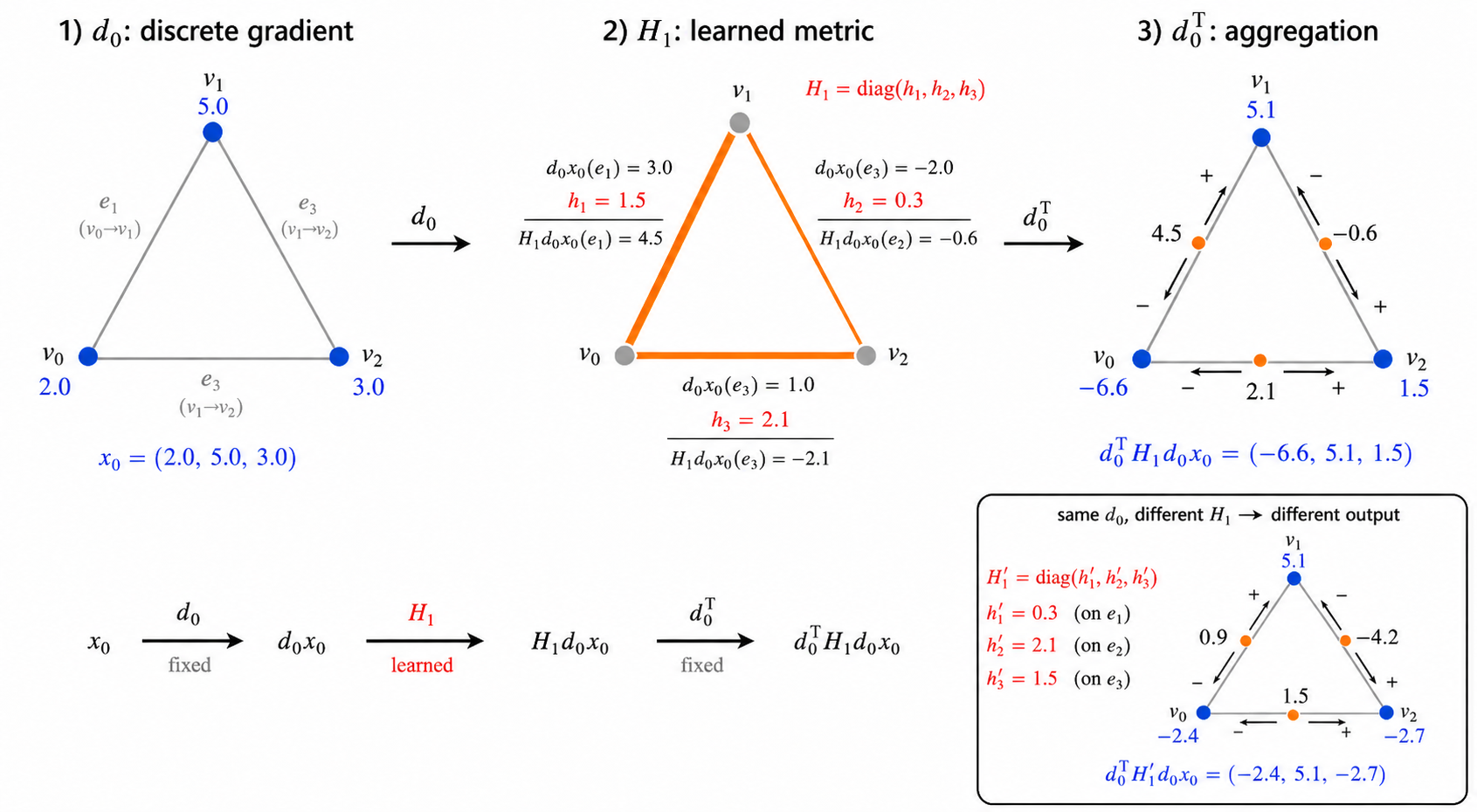}
\caption{One upper Hodge step $d_0^\top H_1 d_0 x_0$ on the triangle of Figure~\ref{fig:triangle_coboundary}, with $x_0=(2.0, 5.0, 3.0)$. Panel 1: the coboundary $d_0$ computes oriented differences along edges (discrete gradient, fixed). Panel 2: the diagonal metric $H_1=\diag(1.5, 0.3, 2.1)$ weights each edge's gradient by a learned constitutive weight (red); edge thickness reflects the metric value. Panel 3: the adjoint $d_0^\top$ aggregates weighted gradients to vertices; arrows from each edge midpoint show the sign ($-$ at source, $+$ at target). Inset: the same $d_0$ with a permuted metric $H_1'=\diag(0.3, 2.1, 1.5)$ produces a different output, illustrating the topology--geometry separation.}
\label{fig:triangle_hodge_step}
\end{figure}

\subsection{Continuous $k$-form Background}

\paragraph{Continuous $k$-forms.}
A $k$-form $\omega$ is an object that can be integrated over oriented $k$-dimensional pieces of a domain. Functions are $0$-forms, line circulations and fluxes are represented by $1$-forms, and surface intensities such as vorticity or field strength are represented by $2$-forms. The exterior derivative $d$ maps $k$-forms to $(k+1)$-forms and satisfies $d^2=0$. This identity is topological, independent of the coordinate system and material parameters; the standard discrete counterparts are developed in DEC~\cite{desbrun2005discrete} and FEEC~\cite{arnold2006finite}.

\paragraph{The Riemannian metric and the Hodge star.}
A Riemannian metric $g$ defines lengths, angles, volumes, and the Hodge star
$\star_g$.  The metric-induced inner product on $k$-forms is
\begin{equation}
  \langle \omega,\eta\rangle_g
  = \int_M \omega\wedge \star_g\eta.
  \label{eq:continuous_form_inner}
\end{equation}
Changing $g$ changes this inner product and therefore changes diffusion,
wave propagation, constitutive response, and curvature-dependent terms.  The
exterior derivative $d$ stays the same; the geometry is in $\star_g$.

\paragraph{Discretization by cochains.}
Given an oriented cell complex, a smooth $k$-form is discretized by
integrating it over each $k$-cell:
\begin{equation}
  x_i = \int_{\sigma_i^k}\omega,\qquad i=1,\ldots,n_k.
\end{equation}
The vector $x\in\R^{n_k}$ is a $k$-cochain.  Stokes' theorem gives the
discrete coboundary matrices $d_k$ and preserves $d_{k+1}d_k=0$ exactly.

\paragraph{The discrete metric matrix.}
Choose basis functions $\{\varphi_i^k\}_{i=1}^{n_k}$ for discrete
$k$-forms.  The matrix representation of the continuous inner product is
\begin{equation}
  (H_k)_{ij}
  = \int_M \varphi_i^k\wedge \star_g\varphi_j^k.
  \label{eq:mass_matrix}
\end{equation}
For any coefficient vector $z$, $z^\top H_k z$ is the squared $L^2_g$ norm of the
corresponding discrete $k$-form.  Hence $H_k$ is symmetric positive definite
whenever the basis has no null mode.  This is the finite-dimensional version
of a Riemannian metric.

\paragraph{From $H_k$ to the Hodge operator.}
For a $k$-cochain $x_k$, the upper Hodge energy is
\begin{equation}
  \|d_k x_k\|_{H_{k+1}}^2
  = x_k^\top d_k^\top H_{k+1}d_k x_k.
\end{equation}
Thus the operator $d_k^\top H_{k+1}d_k$ is positive semidefinite whenever
$H_{k+1}\succ0$. With inner products $H_k$ and $H_{k+1}$, the discrete
adjoint is $d_k^*=H_k^{-1}d_k^\top H_{k+1}$. Its upper Hodge term
$d_k^*d_k$ has coefficient-space energy matrix
$H_kd_k^*d_k=d_k^\top H_{k+1}d_k$. The RHMP lower energy block has the
analogous placed-metric form $d_{k-1}H_{k-1}^{\downarrow}d_{k-1}^\top$.
RHMP learns these metric matrices while keeping the coboundaries fixed,
preserving the topology while adapting the geometry.

\paragraph{Reference metrics and constitutive response.}
\label{app:constitutive-metrics}
A positive reference metric can incorporate cell geometry and known material
coefficients. A bounded diagonal correction takes the form
\[
  H_k=\diag\!\left(h_{k,i}^{\rm ref}e^{a\tanh z_{k,i}}\right),
  \qquad
  e^{-a}H_k^{\rm ref}\preceq H_k\preceq e^aH_k^{\rm ref},
\]
where $H_k^{\rm ref}=\diag(h_k^{\rm ref})$, $a>0$, and $z_{k,i}$ is predicted
from invariant descriptors by a network shared across cells. The relative
metric $(H_k^{\rm ref})^{-1/2}H_k(H_k^{\rm ref})^{-1/2}$ has condition
number at most $e^{2a}$.
On simplicial meshes, positive-definite material tensors
$\Sigma_T=b_T\exp(S_T)$, with $b_T>0$ and $S_T=S_T^\top$, give the local
Whitney--Galerkin construction
\[
  (H_1)_{ee'}=\sum_T\int_T w_e^\top\Sigma_Tw_{e'}\,dx,
  \qquad d_0^\top H_1d_0=K_\Sigma,
\]
where $w_e$ are Whitney 1-form basis functions and $K_\Sigma$ is the $P_1$
finite-element stiffness matrix for anisotropic diffusion~\cite{arnold2006finite}.
The same local constitutive description can be assembled on meshes of
different resolutions.
In physical units, the metric defines the discrete equation
$(K_\Sigma+\mu M_0)u=M_0f$, with lumped mass $M_0$, $\mu\ge0$, and prescribed
boundary conditions; at $\mu=0$, Dirichlet conditions or a compatible
mean-zero formulation determine the solution. The predicted metric is held
fixed within each solve, carried out by conjugate gradients in the
mass-weighted symmetric frame and differentiated through an adjoint solve.
The corresponding PDE residual constrains the metric through its action on
observed fields. Recovery of
individual material tensors depends on the excitation and parameterization.

\paragraph{Why this implies the symmetry constraints.}
A metric is a coordinate-free object.  In the network this is implemented
by making the predicted $H_k$ invariant to feature-basis changes
($\OC$ invariance) and to translations, rotations, and reflections of the
spatial coordinate system ($\En$ invariance).  The corresponding Hodge
messages then inherit cochain-frame equivariance and spatial invariance/equivariance.
This is the logic used in Theorem~\ref{thm:equivariance}.

\section{Cochain Complex Data Flow}
\label{app:cochain-dataflow}

\begin{figure}[h!]
\centering
\includegraphics[width=0.95\linewidth]{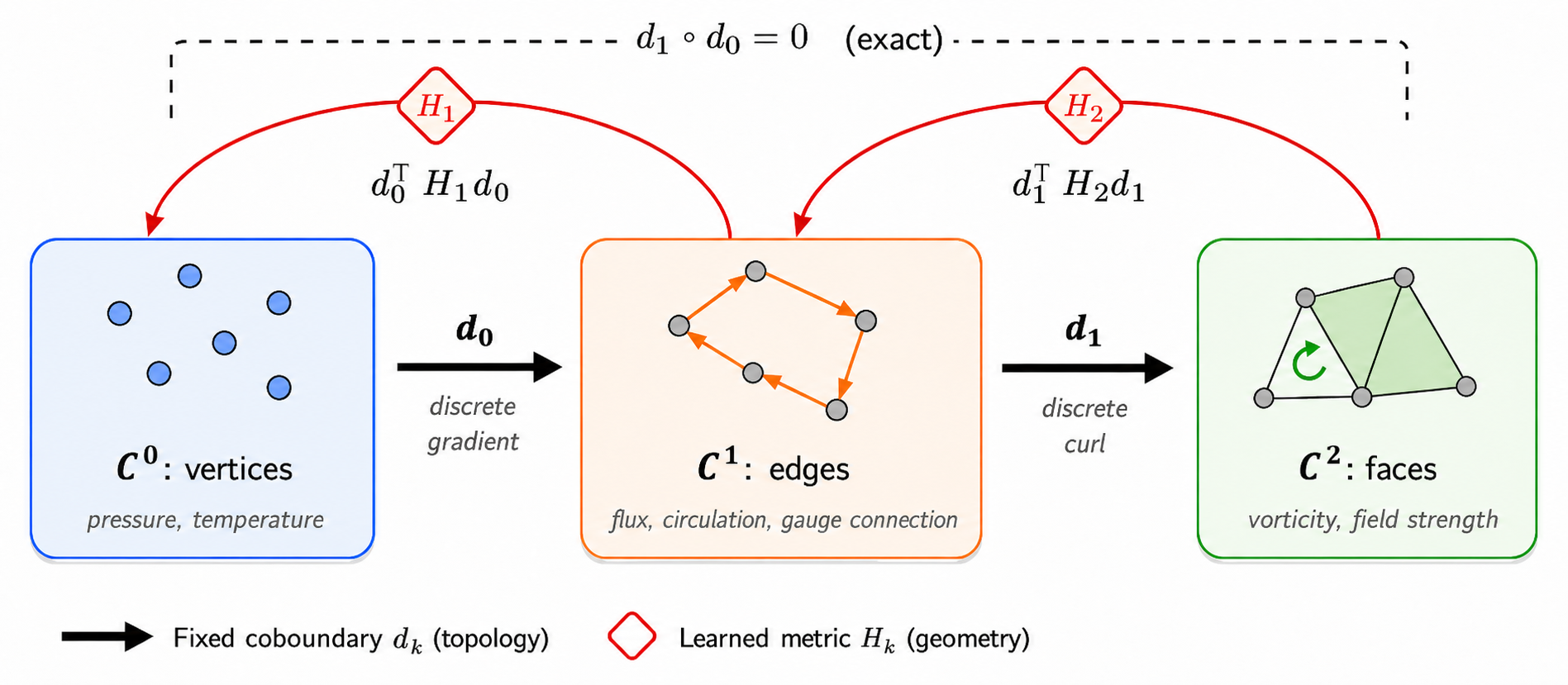}
\caption{Schematic of the cochain-complex data flow within one RHMP layer. The horizontal arrows are the fixed coboundaries $d_0$ and $d_1$ (oriented incidence; satisfy $d_1\circ d_0 = 0$). Each upper loop $d_k^\top H_{k+1} d_k$ lifts features to the adjacent cochain space via $d_k$, re-weights them with the learned SPD metric $H_{k+1}$, and projects back via $d_k^\top$.}
\label{fig:cochain_dataflow}
\end{figure}

Figure~\ref{fig:cochain_dataflow} provides a schematic view of the cochain complex data flow within one RHMP layer. The three cochain spaces $C^0(K)$, $C^1(K)$, $C^2(K)$ carry features on cells of increasing dimension. The forward coboundaries $d_0$ (discrete gradient) and $d_1$ (discrete curl) are determined by oriented incidence and remain frozen throughout training; their composition satisfies $d_1 \circ d_0 = 0$ exactly.

Geometry enters through two upper Hodge return loops. The upper loop on $C^0$ is the composition
\begin{equation}
  C^0(K) \xrightarrow{\;d_0\;} C^1(K) \xrightarrow{\;H_1\;} C^1(K) \xrightarrow{\;d_0^\top\;} C^0(K),
  \qquad x_0 \mapsto d_0^\top H_1\, d_0\, x_0,
\end{equation}
and the upper loop on $C^1$ is
\begin{equation}
  C^1(K) \xrightarrow{\;d_1\;} C^2(K) \xrightarrow{\;H_2\;} C^2(K) \xrightarrow{\;d_1^\top\;} C^1(K),
  \qquad x_1 \mapsto d_1^\top H_2\, d_1\, x_1.
\end{equation}
In each loop, $d_k$ lifts features to the adjacent cochain space, the learned metric $H_{k+1}\in\spd(n_{k+1})$ re-weights them, and $d_k^\top$ projects back. The composite $d_k^\top H_{k+1} d_k$ is positive semidefinite whenever $H_{k+1}\succ 0$ (Eq.~\eqref{eq:psd_energy}). A concrete numerical walkthrough on a single triangle is given in §\ref{app:triangle-example}.

\section{Architecture and Training Details}
\label{app:architecture}

\subsection{Hyperparameters}
Unless stated otherwise below, the RHMP models in Tables~\ref{tab:main} and~\ref{tab:airfoil} use four Hodge MP layers, polynomial order $2$,
metric-predictor hidden dimension $32$, and lifting hidden dimension $64$.
The cochain widths are $C=102,64,64,128,120$ for NS vorticity, torus
transport, ellipsoid flow, Wilson loop, and Yang--Mills,
respectively, and $C=118$ for AirfRANS. The corresponding parameter counts
are $88{,}822$, $35{,}526$, $13{,}332$, $72{,}920$, $105{,}808$,
and $88{,}921$. The NS, torus, Wilson loop, Yang--Mills, and AirfRANS models use the metric projection of their $1$-form input (Appendix~\ref{app:structured-readouts}) as their first layer. The torus model uses tensor metrics and a resolvent in its third layer,
with 12 conjugate-gradient iterations and implicit differentiation. The ellipsoid model uses linear lifting and a linear stream-function head. Maxwell uses the physical solve layer of
Appendix~\ref{app:structured-readouts} with $2{,}229$ parameters.
Training uses batches of 64 fixed-mesh samples or 8 AirfRANS meshes, weight
decay $10^{-5}$, and gradient clipping at $1$; checkpoints are selected by
validation $R^2$.

\subsection{Metric Parameterization and PSD Guarantee}
The NS, ellipsoid, Wilson loop, Yang--Mills, and AirfRANS models use
$H_k=H_k^{\mathrm{ref}}\diag(\exp(a\tanh z_k))$ with $a=2$, as in
Appendix~\ref{app:constitutive-metrics}. This gives
$e^{-2}H_k^{\mathrm{ref}}\preceq H_k\preceq e^2H_k^{\mathrm{ref}}$.
The torus and Maxwell models use the Whitney--Galerkin tensor form of the same appendix with $a=5$.
The ablation studies use the diagonal-plus-low-rank form
\[
  H_k = B_k B_k^\top + \diag(\softplus(\mathbf{h}_k) + \epsilon),\qquad \epsilon=10^{-6},\qquad B_k \in \R^{n_k \times r},\qquad r=8.
\]
The metric predictor outputs the diagonal scalars $\mathbf{h}_k$ per $k$-cell and the basis $B_k$ from the cell-level invariants. SPD is guaranteed by construction: for any nonzero $z\in\R^{n_k}$,
\[
  z^\top H_k z = \|B_k^\top z\|_2^2 + \sum_i\big(\softplus(h_i)+\epsilon\big) z_i^2 \ge \epsilon\|z\|_2^2 > 0.
\]
The PSD property of the Hodge operator then follows from the energy identity
\[
  z^\top d_k^\top H_{k+1}d_k z = (d_k z)^\top H_{k+1}(d_k z)\ge 0.
\]
No eigenvalue clipping or post-hoc projection is used. The pure diagonal form ($r{=}0$) and the full Cholesky form $H_k = L_k L_k^\top + \epsilon I$ are special cases of the same algebraic pattern; their empirical comparison with the $r{=}8$ configuration is given in §\ref{app:metric-comparison}.

\subsection{Empirical Comparison of Metric Parameterizations}
\label{app:metric-comparison}

To justify the choice of metric parameterization, we sweep five SPD variants of $H_k$ under matched training conditions on $SU(2)$ Yang--Mills (256-node Delaunay lattice; 1000 train / 200 val / 500 test, 40 epochs); the rest of the architecture and all hyperparameters are held fixed.

\begin{table}[h]
\centering
\small
\caption{Metric parameterization sweep on $SU(2)$ Yang--Mills (256-node mesh; 1000 train / 200 val / 500 test, 40 epochs). The diagonal-plus-low-rank-basis form with $r{=}8$ is used in the ablation studies.}
\label{tab:metric-sweep}
\begin{tabular}{lcccccc}
\toprule
Variant & Params & Val $R^2$ & Test $R^2$ & NRMSE & SSIM & Pearson \\
\midrule
scalar                                          & 0.011M  & 0.488 & 0.495 & 0.0235 & 0.7653 & 0.666 \\
diagonal, $r{=}0$                               & 0.125M  & 0.487 & 0.493 & 0.0236 & 0.7646 & 0.664 \\
\textbf{diagonal-plus-basis, $r{=}8$} & \textbf{0.066M} & \textbf{0.503} & \textbf{0.510} & \textbf{0.0232} & \textbf{0.7740} & \textbf{0.676} \\
diagonal-plus-basis, $r{=}16$                   & 0.122M  & 0.416 & 0.421 & 0.0252 & 0.7291 & 0.610 \\
full Cholesky                                   & 36.573M & 0.452 & 0.455 & 0.0244 & 0.7449 & 0.637 \\
\bottomrule
\end{tabular}
\end{table}

\textbf{Diagonal-plus-low-rank-basis is Pareto-optimal.}
The $r{=}8$ variant attains the highest test $R^2$ at the smallest parameter budget among the configurations with mesh-dependent capacity. Both larger-capacity variants in the same family ($r{=}16$, full Cholesky) and the scalar and diagonal ($r{=}0$) variants trail it.

\textbf{Full Cholesky is harder to optimize despite higher nominal capacity.}
$H_k = L_k L_k^\top + \epsilon I$ has $\sim$$555\times$ the parameter count of the $r{=}8$ variant; its training loss reaches lower values, but validation $R^2$ saturates $5.1$ points below it (test $R^2$ gap $5.5$). Off-diagonal coupling expands the optimization landscape without expanding the useful capacity for this task.

\textbf{Capacity is not the only axis.}
The scalar variant (one learned $h$ per layer) reaches $R^2{=}0.495$, only $1.5$ points behind the $r{=}8$ variant. A per-cell diagonal contains the uniform scaling as a special case; the size of the gap is task-dependent. The $r{=}8$ basis appears to capture the additional useful structure while remaining well-conditioned; pushing to $r{=}16$ overshoots and loses ground.

The diagonal-plus-$r{=}8$-basis configuration is the Pareto-optimal point of this architecture-level sweep under matched optimization.

\subsection{Empirical Characterization of the Learned Metric}
\label{app:learned-metric}

We examine the metric in heterogeneous Poisson problems,
$-\nabla\!\cdot(\Sigma\nabla u)=f$, where material descriptors and field
solutions are available on unstructured meshes. These experiments probe
how the learned cochain energy represents constitutive geometry.

\textbf{Scalar calibration.}
For isotropic media with conductivity contrast $100$, the physical solve
layer predicts its metric from material descriptors and is trained against
field solutions. Table~\ref{tab:metric-calibration} compares learned log
coefficients with $\log\sigma$ over the first 32 test meshes. The diagonal log correction $\log(h_e/h_e^{\rm ref})$ and the tensor
scale $\tfrac12\log\det\Sigma$ both recover a nearly unit slope. The tensor's edge-direction material response also has
Pearson correlation $0.993$ with the reference response. Thus the learned
energy retains the magnitude and spatial variation of the supplied medium.

\begin{table}[H]
\centering
\small
\caption{Constitutive calibration on 32 heterogeneous Poisson test meshes
(seed $42$). Slopes and intercepts fit learned log coefficients against
reference log conductivity, pooled over cells.}
\label{tab:metric-calibration}
\begin{tabular}{lrrrr}
\toprule
Metric quantity & Cells & Pearson & Slope & Intercept \\
\midrule
Diagonal $\log(h_e/h_e^{\rm ref})$ & $139{,}805$ & $0.998$ & $0.985$ & $0.070$ \\
Tensor $\tfrac12\log\det\Sigma$ & $91{,}550$ & $0.996$ & $1.039$ & $0.005$ \\
\bottomrule
\end{tabular}
\end{table}

\textbf{Anisotropic direction and strength.}
For tensor media with nominal anisotropy ratio $100$, a test mesh with
$2{,}430$ triangles gives a median principal-direction error of $8.95^\circ$
and a mean error of $11.03^\circ$, measured where the reference eigenvalue
ratio exceeds $10$. The median learned and reference eigenvalue ratios are
$186.13$ and $96.87$, respectively. Directional alignment and anisotropy strength therefore give complementary
views of the learned tensor.

\textbf{Dependence on the prediction map.}
The correspondence between the learned metric and material coefficients
depends on how the metric enters the prediction map. On the same scalar
medium, a four-layer tensor polynomial model has layerwise correlations
$(-0.197,-0.276,0.108,-0.730)$ between
$\tfrac12\log\det\Sigma$ and $\log\sigma$; the physical solve gives $0.996$.
These separately trained configurations associate a direct constitutive
interpretation with the metric's role in the discrete field equation.
Intermediate propagation metrics also shape the energy of hidden modes.

\subsection{Batched Forward Pass}
For samples sharing the same mesh,
we merge the batch into the feature dimension by exploiting the linearity of sparse matrix multiply:
$\mathrm{spmm}(d, [\mathbf{x}_1, \dots, \mathbf{x}_B])
= [d\mathbf{x}_1, \dots, d\mathbf{x}_B]$,
avoiding per-sample loops.
The AirfRANS tasks (per-sample meshes) are batched as block-diagonal complexes.

\subsection{Lifting Encoder Details}
The 0-cochain MLP input is $[\mathbf{f}_0 \| \deg(v) \| \bar{d}(v)]$,
where $\bar{d}(v)$ is the average distance to the neighbors.
The 2-cochain MLP input is
$[\sum_{e \in \partial f} \sigma_e \mathbf{x}_1(e) \| A_f \| \alpha_0, \alpha_1, \alpha_2]$,
where $\sigma_e \in \{+1, -1\}$ is given by the sparsity structure of $d_1$.

\subsection{Compute Environment}
\label{app:compute}
Experiments were conducted on a workstation with two NVIDIA RTX PRO 6000 Blackwell GPUs (96 GB each); part of the baseline runs used NVIDIA RTX A6000 GPUs (48 GB). Each individual run uses a single GPU.
Software environment: Python 3.12, PyTorch 2.10, CUDA 12.8.
No third-party geometric deep learning library dependencies (torch-scatter, PyG, etc.) are used;
sparse matrix operations use PyTorch's native \texttt{torch.sparse}.

\subsection{Training Cost}
For the configurations of Tables~\ref{tab:main} and~\ref{tab:airfoil},
each model-task pair is trained for 100 epochs. The total training time of RHMP per task is as follows:
NS Vorticity 68 min,
Torus Adv-Diff 88 min,
Ellipsoid Flow 18 min,
Maxwell 150 min,
Wilson Loop 66 min,
Yang-Mills 68 min,
Airfoil 63 min.
The total training time of RHMP across all 7 tasks is approximately 9 hours.
The full experiment including 12 baselines takes about 35 GPU hours.
Code and pre-trained checkpoints are available at \url{https://github.com/CAB-Lab-Princeton/Riemannian-Hodge-Message-Passing}.

\section{Benchmark Tasks}
\label{app:tasks}

\subsection{Design Principles}

The 7 tasks are organized around the structural capability being tested.
Public benchmarks serve as external anchors, and the remaining tasks fill in
structural dimensions beyond currently public data.  Table~\ref{tab:bench_design}
lists, for each task, the corresponding physical domain, the structural
capability being tested, the diagnostic signal associated with that capability,
and the external anchor.

\begin{table}[h]
\centering
\small
\caption{Design mapping of the benchmark tasks. Each task corresponds to a structural capability that can be characterized independently;
the diagnostic-signal column lists the structural cue measured by that task.}
\label{tab:bench_design}
\adjustbox{max width=\linewidth}{%
\begin{tabular}{llll}
\toprule
Task & Physical system & Capability tested & Diagnostic signal \\
\midrule
NS Vorticity      & Compressible NS~\cite{takamoto2022pdebench} & Cochain propagation on a stretched quadrilateral grid       & Oriented vorticity readout \\
Torus Adv-Diff    & Anisotropic transport on a genus-1 surface & $d_0$ and tensor metrics under nontrivial topology & Topology-sensitive transport \\
Ellipsoid Flow    & Surface flow~\cite{bhatia2013helmholtz}   & $\En$-equivariant vector reconstruction      & Cochain-type support \\
Maxwell           & Magnetostatics in anisotropic steel   & Exact encoding of $\nabla\!\cdot\!\mathbf{B}=0$ & Conservation residual \\
Wilson Loop       & $U(1)$ lattice gauge theory~\cite{wilson1974confinement} & Coboundary $d_1$ accuracy    & Curvature invariance \\
Yang-Mills SU(2)  & $SU(2)$ non-abelian gauge theory~\cite{creutz1980monte} & Coupling of differential and algebraic terms      & Explicit $[A,A]$ coupling \\
Airfoil Pressure  & RANS CFD~\cite{bonnet2022airfrans}  & Per-sample variable-mesh generalization         & Mesh-transfer behavior \\
\bottomrule
\end{tabular}
}
\end{table}

The seven physical systems span a spectrum of structural difficulty.
Compressible Navier--Stokes~\cite{takamoto2022pdebench} and external aerodynamics~\cite{bonnet2022airfrans} place the architecture against established surrogates on common PDEs and engineering CFD;
the remaining tasks reach into more demanding regimes (advection--diffusion on closed surfaces, surface flow on curved manifolds, magnetostatics with hard conservation constraints, and $U(1)$ / $SU(2)$ lattice gauge theory), where standard benchmarks do not yet probe the relevant differential and gauge structure.
For these systems we follow the established discretization schemes for the underlying physics~\cite{desbrun2005discrete}, with the per-task numerical setup detailed below.

\subsection{Dataset Composition and Generation Protocol}

Table~\ref{tab:tasks_full} gives the input/output, reference method, mesh size, and number of samples for each task. The following lists per-task sampling distributions, hyperparameters, and implementation details that complement §\ref{sec:experiments}.
\textbf{NS vorticity}: the $128^2$ periodic PDEBench fields are sampled on a stretched $32\times32$ tensor-product grid whose spacings range from $1$ to $8$ fine cells along each axis; the target is the vorticity at the next time step, computed spectrally on the fine grid.
\textbf{Torus advection--diffusion}: $\partial_t u+\nabla\!\cdot(\mathbf{v}u)=\nabla\!\cdot(D\nabla u)$ is integrated to $t=0.2$ with a divergence-free tangent velocity (toroidal and poloidal amplitudes $\alpha\sim U(0.3,1.5)$ and $\beta\sim U(0,0.8)$ plus a random stream-function perturbation) and a field-aligned diffusion tensor $D=D_\perp I+(D_\parallel-D_\perp)\mathbf{b}\mathbf{b}^\top$ along helical field lines $\mathbf{b}$ with $D_\parallel/D_\perp\sim\log U(1,100)$; the reference uses conservative finite volumes on the dual cells and Crank--Nicolson time stepping with 200 sub-steps solved by sparse LU factorization.
\textbf{Ellipsoid surface flow}: the stream function $\psi$ is sampled from spherical harmonics, and $\mathbf{v}_\text{tan} = \mathbf{n} \times \nabla\psi$ is evaluated by cotangent-weighted discrete gradients~\cite{desbrun2005discrete}.
\textbf{Maxwell magnetostatics}: $\nabla\!\times(\nu(\mathbf{B})\,\nabla\!\times\!A_z)=J_z$ on cross-sections of E--I and U--I electromagnets and of E--I transformers with magnetic shunts, with random dimensions, air gaps, windings, and currents. The cores are built from non-oriented, grain-oriented, and laminated steel; the reluctivity tensor $\nu$ follows the rolling direction of each piece, with permeability ratios up to $25$ and a saturation polarization of $2$\,T. The models receive the material indicators and the rolling directions. The reference is a P1 finite-element Newton solve with relative residual below $10^{-10}$, and the target is $\mathbf{B}=\nabla\!\times\!A_z$ at the vertices.
\textbf{$U(1)$ Wilson loop}: the real-valued edge connection $\theta$ combines gradients of five Gaussian bumps with one to three regularized vortices; $F=d_1\theta$ is the oriented plaquette sum~\cite{wilson1974confinement}.
\textbf{$SU(2)$ Yang--Mills}: each Lie-algebra component combines Gaussian edge noise (standard deviation $0.1$), four bump gradients, and two vortices. The face field $F=d_1A+[A,A]$ combines the oriented edge sum with cyclic cross products~\cite{creutz1980monte}.
The six fixed-mesh datasets use sequential $70/15/15$ splits
(train/validation/test); Table~\ref{tab:main} evaluates the first 100 test
samples in physical units. AirfRANS uses 630 training, 70 validation, and
300 test meshes; Table~\ref{tab:airfoil} evaluates the 100 meshes with source
zero-based indices $850$--$949$. Its normalization statistics use the first 700 meshes;
fixed-mesh normalization uses the training partition. The AirfRANS meshes and targets
are described in §\ref{app:airfrans}.

\begin{table}[h]
  \caption{Full description of the 7 benchmark tasks.}
  \label{tab:tasks_full}
  \centering
  \small
  \adjustbox{max width=\linewidth}{%
  \begin{tabular}{lllllcc}
    \toprule
    Task & Physics & Reference method & Input/output & Mesh & Nodes & Samples \\
    \midrule
    NS Vorticity & Compressible NS & PDEBench~\cite{takamoto2022pdebench} & $(\rho,V_x,V_y,p) \to \omega$ & Stretched $32^2$ & 1024 & 10000 \\
    Torus Adv-Diff & Anisotropic transport & finite volumes, Crank--Nicolson (SuperLU) & $(f,\mathbf{v},D) \to u$ & Irregular torus & 1711 & 3000 \\
    Ellipsoid Flow & Surface flow & discrete surface gradient~\cite{desbrun2005discrete} & $\psi \to \mathbf{v}_\text{tan}$ & Irregular ellipsoid & 1157 & 10000 \\
    Maxwell & Magnetostatics & P1 FEM, Newton (SuperLU) & $(J,\nu) \to \mathbf{B}$ & Irregular triangulation & 1818 & 5000 \\
    Wilson Loop & U(1) gauge curvature & oriented plaquette sum~\cite{wilson1974confinement} & $\theta \to F$ & Irregular Delaunay & 1024 & 10000 \\
    Yang-Mills SU(2) & Non-abelian field strength & Lie-algebra plaquette formula~\cite{creutz1980monte} & $A \to F\!=\!dA\!+\![A,A]$ & Irregular Delaunay & 1024 & 10000 \\
    Airfoil Pressure & RANS CFD & AirfRANS~\cite{bonnet2022airfrans} & $(\text{sdf},V,\alpha) \to p$ & Per-sample mesh & ${\sim}1900$ & 1000 \\
    \bottomrule
  \end{tabular}
  }
\end{table}

\subsection{Implementation Details}

\textbf{Numerical setup.}
All reference fields are computed at numerical tolerance $\le 10^{-6}$ using validated FEM, finite-volume, and sparse-direct-solve routines (SciPy SuperLU, cotangent-weight FEM); per-task numerical schemes are listed in §\ref{app:tasks}.
Code, simulation scripts, and datasets are released together for reproduction.

\textbf{Baseline protocol.}
Each baseline is used under the input/output form of its original paper, without any core architectural modification to adapt to the tasks in this paper.
We separate three evaluation regimes by the native domain of each method:
(i) the six fixed-mesh tasks (Table~\ref{tab:main}), evaluated on the 11 baselines whose native definition admits irregular meshes;
(ii) the structured-grid task NS Vorticity, where FNO is additionally compared inline in §\ref{sec:main_results};
(iii) per-sample variable meshes (AirfRANS, §\ref{sec:airfoil}, Table~\ref{tab:airfoil}), where only methods that admit a different mesh per sample participate.
The benchmark runs use $100$ epochs and the task-specific splits above.
Baseline widths are selected against the RHMP budget for the same input
encoding, with $+20\%$ headroom; for ellipsoid flow and Maxwell the budgets are
$29{,}621$ and $91{,}861$ parameters. MPSN, SCCNN, and Clifford-SMPN are evaluated
on edge-averaged node targets, $\tfrac12(y_i+y_j)$; for ellipsoid flow this
readout uses the first vector component. Other methods in Table~\ref{tab:main}
are evaluated on the node targets.

\textbf{Cochain representation and readout.}
NS vorticity uses a periodic quadrilateral complex, with density and pressure on
vertices and velocity projected onto oriented edges as the mean and the
half-difference of the two endpoint projections. Vorticity and gauge
curvature are predicted on faces and mapped to the node targets by fixed
oriented averaging. The NS and $SU(2)$ readouts include the wedge product of two
learned edge cochains on each face, $(a\wedge b)_f=c_f\sum_j(a_jb_{j+1}-a_{j+1}b_j)$ over the boundary edges of the face
in cyclic order, with $c_f=\tfrac14$ on quadrilaterals and $\tfrac16$ on triangles,
which represents the advection term and the commutator; these two models use $16$
context channels from adjacent degrees in the lifting, a three-stage face readout of
width $32$, and a gate evaluated on $\|\mathbf m\|/\|\mathbf x\|$ and the normalized inner
product of $\mathbf x$ and $\mathbf m$. On the torus, the velocity enters as an edge cochain
$\theta$ and the diffusion tensor through edge and face descriptors of the metric
predictor; the message-passing layers after the projection add conservative, advective,
and upwind transport operators of $\theta$ with learned coefficients to the vertex update. On the ellipsoid, the model predicts a stream function, applies $d_0$, and uses a
fixed edge-to-node gradient map followed by
$\mathbf v=\mathbf n\times\mathbf g$; vector components share one RMS scale.
Maxwell predicts the vector potential, applies $d_0$, and uses a fixed
edge-to-node vector map. For $U(1)$, the connection enters the lifting
through $d_1\theta$. The $SU(2)$ lifting also retains an odd connection
path. Graph baselines use node encodings of
these fields; the raw connection input retains their oriented edge support.

\textbf{Evaluation metrics.}
NRMSE is the pooled RMSE after inverse normalization divided by the global
target range of the evaluation subset. SSIM uses global statistics of each
flattened sample; SSIM and Pearson are then averaged over samples.
For SSIM and Pearson, predictions are clipped to the target range with
$10\%$ padding, and negative Pearson correlations are set to zero.

\subsection{Structured Readouts and Physical Solves}
\label{app:structured-readouts}

\textbf{Face curvature.}
Direct supervision of $F$ as a $2$-cochain retains the oriented face target
throughout prediction and evaluation. With the same four-layer diagonal
metric models, Wilson loop attains SSIM $0.9999998$, Pearson $0.999988$,
and NRMSE $1.44\times10^{-5}$; Yang--Mills attains $0.9999972$, $0.9999954$, and
$8.01\times10^{-5}$, respectively. The corresponding uncentered
$R^2=1-\sum\|\hat F-F\|^2/\sum\|F\|^2$ values are $0.999991$ and $0.9999908$.
Both evaluations use the first 100 test samples and seed $42$.

\textbf{Metric Hodge projection.}
For a $1$-form input $x_1$, the first layer of the NS, torus, Wilson loop,
Yang--Mills, and AirfRANS models solves $(d_0^\top H_1d_0)\alpha=d_0^\top H_1x_1$ with the
learned edge metric of the layer and passes $u=x_1-d_0\alpha$ and the potential $\alpha$ to
the gates. The potential is fixed to zero on the boundary vertices of the Delaunay meshes and
on the outer boundary of the airfoil meshes and is free elsewhere; on the periodic grid and
on the torus, where the operator has the constant null space, one vertex is fixed in the
solve and the potential is shifted to zero mean. The solve is a batched Cholesky
factorization in double precision on the fixed meshes and a sparse factorization per mesh
on AirfRANS, with gradients through the factorization.

\textbf{Nonlinear magnetostatic solve.}
For Maxwell magnetostatics, a single physical solve layer with $C=4$ predicts
the vector potential using a locally assembled tensor metric, linear lifting and
readout, and scale-only normalization, followed by the exact edge gradient
and the fixed node-vector map; the potential vanishes on the outer boundary. The metric predictor receives the material and
rolling-direction descriptors of each cell together with the local field energy of
the previous iterate, and the solve is repeated four times with a learned relaxation,
so that the metric represents the field-dependent reluctivity.
Each solve uses two-level preconditioned conjugate gradients, at most 128
iterations, a target relative tolerance of $10^{-5}$, and implicit adjoint
gradients. The linear readout is initialized by least squares on the first
64 training samples. This configuration ($2{,}229$ parameters) is the Maxwell
entry of Table~\ref{tab:main}.

\subsection{AirfRANS Tasks}
\label{app:airfrans}

For each of the 1000 AirfRANS simulations~\cite{bonnet2022airfrans} we triangulate the fluid within $0.5$ chord of the
airfoil with a Delaunay mesh that conforms to the airfoil ($1{,}887$ nodes on average). The 300 test simulations are
also meshed with the spacing halved ($6{,}865$ nodes) and doubled ($550$ nodes). The sequential split tests
extrapolation in inlet speed: training simulations span $31$--$70$\,m/s and test simulations $74$--$94$\,m/s.

\textbf{Pressure.}
The target is the pressure at the vertices. The baselines receive signed distance, inlet speed, and angle of attack
on vertices; RHMP receives the signed distance and the inflow as the edge cochain
$\theta_e=\mathbf u_\infty\cdot(\mathbf r_j-\mathbf r_i)$. Its first layer projects the edge features with the learned
metric, $u=x_1-d_0\alpha$ with $(d_0^\top H_1d_0)\alpha=d_0^\top H_1x_1$ on the vertices off the outer boundary and
$\alpha=0$ on the outer boundary, which for $H_1=\star_1$ is the potential flow around the airfoil; it then removes the harmonic
cochain of the annular fluid domain in a second learned metric, which sets the circulation, and the potential and the
projected field enter the cochain update through the gates with learned coefficients. The vertex readout includes
the Whitney energy density of two learned edge cochains, and the output is
scaled by $|\mathbf u_\infty|^2$, so the model is homogeneous of degree two in the inflow. RHMP uses $C=118$
($88{,}921$ parameters). Table~\ref{tab:airfrans_pressure} reports all methods at the training size and, without
retraining, on the coarser and finer meshes. Without the projection layer, the energy readout, or the circulation
term, the $R^2$ of RHMP is $0.432$, $0.963$, and $0.983$. Trained on the pressure coefficient with chord-frame positions
as additional inputs, GEM-CNN reaches $R^2=0.987$; with the Bernoulli pressure of a potential-flow solution as an
additional input, $0.989$.

\begin{table}[h]
\centering
\small
\caption{AirfRANS pressure (100 test simulations, seed $42$). SSIM, Pearson, and NRMSE at the 1.9K-node training size; $R^2$ on 0.55K-, 1.9K-, and 6.9K-node meshes, the first and the last without retraining; $R^2$ on the airfoil surface and mean absolute error of the lift coefficient. \textbf{Bold} = best, \underline{underline} = second best.}
\label{tab:airfrans_pressure}
\adjustbox{max width=\linewidth}{%
\begin{tabular}{lcccccccc}
\toprule
Method & SSIM & Pearson & NRMSE & $R^2_{0.55\mathrm{K}}$ & $R^2_{1.9\mathrm{K}}$ & $R^2_{6.9\mathrm{K}}$ & $R^2_{\mathrm{wall}}$ & $|\Delta C_l|$ \\
\midrule
GCN & .689 & .164 & .0227 & .027 & .028 & .024 & .012 & .810 \\
GAT & .617 & .170 & .0233 & $-$.040 & $-$.023 & $-$.011 & $-$.053 & .810 \\
SchNet & .673 & .233 & .0228 & .012 & .023 & .020 & .031 & .810 \\
EGNN & .697 & .252 & .0227 & .017 & .036 & $-$.001 & .006 & .810 \\
MPSN & .676 & .149 & .0229 & .017 & .015 & .013 & $-$.008 & .811 \\
SCCNN & .689 & .165 & .0227 & .028 & .030 & .027 & .012 & .810 \\
GaugeEqCNN & .921 & .857 & .0098 & \underline{.757} & .820 & \underline{.688} & .810 & .189 \\
GEM-CNN & \underline{.947} & \underline{.903} & \underline{.0077} & .728 & \underline{.890} & .677 & \underline{.893} & \underline{.152} \\
CW Net & .693 & .219 & .0226 & .023 & .041 & $-$.015 & .023 & .802 \\
Cliff-SMPN & .684 & .182 & .0228 & .016 & .026 & .020 & .018 & .805 \\
\textbf{RHMP} & \textbf{.998} & \textbf{.997} & \textbf{.0016} & \textbf{.969} & \textbf{.995} & \textbf{.982} & \textbf{.995} & \textbf{.020} \\
\bottomrule
\end{tabular}
}
\end{table}

\textbf{Vorticity.}
RHMP receives the circulation of the RANS velocity along each edge,
$\theta_e=\tfrac12(\mathbf u_i+\mathbf u_j)\cdot(\mathbf r_j-\mathbf r_i)$, and the baselines the nodal velocity; all
models also receive signed distance, inlet speed, and angle of attack on vertices. The target at vertex $i$ is the
mean vorticity over its star, $\sum_{f\ni i}\sigma_f(d_1\theta)_f/\sum_{f\ni i}|A_f|$, with $\sigma_f$ the face
orientation and $A_f$ the face area. RHMP ($C=128$, $73{,}428$ parameters) predicts a face cochain and applies this
oriented area average. Table~\ref{tab:airfrans_vorticity} reports all methods at the training size and, without
retraining, on the coarser and finer meshes.

\begin{table}[h]
\centering
\small
\caption{AirfRANS vorticity (100 test simulations, seed $42$). SSIM, Pearson, and NRMSE at the 1.9K-node training size; $R^2$ on 0.55K-, 1.9K-, and 6.9K-node meshes, the first and the last without retraining. \textbf{Bold} = best, \underline{underline} = second best.}
\label{tab:airfrans_vorticity}
\begin{tabular}{lcccccc}
\toprule
Method & SSIM & Pearson & NRMSE & $R^2_{0.55\mathrm{K}}$ & $R^2_{1.9\mathrm{K}}$ & $R^2_{6.9\mathrm{K}}$ \\
\midrule
GCN & .682 & .837 & .0377 & .472 & .703 & .544 \\
GAT & .901 & .950 & .0243 & .698 & .877 & .600 \\
SchNet & .863 & .929 & .0272 & .455 & .845 & .558 \\
EGNN & .936 & .957 & .0195 & .679 & .921 & .678 \\
MPSN & .874 & .928 & .0257 & .781 & .862 & .681 \\
SCCNN & .695 & .837 & .0376 & .620 & .705 & .505 \\
GaugeEqCNN & .997 & \underline{.998} & \underline{.0040} & \underline{.990} & \underline{.997} & \underline{.928} \\
GEM-CNN & \underline{.998} & \underline{.998} & .0041 & .483 & \underline{.997} & .785 \\
CW Net & .961 & .980 & .0139 & .751 & .960 & .765 \\
Cliff-SMPN & .966 & .980 & .0137 & .922 & .961 & .799 \\
\textbf{RHMP} & \textbf{1.000} & \textbf{1.000} & \textbf{.0003} & \textbf{1.000} & \textbf{1.000} & \textbf{1.000} \\
\bottomrule
\end{tabular}
\end{table}

\textbf{Eddy-viscosity diffusion.}
On the same meshes we solve $-\nabla\!\cdot(\kappa\nabla T)=f$ with $\kappa=1+\nu_t/\nu$ taken from the RANS solution
(contrast up to $5.9\times10^{3}$), Gaussian sources, and $T=0$ on the airfoil and the outer boundary. The physical
solve layer with a learned tensor metric ($2{,}355$ parameters) attains $R^2=1.000$ at 0.55K, 1.9K, and 6.9K nodes,
and $\tfrac12\log\det\Sigma$ follows $\log\kappa$ with Pearson correlation $0.9997$ and slope $0.983$. EGNN, the
most accurate baseline at the training size, attains $0.663$, $0.951$, and $0.752$.

\textbf{Velocity fluxes.}
With a stream-function readout $F=d_0\psi$, RHMP predicts the edge fluxes of the velocity with $R^2=0.995$ and zero net
flux through every triangle; the fluxes of GEM-CNN ($R^2=0.996$) leave a median net flux of $0.4\%$.

\subsection{Zero-Shot Transfer}
\label{app:zeroshot}

The fixed-mesh models of Table~\ref{tab:main} are evaluated without retraining on test data that differ from the
training data in mesh, geometry, gauge, or frame ($R^2$ on 100 test samples, normalization statistics of training).
On every new mesh the fields come from the generator of the original data set.
The NS and torus models keep the operator normalization of the training mesh on every mesh; cell-size descriptors are
expressed in training units when they lie inside the training range and relative to the mesh otherwise.

\textbf{Meshes and gauge.}
Table~\ref{tab:zeroshot_gauge} uses another random Delaunay mesh with 1024 nodes, the Delaunay triangulation of the
same nodes after a $3{:}1$ stretch, and a mesh with 4096 nodes; on the Wilson loop the test inputs are also gauge
transformed, $\theta\mapsto\theta+d_0\lambda$. RHMP stays at $R^2\ge0.9995$ on both tasks in every condition
and is unchanged by the gauge transformations to $7.6\times10^{-9}$.

\begin{table}[h]
\centering
\small
\caption{Zero-shot transfer on the gauge-field tasks ($R^2$, 100 test samples, seed $42$, no retraining). Orig.: training mesh; New: another random Delaunay mesh with 1024 nodes; Aniso.: the same nodes triangulated after a $3{:}1$ stretch; Fine: 4096 nodes; $s\lambda$: gauge transformation $\theta\mapsto\theta+d_0\lambda$ with $\lambda$ of standard deviation $s$ times the RMS of $\theta$. $\dagger$: edge-averaged targets as in Table~\ref{tab:main}. Values below $-1$ are shown as $<\!-1$. \textbf{Bold} = best.}
\label{tab:zeroshot_gauge}
\adjustbox{max width=\linewidth}{%
\begin{tabular}{l|ccccccc|cccc}
\toprule
 & \multicolumn{7}{c|}{$U(1)$ Wilson loop} & \multicolumn{4}{c}{$SU(2)$ Yang--Mills} \\
Method & Orig. & New & Aniso. & Fine & $0.1\lambda$ & $1\lambda$ & $10\lambda$ & Orig. & New & Aniso. & Fine \\
\midrule
GCN & $-$.0060 & $-$.0974 & $-$.2679 & $-$.0210 & $-$.0050 & $-$.4999 & $<\!-1$ & .029 & .029 & .028 & $-$.001 \\
GAT & .5237 & $-$.5283 & $-$.4943 & $-$.3501 & .5006 & $<\!-1$ & $<\!-1$ & .064 & .014 & $-$.003 & $-$.005 \\
SchNet & .8652 & $<\!-1$ & $<\!-1$ & $<\!-1$ & .8161 & $<\!-1$ & $<\!-1$ & .105 & .013 & .006 & $-$.016 \\
EGNN & .7618 & $<\!-1$ & $<\!-1$ & $<\!-1$ & .6646 & $<\!-1$ & $<\!-1$ & .020 & $-$.034 & $-$.037 & $-$.035 \\
MPSN$^\dagger$ & .1125 & $-$.2701 & $-$.2993 & $-$.1986 & .1018 & $-$.4176 & $<\!-1$ & .132 & .105 & .097 & .028 \\
SCCNN$^\dagger$ & .3600 & $<\!-1$ & $<\!-1$ & $<\!-1$ & .3420 & $<\!-1$ & $<\!-1$ & .135 & .093 & $-$.146 & $<\!-1$ \\
GaugeEqCNN & .9668 & .6335 & .3166 & $-$.0401 & .9632 & .1717 & $<\!-1$ & .601 & .538 & .436 & .575 \\
GEM-CNN & .9878 & $-$.3918 & $<\!-1$ & $-$.7016 & .9831 & .3727 & $<\!-1$ & .685 & .497 & .093 & .538 \\
CW Net & .8587 & $-$.0667 & $-$.4074 & $-$.3395 & .8349 & $<\!-1$ & $<\!-1$ & .351 & .320 & .290 & .337 \\
Cliff-SMPN$^\dagger$ & .5969 & $-$.3770 & $-$.6153 & $<\!-1$ & .5625 & $<\!-1$ & $<\!-1$ & .128 & .083 & .059 & .078 \\
DeepONet & .1361 & $-$.0429 & $-$.3605 & $-$.0968 & .1332 & $-$.2588 & $<\!-1$ & $-$.009 & $-$.010 & $-$.010 & $-$.010 \\
\textbf{RHMP} & \textbf{1.0000} & \textbf{.9998} & \textbf{.9999} & \textbf{.9995} & \textbf{1.0000} & \textbf{1.0000} & \textbf{1.0000} & \textbf{1.000} & \textbf{1.000} & \textbf{1.000} & \textbf{1.000} \\
\bottomrule
\end{tabular}
}
\end{table}

\textbf{Geometry.}
Table~\ref{tab:zeroshot_ellipsoid} uses the training ellipsoid at 4625, 577, and 290 nodes and two ellipsoids with
other semi-axes. RHMP keeps $R^2\ge0.9999$ in all five conditions.

\begin{table}[h]
\centering
\small
\caption{Zero-shot transfer on ellipsoid surface flow ($R^2$ with one variance for the three vector components, 100 test samples, seed $42$, no retraining). The training ellipsoid has semi-axes $1.4{:}1{:}0.7$ and 1157 nodes. $\dagger$: edge-averaged first component as in Table~\ref{tab:main}. \textbf{Bold} = best.}
\label{tab:zeroshot_ellipsoid}
\begin{tabular}{lcccccc}
\toprule
 & & \multicolumn{3}{c}{Same ellipsoid, nodes} & \multicolumn{2}{c}{Other semi-axes} \\
Method & Orig. & 4625 & 577 & 290 & $2.1{:}1{:}0.7$ & $1.6{:}1{:}0.5$ \\
\midrule
GCN & $-$.009 & $-$.009 & $-$.009 & $-$.009 & $-$.008 & $-$.009 \\
GAT & $-$.010 & $-$.009 & $-$.010 & $-$.011 & $-$.009 & $-$.010 \\
SchNet & $-$.010 & $-$.010 & $-$.010 & $-$.010 & $-$.009 & $-$.009 \\
EGNN & $-$.010 & $-$.010 & $-$.010 & $-$.010 & $-$.009 & $-$.009 \\
MPSN$^\dagger$ & $-$.014 & $-$.014 & $-$.014 & $-$.014 & $-$.013 & $-$.012 \\
SCCNN$^\dagger$ & $-$.009 & $<\!-1$ & $-$.169 & $<\!-1$ & $-$.011 & $-$.010 \\
GaugeEqCNN & .947 & $<\!-1$ & $<\!-1$ & $<\!-1$ & $<\!-1$ & .255 \\
GEM-CNN & .948 & $<\!-1$ & $<\!-1$ & $<\!-1$ & $<\!-1$ & .057 \\
CW Net & .242 & .088 & .187 & $-$.249 & .244 & .217 \\
Cliff-SMPN$^\dagger$ & $-$.008 & $-$.013 & $-$.018 & $-$.095 & $-$.008 & $-$.005 \\
DeepONet & .028 & .027 & .030 & .032 & .031 & .036 \\
\textbf{RHMP} & \textbf{1.000} & \textbf{1.000} & \textbf{1.000} & \textbf{1.000} & \textbf{1.000} & \textbf{1.000} \\
\bottomrule
\end{tabular}
\end{table}

\textbf{Stretched grids.}
Table~\ref{tab:zeroshot_ns} evaluates the NS models on a second stretched grid, on a uniform grid with $64^2$
nodes, and on the training grid rotated by $90^\circ$. RHMP obtains the highest $R^2$ in every condition.

\begin{table}[h]
\centering
\small
\caption{Zero-shot transfer on NS vorticity ($R^2$, 100 test samples, seed $42$, no retraining). Orig.: training grid; Stretched: a second stretched grid with two fine regions per axis; $64^2$: uniform grid with $64^2$ nodes; Rot.: training grid rotated by $90^\circ$. $\dagger$: edge-averaged targets as in Table~\ref{tab:main}. \textbf{Bold} = best.}
\label{tab:zeroshot_ns}
\begin{tabular}{lcccc}
\toprule
Method & Orig. & Stretched & $64^2$ & Rot. \\
\midrule
GCN & $-$.058 & $-$.103 & $-$.079 & $-$.057 \\
GAT & $-$.091 & $-$.159 & $-$.127 & $-$.055 \\
SchNet & $-$.049 & $-$.104 & $-$.192 & $-$.043 \\
EGNN & $-$.057 & $-$.104 & $-$.081 & $-$.054 \\
MPSN$^\dagger$ & $-$.062 & $-$.111 & $-$.091 & $-$.059 \\
SCCNN$^\dagger$ & $-$.062 & $-$.106 & $-$.079 & $-$.062 \\
GaugeEqCNN & .862 & .818 & .786 & .162 \\
GEM-CNN & .921 & .894 & .871 & $-$.045 \\
CW Net & .831 & .748 & .746 & .286 \\
Cliff-SMPN$^\dagger$ & .308 & .166 & .223 & $-$.008 \\
FNO & .555 & $-$.739 & $-$.732 & $-$.225 \\
DeepONet & $-$.049 & $-$.124 & $-$.087 & $-$.068 \\
\textbf{RHMP} & \textbf{.975} & \textbf{.971} & \textbf{.971} & \textbf{.975} \\
\bottomrule
\end{tabular}
\end{table}

\textbf{Transport parameters.}
Table~\ref{tab:zeroshot_torus} evaluates the torus models beyond the training ranges of the flow speed and of the
anisotropy ratio, on a roll-out to twice the training horizon, on finer and coarser tori, and under rigid rotations. RHMP obtains the highest $R^2$ in 9 of the 10 conditions.

\begin{table}[h]
\centering
\small
\caption{Zero-shot evaluation on torus advection--diffusion ($R^2$, seed $42$, no retraining). Orig.: 100 test samples; $\alpha$: toroidal speed beyond the training range $\alpha\le1.5$; $r$: anisotropy ratio beyond the training range $r\le100$; Roll-out: two applications of the model against the solution at twice the horizon; Fine and Coarse: the test instances solved on tori with 6844 and 450 vertices (450 test instances each); $30^\circ$, $90^\circ$: rigid rotations of the torus. $\dagger$: edge-averaged targets as in Table~\ref{tab:main}. \textbf{Bold} = best.}
\label{tab:zeroshot_torus}
\adjustbox{max width=\linewidth}{%
\begin{tabular}{lcccccccccc}
\toprule
Method & Orig. & $\alpha{=}2$ & $\alpha{=}2.5$ & $r{=}300$ & $r{=}1000$ & Roll-out & Fine & Coarse & $30^\circ$ & $90^\circ$ \\
\midrule
GCN & .667 & .224 & .094 & .663 & .636 & .249 & .631 & .638 & .666 & .653 \\
GAT & .827 & .509 & .311 & .792 & .734 & .562 & .706 & .528 & .803 & .682 \\
SchNet & .668 & .222 & .077 & .668 & .637 & .258 & .626 & .585 & .668 & .665 \\
EGNN & .672 & .218 & .012 & .672 & .643 & .243 & .621 & .548 & .672 & .661 \\
MPSN$^\dagger$ & .668 & .221 & .082 & .662 & .636 & .244 & .631 & .659 & .667 & .656 \\
SCCNN$^\dagger$ & .669 & .228 & .085 & .660 & .631 & .258 & .631 & .609 & .668 & .655 \\
GaugeEqCNN & .934 & .727 & .489 & .872 & .786 & .773 & .800 & .606 & .888 & .753 \\
GEM-CNN & .955 & .775 & .530 & .888 & \textbf{.790} & .857 & .785 & .542 & .820 & .700 \\
CW Net & .890 & .522 & .215 & .847 & .765 & .361 & .765 & .561 & .879 & .821 \\
Cliff-SMPN$^\dagger$ & .831 & .519 & .270 & .778 & .701 & .453 & .669 & .457 & .801 & .646 \\
DeepONet & $-$.012 & $-$.002 & $-$.002 & $-$.002 & $-$.003 & $-$.003 & $-$.002 & $-$.002 & $-$.011 & $-$.010 \\
\textbf{RHMP} & \textbf{.985} & \textbf{.919} & \textbf{.804} & \textbf{.905} & .772 & \textbf{.948} & \textbf{.958} & \textbf{.934} & \textbf{.985} & \textbf{.985} \\
\bottomrule
\end{tabular}
}
\end{table}

\textbf{Currents, cores, and meshes.}
Table~\ref{tab:zeroshot_maxwell} scales the coil currents of the Maxwell task by $1.5$, uses a core family absent
from training, refines and coarsens the mesh, and rotates the device. RHMP obtains the highest $R^2$ in every condition.

\begin{table}[h]
\centering
\small
\caption{Zero-shot transfer on Maxwell magnetostatics ($R^2$, 100 test samples, seed $42$, no retraining). Orig.: training distribution; $1.5I$: coil currents scaled by $1.5$; Held-out: E--E inductor cores, a family absent from training; Fine and Coarse: meshes with 7255 and 459 nodes; $30^\circ$, $90^\circ$: rigid rotations of the device. $\dagger$: edge-averaged targets as in Table~\ref{tab:main}. \textbf{Bold} = best.}
\label{tab:zeroshot_maxwell}
\begin{tabular}{lccccccc}
\toprule
Method & Orig. & $1.5I$ & Held-out & Fine & Coarse & $30^\circ$ & $90^\circ$ \\
\midrule
GCN & .002 & .002 & $-$.008 & $-$.008 & $-$.001 & $-$.006 & $-$.014 \\
GAT & .037 & .037 & $-$.026 & $-$.003 & $-$.027 & $-$.001 & $-$.064 \\
SchNet & $-$.004 & $-$.005 & $-$.005 & $-$.007 & $-$.005 & $-$.006 & $-$.007 \\
EGNN & .051 & .049 & $-$.033 & $-$.002 & $-$.049 & .025 & $-$.071 \\
MPSN$^\dagger$ & .065 & .056 & $-$.051 & .008 & $-$.053 & .008 & $-$.086 \\
SCCNN$^\dagger$ & .128 & .131 & $-$.120 & .025 & $-$.163 & .028 & $-$.104 \\
GaugeEqCNN & .344 & .363 & .395 & .116 & .379 & .262 & .215 \\
GEM-CNN & .344 & .363 & .397 & .109 & .327 & .067 & .069 \\
CW Net & .233 & .250 & .239 & .075 & .298 & .166 & $-$.249 \\
Cliff-SMPN$^\dagger$ & .073 & .087 & .046 & .011 & .033 & .052 & $-$.060 \\
DeepONet & $-$.005 & $-$.006 & $-$.004 & $-$.006 & $-$.004 & $-$.006 & $-$.005 \\
\textbf{RHMP} & \textbf{.946} & \textbf{.947} & \textbf{.890} & \textbf{.874} & \textbf{.944} & \textbf{.946} & \textbf{.946} \\
\bottomrule
\end{tabular}
\end{table}

\textbf{Frame.}
Table~\ref{tab:zeroshot_frame} rotates the test meshes rigidly together with all vector-valued inputs and targets.
RHMP is unchanged to round-off.

\begin{table}[h]
\centering
\small
\caption{Rigid rotation of the test meshes ($R^2$, 100 test samples, seed $42$, no retraining). Planar meshes are rotated about their centroid, the ellipsoid about a fixed oblique axis; vector-valued inputs and targets rotate with the mesh. $\dagger$: edge-averaged targets as in Table~\ref{tab:main}. Values below $-1$ are shown as $<\!-1$. \textbf{Bold} = best.}
\label{tab:zeroshot_frame}
\adjustbox{max width=\linewidth}{%
\begin{tabular}{l|ccc|ccc|ccc|ccc}
\toprule
 & \multicolumn{3}{c|}{Ellipsoid flow} & \multicolumn{3}{c|}{Wilson loop} & \multicolumn{3}{c|}{Yang--Mills} & \multicolumn{3}{c}{AirfRANS pressure} \\
Method & $0^\circ$ & $30^\circ$ & $90^\circ$ & $0^\circ$ & $30^\circ$ & $90^\circ$ & $0^\circ$ & $30^\circ$ & $90^\circ$ & $0^\circ$ & $30^\circ$ & $90^\circ$ \\
\midrule
GCN & $-$.009 & $-$.009 & $-$.009 & $-$.0060 & $-$.0337 & $-$.0576 & .029 & .029 & .028 & .028 & .028 & .028 \\
GAT & $-$.010 & $-$.010 & $-$.010 & .5237 & .0057 & $-$.2073 & .064 & .042 & .045 & $-$.023 & $-$.023 & $-$.023 \\
SchNet & $-$.010 & $-$.010 & $-$.010 & .8652 & .5530 & $-$.7015 & .105 & .088 & .004 & .023 & .023 & .023 \\
EGNN & $-$.010 & $-$.010 & $-$.009 & .7619 & .7619 & .7619 & .020 & .020 & .020 & .036 & .036 & .036 \\
MPSN$^\dagger$ & $-$.014 & $-$.011 & $-$.006 & .1125 & $-$.2260 & $-$.1830 & .132 & .123 & .071 & .015 & .015 & .015 \\
SCCNN$^\dagger$ & $-$.009 & $-$.010 & $-$.016 & .3600 & $-$.3204 & $<\!-1$ & .135 & .121 & .065 & .030 & .030 & .030 \\
GaugeEqCNN & .947 & $<\!-1$ & $<\!-1$ & .9668 & .7873 & .5722 & .601 & .549 & .537 & .820 & $-$.663 & $-$.745 \\
GEM-CNN & .948 & $<\!-1$ & $<\!-1$ & .9878 & .4934 & .0409 & .685 & .424 & .471 & .890 & $-$.416 & $-$.709 \\
CW Net & .242 & .205 & $-$.032 & .8587 & .3411 & .1079 & .351 & .342 & .320 & .041 & .041 & .041 \\
Cliff-SMPN$^\dagger$ & $-$.008 & $-$.007 & $-$.015 & .5969 & $-$.0029 & $-$.1971 & .128 & .111 & .079 & .026 & .026 & .026 \\
DeepONet & .028 & .016 & $-$.099 & .1361 & .0847 & .1016 & $-$.009 & $-$.009 & $-$.009 & -- & -- & -- \\
\textbf{RHMP} & \textbf{1.000} & \textbf{1.000} & \textbf{1.000} & \textbf{1.0000} & \textbf{1.0000} & \textbf{1.0000} & \textbf{1.000} & \textbf{1.000} & \textbf{1.000} & \textbf{.995} & \textbf{.995} & \textbf{.995} \\
\bottomrule
\end{tabular}
}
\end{table}

\subsection{Robustness Checks}

We probe two robustness axes that bear directly on the architecture's claims, on the tasks where each axis is most diagnostic.
(i) \textbf{Resolution sensitivity}:
on $U(1)$ Wilson loop the model trained on $1024$ nodes is applied without retraining to a mesh with $4096$ nodes,
and the $R^2$ of RHMP remains $0.9995$ (§\ref{app:zeroshot}).
(ii) \textbf{Parameter OOD}:
on Maxwell magnetostatics the coil currents of the test set are scaled by $1.5$ beyond the training range,
and the $R^2$ of RHMP changes from $0.946$ to $0.947$ (§\ref{app:zeroshot}).

\section{Full Proof of Theorem~\ref{thm:equivariance}}
\label{app:proofs}

\noindent\textbf{Conventions.}
$\mathbf{x}_k \in \R^{n_k \times C}$ is in row-vector format (each row is a $C$-dimensional feature of a $k$-cell).
$R \in \OC$ acts on the channel (column) index via right-multiplication $\mathbf{x}_k \mapsto \mathbf{x}_k R^\top$, with the same $R$ applied to all cells and cochain degrees.
We \emph{use normalization layers without learnable affine (scale/shift) parameters}: the $\OC$ action treats all channel bases symmetrically, and a per-channel learnable $\gamma$ or $\beta$ would select a preferred direction in channel space. This is the condition used in our implementation (Section~\ref{sec:hodge_mp}).

\begin{lemma}[Positive semidefiniteness]
\label{lem:psd}
If $A\succeq 0$ and $B\succeq 0$, then
$L_k=d_{k-1}A d_{k-1}^\top+d_k^\top B d_k$ is positive semidefinite.
If $A\succ0$ and $B\succ0$, this holds in particular for the learned metric
matrices used by RHMP.
\end{lemma}

\begin{proof}
For any $\mathbf{x}\in\R^{n_k\times C}$,
\[
  \tr(\mathbf{x}^\top L_k\mathbf{x})
  = \|d_{k-1}^\top\mathbf{x}\|_{A}^{2}
  + \|d_k\mathbf{x}\|_{B}^{2}
  \ge 0.
\]
The same argument applies channel by channel and then sums over channels.
\end{proof}

\begin{lemma}[Cochain-frame equivariance]
\label{lem:oc}
Let $\Phi^{(\ell)}$ denote one layer of the architecture in
Theorem~\ref{thm:equivariance}, composed of metric-weighted Hodge operators,
cochain-frame-invariant metrics $H_k$, norm-gated activation
$\sigma_{\mathrm{gate}}$, per-cell RMSNorm (without learnable affine
parameters), and a residual connection. Then for all $R \in \OC$,
\[
  \Phi^{(\ell)}(\mathbf{x} R^\top) = \Phi^{(\ell)}(\mathbf{x})\, R^\top,
  \qquad H_k(\mathbf{x} R^\top) = H_k(\mathbf{x}).
\]
\end{lemma}

\begin{proof}
The proof rests on the following standard algebraic fact,
which we state explicitly for clarity.

\begin{lemma}[Commutation of cell and channel actions]
\label{lem:kronecker}
Let $\mathbf{x} \in \R^{m \times C}$,
$A \in \R^{m' \times m}$ acting on the row (cell) index,
$B \in \R^{C \times C}$ acting on the column (channel) index.
Associativity gives
\begin{equation}
  A\,(\mathbf{x}\, B^\top)
  = (A\, \mathbf{x})\, B^\top.
  \label{eq:kronecker}
\end{equation}
Actions on the cell and channel indices commute.
\end{lemma}

We now verify each component of the layer.

\textbf{Linear Hodge operators.}
Each $M_j \in \R^{n_k \times n_{k'}}$
(e.g.\ $d_{k-1}H_{k-1}^{\downarrow}d_{k-1}^\top$ or $d_k^\top H_{k+1}^{\uparrow}d_k$)
acts on the cell index, while $R^\top \in \R^{C \times C}$ acts on the
channel index.
By Lemma~\ref{lem:kronecker}:
\begin{equation}
  M_j\,(\mathbf{x}\, R^\top) = (M_j\, \mathbf{x})\, R^\top.
  \label{eq:mixed_product}
\end{equation}

\textbf{Frame-invariance of $H_k$.}
$H_k$ is driven by the row-wise statistics
$\psi(\mathbf{x}_k) = \{\|\mathbf{x}_i\|_2^2,\; \mathbf{x}_i \mathbf{x}_j^\top : j\sim i\}$ together with geometric invariants.
For the Frobenius norm, the cyclic property of trace and orthogonality
$R^\top R = I$ give
\[
  \|\mathbf{x}_k R^\top\|_F^2
  = \tr\!\bigl(R\,\mathbf{x}_k^\top \mathbf{x}_k\, R^\top\bigr)
  \overset{\mathrm{cyc}}{=}
  \tr\!\bigl(\mathbf{x}_k^\top \mathbf{x}_k\, R^\top R\bigr)
  = \tr(\mathbf{x}_k^\top \mathbf{x}_k)
  = \|\mathbf{x}_k\|_F^2.
\]
For the per-cell inner product, with $\mathbf{x}_i \in \R^{1 \times C}$
denoting the $i$-th row:
$(\mathbf{x}_i R^\top)(\mathbf{x}_j R^\top)^\top
= \mathbf{x}_i R^\top R\, \mathbf{x}_j^\top
= \mathbf{x}_i \mathbf{x}_j^\top$.
Hence $\psi(\mathbf{x} R^\top) = \psi(\mathbf{x})$,
and since $H_k = f(\psi(\mathbf{x}))$ for a deterministic $f$,
we obtain $H_k(\mathbf{x} R^\top) = H_k(\mathbf{x})$.

\textbf{Norm-gated activation.}
For the row-wise norm
$\|\mathbf{m}_i R^\top\|_2 = \|\mathbf{m}_i\|_2$
(since $R$ is orthogonal), we have
$\sigma_{\mathrm{gate}}(\mathbf{m}\, R^\top)
= g(\|\mathbf{m}\, R^\top\|)\,\mathbf{m}\, R^\top
= g(\|\mathbf{m}\|)\,\mathbf{m}\, R^\top
= \sigma_{\mathrm{gate}}(\mathbf{m})\, R^\top$.

\textbf{Per-cell RMSNorm (without learnable affine parameters).}
$\mathrm{RMSNorm}(\mathbf{x}_i)
= \mathbf{x}_i / \sqrt{\|\mathbf{x}_i\|_2^2/C+\epsilon_N}$,
where $\epsilon_N>0$ also defines the normalization at zero.
Then:
\[
  \mathrm{RMSNorm}(\mathbf{x}_i R^\top)
  = \frac{\mathbf{x}_i R^\top}{\sqrt{\|\mathbf{x}_i R^\top\|_2^2/C+\epsilon_N}}
  = \frac{\mathbf{x}_i R^\top}{\sqrt{\|\mathbf{x}_i\|_2^2/C+\epsilon_N}}
  = \mathrm{RMSNorm}(\mathbf{x}_i)\, R^\top.
\]
\emph{Remark.}\
Standard LayerNorm with mean-subtraction uses the all-ones vector
$\mathbf{1} \in \R^C$ as a preferred channel direction; for a general
$R\in\OC$, $\sum_c (\mathbf{x}_i R^\top)_c \neq \sum_c x_{ic}$.
Learnable scale $\gamma \in \R^C$ or shift $\beta \in \R^C$ similarly
selects a channel basis.
Our implementation therefore uses RMSNorm without learnable affine parameters;
the measured equivariance error is $4 \times 10^{-7}$
(Table~\ref{tab:symmetry_full}).

\textbf{Residual connection.}
$(\mathbf{x} + \sigma(M\mathbf{x}))\, R^\top
= \mathbf{x}\, R^\top + \sigma(M\mathbf{x})\, R^\top$,
since right-multiplication by $R^\top$ distributes over addition.
Composing all components, each layer is $\OC$-equivariant.
\end{proof}

\begin{lemma}[$\En$-invariance and equivariance]
\label{lem:en}
If $H_k$ depends only on $\En$-invariant quantities of $K$'s embedding
(assumption 3 of Theorem~\ref{thm:equivariance}),
then all cochain features $\mathbf{x}_k^{(\ell)}$ are $\En$-invariant,
scalar readouts are $\En$-invariant,
and vector reconstructions are $\En$-equivariant.
\end{lemma}

\begin{proof}
Let $T = (R_s, t) \in \En$ act on the embedding by
$\mathbf{r}_v \mapsto R_s \mathbf{r}_v + t$.

\textbf{Lifting encoder (base case).}
All encoder inputs are $\En$-invariant:
edge lengths $\ell_e = \|\mathbf{r}_j - \mathbf{r}_i\|$ are preserved
because $R_s$ preserves norms and $t$ cancels in differences;
vertex degrees $\deg(v)$ are combinatorial;
face areas $A_f = \tfrac{1}{2}\|(\mathbf{r}_1 - \mathbf{r}_0)
\times (\mathbf{r}_2 - \mathbf{r}_0)\|$ are preserved because
for $R_s \in \On$,
$R_s \mathbf{a} \times R_s \mathbf{b} = (\det R_s)\, R_s(\mathbf{a} \times \mathbf{b})$,
so $\|R_s \mathbf{a} \times R_s \mathbf{b}\|
= |\det R_s| \cdot \|\mathbf{a} \times \mathbf{b}\|
= \|\mathbf{a} \times \mathbf{b}\|$
(the norm absorbs the sign under reflections);
interior angles are determined by normalized inner products, preserved by $R_s$.
Therefore $\mathbf{x}_k^{(0)}$ is $\En$-invariant.

\textbf{Inductive step (layer $\ell \to \ell+1$).}
Assume $\mathbf{x}_k^{(\ell)}$ is $\En$-invariant. Then:
\begin{enumerate}
\item[(a)] $\psi(\mathbf{x}_k^{(\ell)})$ is invariant
  (norms and inner products of invariant vectors are invariant);
\item[(b)] $H_k^{(\ell)} = f(\psi(\mathbf{x}_k^{(\ell)}))$ is invariant
  ($f$ is a deterministic function of invariant inputs);
\item[(c)] $M_j \mathbf{x}_k^{(\ell)}$ is invariant
  ($d_k$ is a topological operator independent of the embedding,
  and $H_k^{(\ell)}$ is invariant by (b));
\item[(d)] $\sigma_{\mathrm{gate}}(M_j \mathbf{x}_k^{(\ell)})$ is invariant
  ($g(\|\mathbf{m}\|)\,\mathbf{m}$ with $\|\mathbf{m}\|$ and $\mathbf{m}$ both invariant);
\item[(e)] $\mathrm{RMSNorm}(\cdot)$ preserves invariance
  ($\|\mathbf{x}_i\|$ is invariant, so the normalized output inherits invariance);
\item[(f)] the residual sum $\mathbf{x}_k^{(\ell)} + (\text{gated, normalized update})$ is invariant
  (a sum of invariant quantities is invariant).
\end{enumerate}
By induction, $\mathbf{x}_k^{(\ell)}$ is $\En$-invariant for all $\ell$.

\textbf{Scalar readout.}
$\hat{y}(v) = h(\mathbf{x}_0^{(L)}(v))$ is a function of invariant inputs,
hence invariant.

\textbf{Vector reconstruction.}
$\hat{\mathbf{v}}(v) = \tfrac{1}{\deg(v)} \sum_{e \ni v}
w(\mathbf{x}_1(e),\, \|\mathbf{r}_{ij}\|^2,\, \|\mathbf{x}_0^i\|^2,\, \|\mathbf{x}_0^j\|^2)
\cdot (\mathbf{r}_j - \mathbf{r}_i)$.
All arguments of $w$ are $\En$-invariant, so $w$ is invariant.
The displacement $\mathbf{r}_j - \mathbf{r}_i \mapsto R_s(\mathbf{r}_j - \mathbf{r}_i)$
is $\En$-equivariant (translations cancel).
The product of an invariant scalar and an equivariant vector is equivariant:
$\hat{\mathbf{v}} \mapsto R_s \hat{\mathbf{v}}$.
\end{proof}

\textbf{Proof that $d^2 = 0$.}
The identity $d_1 d_0 = 0$ is a purely algebraic consequence of the
CW complex structure. Each column of $d_0$ is indexed by a vertex $v$
and records $\pm 1$ on the edges incident to $v$.
Each row of $d_1$ is indexed by a face $f$ and records $\pm 1$ on the
boundary edges of $f$ (with signs from the chosen orientation).
The $(f, v)$-entry of $d_1 d_0$ equals
$\sum_{e} (d_1)_{f,e}\,(d_0)_{e,v}$.
If $v \notin \partial f$, every term vanishes.
If $v \in \partial f$, exactly two boundary edges of $f$ are incident to $v$.
The boundary orientation makes the two signed incidence products opposite in
sign, so their contributions cancel.  This is the cellular statement that
the boundary of an oriented boundary is empty.  The cancellation holds for
every $(f,v)$ pair,
hence $d_1 d_0 = 0$ exactly. The operators $d_0, d_1$ are frozen
throughout training, so the identity is preserved at all times.

\section{Symmetry Verification Protocol}
\label{app:symmetry_detail}

Each symmetry subgroup test selects its scope according to the natural object of action.
The spatial group $\En$ and the $\Z_2$ edge-orientation act on the cell complex $K$ itself:
translation/rotation/reflection act on the vertex coordinates of $K$ after which $K$ is rebuilt,
and the $\Z_2$ edge orientation is realized by flipping the endpoints of a single edge
and synchronously updating the signs of $d_0[e,:]$ and $d_1[:,e]$.
Fiber $\OC$-type symmetries act on the $C$-dimensional cochain channels,
whose natural test object is the single-layer weighted Hodge operator described in Theorem~\ref{thm:equivariance};
for fair comparison with baselines,
for each baseline we extract its corresponding message passing module,
apply the channel-space transformation to a random hidden state $h_0 \in \R^{n_0 \times C}$, and compare outputs.
Let $\Phi$ denote the operator under test and $\hat T$ the action of $T$ on the output side
(identity for invariance tests, $T^{-1}$ for equivariance tests); then the reported quantity is
\[
  \mathrm{err}(T) = \frac{\| \hat T \, \Phi(T \!\cdot\! \mathrm{input}) - \Phi(\mathrm{input}) \|_F}
                         {\| \Phi(\mathrm{input}) \|_F},
\]
and each (model, subgroup) cell is averaged over $5$ independent random transformations.

We use two test meshes to balance comparability and generality:
the primary test mesh is a $4\!\times\!4$ regular Delaunay
($(n_0, n_1, n_2) = (16, 33, 18)$),
whose regular structure allows FNO to participate on its native grid as well;
the auxiliary test mesh is a Delaunay of $6$ uniformly random points
($(n_0, n_1, n_2) = (6, 11, 6)$),
used to verify the robustness of the same protocol on irregular connectivity,
with full results listed in Table~\ref{tab:sym_comparison_irr}.
Both test meshes use $C=32$ channels, $L=2$ layers, evaluation mode, and initialization seed $42$.
The specific distributions of the transformations are as follows:
translations $v \sim \mathcal{N}(0, 0.01\, I_2)$,
rotations $R(\theta)$ with $\theta \sim \mathrm{Unif}[0, 2\pi)$,
reflections via random-direction Householder;
for $\OC$, $Q$ is taken from the QR decomposition of a $32\!\times\!32$ Gaussian matrix
(for $\SOC$, a column is flipped if $\det Q < 0$),
$S_C$ is a random permutation of the $32$ channels,
and $(\Z_2)^C$ is a random sign vector in $\{\pm 1\}^{32}$.
We adopt $\mathrm{err} < 10^{-3}$ as the numerical pass threshold,
display $\mathrm{err} < 10^{-5}$ (machine precision) as $\sim 0$, and report numerical values otherwise.

The final structural entries, $d^2\!=\!0$ and Abelian $U(1)$ curvature invariance, are reported categorically. A model using the fixed coboundary $K.d_k$ satisfies $d^2\!=\!0$ exactly. A model that performs message passing on 1-cochain connections inherits exact invariance of $F=dA$ under $A\mapsto A+d\lambda$, complementing the $\OC$ cochain-frame action tested above. For methods without cochain-frame design, the $\OC$ columns simply record the absence of this inductive bias. N/A denotes settings outside a method's native domain: under its FFT parameterization, FNO is tied to regular grids, and DeepONet's branch--trunk outer-product structure has no separate channel-space message layer to test.

\begin{table}[h]
  \caption{Relative $L_2$ error of Riemannian Hodge Message Passing on each symmetry subgroup (primary test mesh, averaged over 5 independent random transformations). $d^2\!=\!0$ holds exactly by the algebraic construction of the cell complex; $U(1)$ denotes Abelian curvature invariance under exact shifts of 1-cochain connections, separate from the $\OC$ cochain-frame action.}
  \label{tab:symmetry_full}
  \centering
  \small
  \adjustbox{max width=\linewidth}{%
  \begin{tabular}{llcc}
    \toprule
    Level & Subgroup & Inclusion & Error \\
    \midrule
    \multirow{3}{*}{Spatial}
    & $\Tn$ translation & $\Tn \subset \En$ & $3\!\times\!10^{-7}$ \\
    & $\SOn$ rotation & $\SOn \subset \On \subset \En$ & $3\!\times\!10^{-7}$ \\
    & $\On$ reflection & $\On \subset \En$ & $3\!\times\!10^{-7}$ \\
    \midrule
    \multirow{4}{*}{Cochain-frame}
    & $\OC$ rotation & & $4\!\times\!10^{-7}$ \\
    & $\SOC$ rotation & $\SOC \subset \OC$ & $4\!\times\!10^{-7}$ \\
    & $S_C$ permutation & $S_C \subset \OC$ & $3\!\times\!10^{-8}$ \\
    & $(\Z_2)^C$ sign & $(\Z_2)^C \subset \OC$ & $0$ \\
    \midrule
    \multirow{2}{*}{Topological}
    & $d^2\!=\!0$ & & $0$ (algebraic) \\
    & $\Z_2$ orientation & & $\sim\!0$ \\
    \midrule
    Abelian $U(1)$
    & curvature invariance & & exact via $d^2{=}0$ \\
    \bottomrule
  \end{tabular}
  }
\end{table}

\begin{table}[h]
  \caption{Relative $L_2$ error of each method on ten symmetry subgroups, averaged over 5 random transformations, with primary test mesh a $4\!\times\!4$ Delaunay. $\sim\!0$ denotes error below $10^{-5}$ (near floating-point precision). The $d^2\!=\!0$ column is the fixed-coboundary topology check (marked \emph{exact}); the $U(1)$ column is Abelian curvature invariance under exact shifts of 1-cochain connections (marked $\checkmark$). Other entries are marked N/A. See the protocol in §\ref{app:symmetry_detail}.}
  \label{tab:sym_comparison}
  \centering
  \footnotesize
  \setlength{\tabcolsep}{2.5pt}
  \adjustbox{max width=\linewidth}{%
  \begin{tabular}{l|ccc|c|cccc|cc}
    \toprule
    & \multicolumn{3}{c|}{Spatial $\En$} & Topo & \multicolumn{4}{c|}{Cochain-frame $\OC$} & Topo & Abelian \\
    Model & $\Tn$ & $\SOn$ & $\On$ & $\Z_2$ or. & $\OC$ & $\SOC$ & $S_C$ & $(\Z_2)^{\!C}$ & $d^2$ & $U(1)$ \\
    \midrule
    GCN        & $\sim\!0$ & $\sim\!0$ & $\sim\!0$ & $\sim\!0$ & 1.56 & 1.58 & 1.61 & 1.46 & N/A & N/A \\
    GAT        & $\sim\!0$ & $\sim\!0$ & $\sim\!0$ & $\sim\!0$ & 0.68 & 0.68 & 0.70 & 0.68 & N/A & N/A \\
    SchNet     & $\sim\!0$ & $\sim\!0$ & $\sim\!0$ & $\sim\!0$ & 0.40 & 0.40 & 0.30 & 0.37 & N/A & N/A \\
    EGNN       & $\sim\!0$ & $\sim\!0$ & $\sim\!0$ & $\sim\!0$ & 0.43 & 0.43 & 0.35 & 0.40 & N/A & N/A \\
    MPSN       & $\sim\!0$ & $\sim\!0$ & $\sim\!0$ & $\sim\!0$ & 1.31 & 1.30 & 1.21 & 1.23 & N/A & $\checkmark$ \\
    SCCNN      & $\sim\!0$ & $\sim\!0$ & $\sim\!0$ & $\sim\!0$ & 1.24 & 1.25 & 1.08 & 1.31 & exact & $\checkmark$ \\
    GaugeEqCNN & $\sim\!0$ & 0.02 & 0.03 & 0.08 & 1.53 & 1.52 & 1.52 & 1.39 & N/A & N/A \\
    GEM-CNN    & $\sim\!0$ & 0.16 & 0.17 & 0.07 & 1.40 & 1.39 & 1.36 & 1.35 & N/A & N/A \\
    CW Net     & $\sim\!0$ & $\sim\!0$ & $\sim\!0$ & 0.09 & 1.38 & 1.39 & 1.23 & 1.35 & exact & $\checkmark$ \\
    Cliff-SMPN & $\sim\!0$ & $\sim\!0$ & $\sim\!0$ & 0.29 & 1.37 & 1.36 & 1.35 & 1.33 & exact & $\checkmark$ \\
    FNO        & $\sim\!0$ & $\sim\!0$ & $\sim\!0$ & $\sim\!0$ & 1.40 & 1.41 & 1.25 & 1.40 & N/A & N/A \\
    DeepONet   & 0.01 & 0.10 & 0.09 & N/A & N/A & N/A & N/A & N/A & N/A & N/A \\
    \midrule
    \textbf{RHMP} & $\sim\!0$ & $\sim\!0$ & $\sim\!0$ & $\sim\!0$ & $\sim\!0$ & $\sim\!0$ & $\sim\!0$ & $\sim\!0$ & exact & $\checkmark$ \\
    \bottomrule
  \end{tabular}
  }
\end{table}

\begin{table}[h]
  \caption{Results of the same protocol as Table~\ref{tab:sym_comparison} on the auxiliary test mesh (6-point irregular Delaunay, $n_0=6, n_1=11, n_2=6$). The final $U(1)$ column again denotes Abelian curvature invariance under exact shifts of 1-cochain connections. Under its FFT parameterization, FNO is tied to regular grids.}
  \label{tab:sym_comparison_irr}
  \centering
  \footnotesize
  \setlength{\tabcolsep}{2.5pt}
  \adjustbox{max width=\linewidth}{%
  \begin{tabular}{l|ccc|c|cccc|cc}
    \toprule
    & \multicolumn{3}{c|}{Spatial $\En$} & Topo & \multicolumn{4}{c|}{Cochain-frame $\OC$} & Topo & Abelian \\
    Model & $\Tn$ & $\SOn$ & $\On$ & $\Z_2$ or. & $\OC$ & $\SOC$ & $S_C$ & $(\Z_2)^{\!C}$ & $d^2$ & $U(1)$ \\
    \midrule
    GCN        & $\sim\!0$ & $\sim\!0$ & $\sim\!0$ & $\sim\!0$ & 1.60 & 1.59 & 1.72 & 1.57 & N/A & N/A \\
    GAT        & $\sim\!0$ & $\sim\!0$ & $\sim\!0$ & $\sim\!0$ & 0.70 & 0.72 & 0.71 & 0.72 & N/A & N/A \\
    SchNet     & $\sim\!0$ & $\sim\!0$ & $\sim\!0$ & $\sim\!0$ & 0.44 & 0.44 & 0.34 & 0.43 & N/A & N/A \\
    EGNN       & $\sim\!0$ & $\sim\!0$ & $\sim\!0$ & $\sim\!0$ & 0.44 & 0.44 & 0.35 & 0.42 & N/A & N/A \\
    MPSN       & $\sim\!0$ & $\sim\!0$ & $\sim\!0$ & $\sim\!0$ & 1.29 & 1.28 & 1.24 & 1.18 & N/A & $\checkmark$ \\
    SCCNN      & $\sim\!0$ & $\sim\!0$ & $\sim\!0$ & $\sim\!0$ & 1.56 & 1.57 & 1.22 & 1.56 & exact & $\checkmark$ \\
    GaugeEqCNN & $\sim\!0$ & 0.02 & 0.03 & 0.22 & 1.62 & 1.61 & 1.58 & 1.43 & N/A & N/A \\
    GEM-CNN    & $\sim\!0$ & 0.11 & 0.12 & 0.18 & 1.44 & 1.43 & 1.38 & 1.34 & N/A & N/A \\
    CW Net     & $\sim\!0$ & $\sim\!0$ & $\sim\!0$ & 0.15 & 1.61 & 1.63 & 1.46 & 1.61 & exact & $\checkmark$ \\
    Cliff-SMPN & $\sim\!0$ & $\sim\!0$ & $\sim\!0$ & 0.51 & 1.29 & 1.28 & 1.32 & 1.31 & exact & $\checkmark$ \\
    FNO        & N/A & N/A & N/A & N/A & N/A & N/A & N/A & N/A & N/A & N/A \\
    DeepONet   & 0.02 & 0.13 & 0.09 & N/A & N/A & N/A & N/A & N/A & N/A & N/A \\
    \midrule
    \textbf{RHMP} & $\sim\!0$ & $\sim\!0$ & $\sim\!0$ & $\sim\!0$ & $\sim\!0$ & $\sim\!0$ & $\sim\!0$ & $\sim\!0$ & exact & $\checkmark$ \\
    \bottomrule
  \end{tabular}
  }
\end{table}


\section{Discrete Curvature Comparison}
\label{app:curvature}

\subsection{Ollivier-Ricci Curvature}
Ollivier \cite{ollivier2009ricci} defines Ricci curvature on graphs via optimal transport:
$\kappa(x,y) = 1 - W_1(\mu_x, \mu_y) / d(x,y)$.
It is a scalar on edges, whereas $H_k$ is a matrix on the entire set of $k$-cells,
providing a richer geometric parameterization.

\subsection{Forman-Ricci Curvature}
Forman \cite{forman2003bochner} defined a discrete Ricci curvature $F(e)$ for CW complexes
and established a discrete Bochner--Weitzenb\"ock inequality.
On a weighted graph $(w_v, w_e)$, the Forman curvature involves ratios such as $w_e / w_v$,
while the diagonal entries of $H_0, H_1$ play exactly the roles of $w_v, w_e$.
Both constructions express discrete geometry through cell weights:
Forman curvature evaluates fixed combinatorial weights,
while $H_k$ learns the weights that enter $L_k^H$ from the data.

\subsection{Tangent Bundle Convolution}
Battiloro et al.~\cite{battiloro2023tangent} realize geometry-aware message passing
via convolution on the tangent bundle.
The metric approach of RHMP can be viewed as a complementary dual:
it represents geometry through inner-product weights on cochains.

\begin{samepage}
\section{Expressivity Analysis}
\label{app:expressivity}

Symmetry constraints shrink the hypothesis space by ruling out functions that
violate the physical structure.  The learnable metric still provides local
geometric flexibility.  The following proposition states this for the
same-degree weighted Hodge operator used in Eq.~\eqref{eq:hodge_lap}.

\begin{proposition}[First-order spectral response to metric perturbation]
\label{prop:expressivity}
Let $K$ be a fixed regular CW complex and fix $k$.  For
$A\in\spd(n_{k-1})$ and $B\in\spd(n_{k+1})$, define
\begin{equation}
  L_k(A,B) = d_{k-1}A d_{k-1}^\top + d_k^\top B d_k.
  \label{eq:expressivity_operator}
\end{equation}
Take the base point $(A,B)=(I,I)$ and write
$L_k^0=d_{k-1}d_{k-1}^\top+d_k^\top d_k$.  Let $\lambda_j$ be a simple
eigenvalue of $L_k^0$ with unit eigenvector $\mathbf{v}_j\in\R^{n_k}$.
Then:

\textnormal{(1)} The map $(A,B)\mapsto \lambda_j(A,B)$ is differentiable in
a neighborhood of $(I,I)$.

\textnormal{(2)} For symmetric perturbations
$\delta A\in\mathrm{Sym}(n_{k-1})$ and
$\delta B\in\mathrm{Sym}(n_{k+1})$,
\begin{equation}
  D\lambda_j(I,I)[\delta A,\delta B]
  = (d_{k-1}^\top\mathbf{v}_j)^\top\delta A(d_{k-1}^\top\mathbf{v}_j)
  + (d_k\mathbf{v}_j)^\top\delta B(d_k\mathbf{v}_j).
  \label{eq:spectral_response}
\end{equation}

\textnormal{(3)} If $\mathbf{v}_j$ is non-harmonic, i.e.
$d_{k-1}^\top\mathbf{v}_j\neq0$ or $d_k\mathbf{v}_j\neq0$, then there is a
symmetric metric perturbation direction for which
$D\lambda_j\neq0$.  Under diagonal metrics
$\delta A=\diag(\boldsymbol{\eta})$ and
$\delta B=\diag(\boldsymbol{\zeta})$,
\begin{equation}
  D\lambda_j(I,I)[\delta A,\delta B]
  = \sum_{a=1}^{n_{k-1}}\eta_a(d_{k-1}^\top\mathbf{v}_j)_a^2
  + \sum_{b=1}^{n_{k+1}}\zeta_b(d_k\mathbf{v}_j)_b^2.
  \label{eq:diag_response}
\end{equation}
\end{proposition}
\end{samepage}

\begin{proof}
$L_k(A,B)$ depends affinely on $(A,B)$, hence smoothly on
$\spd(n_{k-1})\times\spd(n_{k+1})$.  Since $\lambda_j$ is a simple
eigenvalue of the self-adjoint matrix $L_k^0$, standard eigenvalue
perturbation theory gives differentiability near $(I,I)$.  Set
$A(t)=I+t\delta A$ and $B(t)=I+t\delta B$.  Then
$L_k(t)=L_k^0+t\delta L_k$ with
$\delta L_k=d_{k-1}\delta A d_{k-1}^\top+d_k^\top\delta B d_k$.
The first-order Rayleigh quotient formula gives
$\frac{d}{dt}\lambda_j(t)|_{t=0}=\mathbf{v}_j^\top\delta L_k\mathbf{v}_j$,
which is Eq.~\eqref{eq:spectral_response}.  If
$d_k\mathbf{v}_j\neq0$, choose
$\delta B=(d_k\mathbf{v}_j)(d_k\mathbf{v}_j)^\top$ and $\delta A=0$;
if $d_{k-1}^\top\mathbf{v}_j\neq0$, choose the analogous $\delta A$.
This gives a nonzero derivative.  The diagonal formula follows by taking
$\delta A$ and $\delta B$ diagonal.
\end{proof}

\begin{corollary}[Preservation of the harmonic kernel]
\label{cor:harmonic}
For every $A\succ0$ and $B\succ0$,
\[
  \ker L_k(A,B)=\ker d_{k-1}^\top\cap\ker d_k=\ker L_k^0.
\]
\end{corollary}
\begin{proof}
The quadratic form of $L_k(A,B)$ vanishes exactly when both differential
terms vanish, since $A$ and $B$ are positive definite.
\end{proof}

\begin{corollary}[The diagonal metric family strictly extends the unweighted class]
\label{cor:diagonal}
If $L_k^0$ has a simple non-harmonic eigenvalue $\lambda_j$, then a diagonal
metric perturbation can change $\lambda_j$ at first order.  Thus the diagonal
metric family is locally more expressive than the unweighted Hodge operator
near the identity metric.
\end{corollary}

Proposition~\ref{prop:expressivity} describes local spectral control around
the identity metric. More generally, if $0<c_-\le c_+$ and
$c_-I\preceq A,B\preceq c_+I$, the same energy identity gives
$c_-L_k^0\preceq L_k(A,B)\preceq c_+L_k^0$.
The metric controls the nonzero spectrum while preserving the harmonic kernel.

\begin{remark}[Connection to existing universal approximation results]
\label{prop:approx}
The frameworks of Villar et al.~\cite{villar2021scalars} and
Yarotsky~\cite{yarotsky2022universal} show that if a set of invariant
statistics separates group orbits, then equivariant layers with equivariant
nonlinearities can approximate functions in the corresponding equivariant
class.  RHMP follows this template: the lifting encoder uses
$\En$-invariants, the metric predictor uses $\OC$-invariants, and the
Hodge layers preserve the induced equivariances.  Applying the full
universality theory to cell-complex cochain metrics would require an
orbit-separation result for the chosen statistics $\psi$; this is orthogonal
to the metric construction itself.
\end{remark}

\section{Computational Complexity and Scalability}
\label{app:complexity}

The diagonal-plus-low-rank configuration used in the timing comparison has metric parameterization
\[
  H_k = B_k B_k^\top + \diag(\softplus(\mathbf{h}_k)+\epsilon),
\]
with rank $r=8$. The pure diagonal case has $O(n_k)$ metric application cost; the rank-$r$ variant adds $O(n_k r C)$, with $r=8$. The dominant operations remain sparse applications of $d_k$ and $d_k^\top$ at complexity $O(\mathrm{nnz}(d_k)\cdot C)$, the same as standard cell-complex message passing; the metric application consists of diagonal scaling plus a low-rank update.

Table~\ref{tab:timing} shows the single-epoch training time of each method on three representative tasks
(same GPU, same data pipeline).
RHMP is comparable in speed to SCCNN and CW Net (both cell-complex methods),
faster than GAT, EGNN, and SchNet (graph methods),
and slower than GCN and operator-learning methods (FNO, DeepONet).

\begin{table}[h]
\centering
\caption{Single-epoch training time (seconds) on 3 representative tasks. Measured on the same GPU.}
\label{tab:timing}
\small
\adjustbox{max width=\linewidth}{%
\begin{tabular}{lccc}
\toprule
Method & NS Vort (1024 nodes) & Wilson (1024 nodes) & Yang-Mills (1024 nodes) \\
\midrule
GCN & 1.3 & 1.3 & 1.4 \\
GEM-CNN & 1.5 & 1.6 & 2.1 \\
GaugeEqCNN & 4.3 & 4.3 & 7.8 \\
MPSN & 6.2 & 6.5 & 7.1 \\
SCCNN & 9.2 & 9.9 & 10.1 \\
CW Net & 9.6 & 5.4 & 5.8 \\
Clifford-SMPN & 10.6 & 11.0 & 12.5 \\
\textbf{RHMP} & \textbf{11.9} & \textbf{12.4} & \textbf{16.3} \\
SchNet & 14.2 & 13.8 & 14.2 \\
EGNN & 20.0 & 15.1 & 14.5 \\
GAT & 29.6 & 28.6 & 28.0 \\
\midrule
FNO & 0.7 & N/A & N/A \\
DeepONet & 0.3 & 0.3 & 0.3 \\
\bottomrule
\end{tabular}
}
\end{table}

\subsection{Large-Mesh Scalability}
\label{app:scalability}

To assess scalability beyond the benchmark meshes ($\sim$1K nodes), we profile RHMP on Delaunay meshes with up to $(n_0, n_1, n_2) \approx (50\text{K}, 150\text{K}, 100\text{K})$ total cells, using the same single-GPU setup as all other experiments (§\ref{app:compute}). Table~\ref{tab:scalability} reports forward time, backward time, peak GPU memory, and parameter count at four mesh scales.

\begin{table}[h]
\centering
\small
\caption{Wall-clock time and memory for RHMP at increasing mesh scale on a single NVIDIA RTX PRO 6000 Blackwell GPU (96~GB). \textbf{Trainable params} (weights of $\mathrm{MLP}_H$, the lifting encoder, and the message-passing transforms) are shared across all cells and fixed regardless of mesh size. \textbf{Activations} count per-pass scalar outputs (one $h_i$ per $k$-cell plus intermediate features), which grow linearly with $n_k$ as in any message-passing architecture.}
\label{tab:scalability}
\begin{tabular}{lccccc}
\toprule
Mesh scale & Fwd (ms) & Bwd (ms) & Peak mem (GB) & Trainable params & Activations (M) \\
\midrule
1K   & 16.7  & 23.7  & 0.04 & $\sim$0.05M & 0.31 \\
5K   & 17.2  & 23.1  & 0.12 & $\sim$0.05M & 1.50 \\
10K  & 21.2  & 25.8  & 0.22 & $\sim$0.05M & 2.98 \\
100K & 169.2 & 183.0 & 3.93 & $\sim$0.05M & 59.4 \\
\bottomrule
\end{tabular}
\end{table}

\textbf{Near-linear time scaling.}
From 1K to 100K cells (a $100\times$ increase in mesh size), the forward pass time grows by approximately $10\times$, from $16.7$~ms to $169.2$~ms. This sub-linear behavior reflects the sparse structure of the Hodge operator: the dominant cost is sparse matrix--dense matrix multiplication at $O(\mathrm{nnz}(d_k) \cdot C)$, which scales linearly in the number of nonzeros rather than quadratically in mesh size.

\textbf{Moderate memory footprint.}
At 100K cells the peak GPU memory is 3.9~GB, fitting comfortably on a single 80~GB accelerator. Linear extrapolation suggests that meshes of $\sim$1M cells would require $\sim$39~GB, remaining feasible on current hardware. By contrast, SCCNN exceeds available GPU memory at this scale (estimated requirement: 83.8~GB at 100K cells).

\textbf{Mesh-invariant trainable weights.}
The diagonal metric produces one scalar $h_i$ per cell, so per-pass activations grow linearly with $n_k$ (right column of Table~\ref{tab:scalability}); the trainable weights ($\sim$0.05M total: $\mathrm{MLP}_H$, lifting encoder, message-passing transforms) are shared across cells and remain fixed regardless of mesh size. This mesh-invariance of the trainable weights is what enables the same trained model to transfer directly across meshes of differing size and connectivity, as demonstrated in the AirfRANS variable-mesh experiment (§\ref{sec:airfoil}).

\textbf{Accuracy holds at large scale, with a $\sim$$19\times$ speedup over CW Net.}
We compare RHMP against CW Net on the $U(1)$ Wilson loop benchmark at $100\times$ the benchmark mesh size, end-to-end on the same single GPU and the same $70/15/15$ split. Table~\ref{tab:wilson100k} reports the comparison. At matched parameter count (10.80M vs 11.30M), RHMP reaches test $R^2 = 0.9823$ versus $0.9231$ for CW Net (5.9-point gap, NRMSE $0.0033$ vs $0.0069$), and completes training in $113$\,s versus $2181$\,s ($\sim$$19\times$ faster). The speedup tracks the per-step cost: $\sim$10\,ms for RHMP (sparse $d_k^\top H_k d_k$) versus $\sim$700\,ms for CW Net (dense hidden-state propagation at $h{=}512$). Convergence is also faster: RHMP saturates near $R^2 \approx 0.97$ by epoch 20, while CW Net is still rising at epoch 30. Both architectures benefit from the larger mesh (RHMP $R^2 \approx 0.955 \to 0.982$; CW Net $R^2 \approx 0.875 \to 0.923$), consistent with denser plaquette statistics on finer lattices.

\begin{table}[h]
\centering
\small
\caption{Parameter-matched comparison on a 100K-cell $U(1)$ Wilson loop. Both models trained end-to-end on a single GPU with the same $70/15/15$ split. Wall time is total training time.}
\label{tab:wilson100k}
\begin{tabular}{lcccc}
\toprule
Model & Params & Test $R^2$ & Test NRMSE & Wall time \\
\midrule
\textbf{RHMP} ($C{=}16$, 3 layers) & \textbf{10.80M} & \textbf{0.9823} & \textbf{0.0033} & \textbf{113\,s} \\
CW Net ($h{=}512$, 6 layers)              & 11.30M & 0.9231 & 0.0069 & 2181\,s \\
\bottomrule
\end{tabular}
\end{table}

Shared-metric variants, for example tying $H_k$ across cells with similar local geometry or expanding $h_i$ in a compact basis over the mesh, can further reduce activation memory at scales beyond $10^5$ cells.

\section{Ablation Studies}
\label{app:ablation}

We conduct one ablation each for four design choices of the diagonal-plus-low-rank RHMP configuration ($r=8$), flipping only one at a time while keeping the rest unchanged: fixing the learned metric $H_k$ to the identity; fixing the cross-dimensional mixing coefficient $\alpha$ to $1$, turning off cross-dimensional transport; replacing the norm-gated nonlinearity with an element-wise ReLU; replacing the fixed coboundary operator with a learnable scalar-valued version of the same sparsity pattern.
Ablations are evaluated on four representative tasks:
CNS vorticity on a regular grid,
tangent vector field reconstruction on an ellipsoidal surface,
$U(1)$ Wilson loop gauge curvature,
and $SU(2)$ Yang--Mills field strength.
All ablated variants are retrained under the same data splits for $50$ epochs
(seed $42$, cosine learning rate $1\!\mathrm{e}{-3}\!\to\!1\!\mathrm{e}{-5}$,
parameter budget kept at most $+20\%$ above the full model);
the full-model reference uses the same $r=8$ configuration trained for $100$ epochs.

\begin{table}[h]
\centering
\small
\caption{Precision-driving ablations on four representative tasks: regular-grid CNS vorticity, ellipsoid surface flow, $U(1)$ Wilson loop, and $SU(2)$ Yang--Mills. Each column flips \emph{one} design pillar relative to the full model. \emph{Full}: the unablated diagonal-plus-low-rank configuration ($r=8$), trained for 100 epochs; $H_k\!=\!I$: identity metric (no learnable Riemannian metric); $\alpha\!=\!1$: no cross-dimensional transfer. Ablation variants are retrained for 50 epochs (seed 42, same data split and optimizer). Higher is better for $R^2$/SSIM, lower for NRMSE; bold marks per-row best.}
\label{tab:ablation}
\adjustbox{max width=\linewidth}{%
\begin{tabular}{ll|ccc}
\toprule
Task & Metric & Full & $H_k\!=\!I$ & $\alpha\!=\!1$ \\
\midrule
\multirow{3}{*}{CNS vort.} & $R^2 \uparrow$ & \textbf{0.918} & 0.552 & 0.611 \\
 & SSIM $\uparrow$ & \textbf{0.984} & 0.926 & 0.936 \\
 & NRMSE $\downarrow$ & \textbf{0.0095} & 0.0167 & 0.0155 \\
\midrule
\multirow{3}{*}{Surface flow} & $R^2 \uparrow$ & \textbf{0.972} & 0.935 & 0.927 \\
 & SSIM $\uparrow$ & \textbf{0.988} & 0.974 & 0.971 \\
 & NRMSE $\downarrow$ & \textbf{0.0095} & 0.0107 & 0.0114 \\
\midrule
\multirow{3}{*}{Wilson loop} & $R^2 \uparrow$ & \textbf{0.955} & 0.496 & 0.876 \\
 & SSIM $\uparrow$ & 0.993 & 0.970 & \textbf{0.994} \\
 & NRMSE $\downarrow$ & 0.0032 & 0.0051 & \textbf{0.0025} \\
\midrule
\multirow{3}{*}{Yang--Mills} & $R^2 \uparrow$ & \textbf{0.653} & 0.168 & 0.263 \\
 & SSIM $\uparrow$ & \textbf{0.852} & 0.767 & 0.798 \\
 & NRMSE $\downarrow$ & 0.0212 & 0.0176 & \textbf{0.0166} \\
\bottomrule
\end{tabular}
}
\end{table}

\begin{table}[h]
\centering
\small
\caption{Structure-preserving ablations of the same $r=8$ configuration on the same four tasks. \emph{ReLU}: norm-gated update replaced with element-wise ReLU; \emph{learn $d_k$}: $d_k$ replaced with learnable scalars on the same CW sparsity. The main metric ($R^2$) shifts by at most a few percent. The corresponding structural quantities, however, fail by orders of magnitude: $\|d_1 d_0\|_F$ drifts from $0$ to $\sim 10^1$ for learnable $d_k$; the message-passing layer's cochain-frame equivariance error $\|Q^\top f(Qx) - f(x)\|_F / \|f(x)\|_F$ under $Q\in O(C)$ jumps from fp32 round-off to $\sim 10^0$ for ReLU.}
\label{tab:diagnostics}
\adjustbox{max width=\linewidth}{%
\begin{tabular}{ll|ccc}
\toprule
Task & Quantity & Full & ReLU & learn $d_k$ \\
\midrule
CNS vort. & $R^2 \uparrow$ & 0.918 & 0.884 & 0.895 \\
 & $\|d_1 d_0\|_F$ & 0.0e+00 & 0.0e+00 & 12.95 \\
 & frame equiv.\ err. & 7.4e-07 & 9.9e-01 & 7.4e-07 \\
\midrule
Surface flow & $R^2 \uparrow$ & 0.972 & 0.960 & 0.968 \\
 & $\|d_1 d_0\|_F$ & 0.0e+00 & 0.0e+00 & 9.43 \\
 & frame equiv.\ err. & 1.6e-06 & 1.1e+00 & 1.6e-06 \\
\midrule
Wilson loop & $R^2 \uparrow$ & 0.955 & 0.956 & 0.933 \\
 & $\|d_1 d_0\|_F$ & 0.0e+00 & 0.0e+00 & 15.68 \\
 & frame equiv.\ err. & 7.4e-07 & 1.0e+00 & 7.4e-07 \\
\midrule
Yang--Mills & $R^2 \uparrow$ & 0.653 & 0.612 & 0.665 \\
 & $\|d_1 d_0\|_F$ & 0.0e+00 & 0.0e+00 & 15.36 \\
 & frame equiv.\ err. & 1.6e-06 & 1.0e+00 & 1.6e-06 \\
\bottomrule
\end{tabular}
}
\end{table}


\textbf{Metric learning and cross-dimensional communication dominate the main-metric accuracy.}
Fixing $H_k$ to the identity causes degradations of
$\Delta R^2 = -0.37, -0.46, -0.49$ on CNS vorticity, Wilson loop, and Yang--Mills, respectively,
while turning off cross-dimensional transport gives $\Delta R^2 = -0.31, -0.08, -0.39$.
Both have minor effect on the ellipsoidal surface flow ($\Delta R^2 \approx -0.04$),
consistent with the physical structure of that task:
the target $\mathbf{v}_\text{tan} = \star d\psi$ is a one-step coexact map,
insensitive to both the metric and cross-order coupling.
In contrast, CNS vorticity $\omega = \nabla\!\times\!\mathbf{u}$
and gauge curvature $F = d\theta$ ($U(1)$) / $F = dA + [A,A]$ ($SU(2)$)
involve both $d_0$ and $d_1$ after discretization,
the setting where Hodge stratification and the geometric metric act together.

\textbf{Fixed $d_k$ and norm-gated nonlinearity also affect the main metric, but by a smaller margin.}
Making the coboundary learnable and scalar-valued
gives $\Delta R^2 = -0.023, -0.004, -0.022, +0.012$ across the four tasks
(maximum absolute value $0.023$, a slight improvement on Yang--Mills);
replacing norm-gating with ReLU gives
$\Delta R^2 = -0.034, -0.012, +0.001, -0.041$
(maximum absolute value $0.041$).
Although the effect on the main metric is small,
the respective physical-structure diagnostic quantities (Table~\ref{tab:diagnostics})
degrade by a much larger margin:
after learning $d_k$, $\|d_1 d_0\|_F$ drifts from $0$ to
$12.95, 9.43, 15.68, 15.36$ on the four tasks,
so the Hodge decomposition loses exact closure;
after the ReLU replacement, the relative cochain-frame equivariance error of a single-layer weighted Hodge operator under a random $Q \in O(C)$,
$\|Q^\top f(Qx) - f(x)\|_F / \|f(x)\|_F$,
rises from floating-point precision ($\sim 10^{-6}$) to $\sim 10^0$.
These two quantities are of consistent magnitude on the $U(1)$ and $SU(2)$ tasks,
indicating that the conclusions hold on both Abelian and non-Abelian gauge tasks.

Qualitatively, the output fields of $H_k = I$ and the cross-dimensional-transport-off variant show visible structural loss,
while ReLU and learnable $d_k$ are close to the full model in visual quality and pointwise accuracy,
and their symmetry deviations are visible only in the diagnostic quantities of Table~\ref{tab:diagnostics}.

\section{Statistics Information}
The following uncertainty analysis uses the RHMP configurations and
baseline runs of Tables~\ref{tab:main} and~\ref{tab:airfoil}. It reports
test-set variability and training variability for these fixed configurations.

\subsection{Test-set bootstrap statistics}
\label{app:bootstrap}

We quantify test-set finite-size uncertainty by non-parametric bootstrap.
For each (task, method) we load the best-epoch checkpoint of one
representative training run (seed $=42$), evaluate 100 samples for the
fixed-mesh tasks and 100 meshes for AirfRANS, and draw $B=1000$ bootstrap
samples of these indices with replacement, recomputing every metric on
each resample. We report the standard deviation across the $B$ resamples
for these checkpoints.
NRMSE uses the fixed range of the evaluation targets across resamples; SSIM and Pearson are averaged over per-sample values selected by the resample indices; R${}^2$, MSE, and MAE are recomputed from scratch on each bootstrap set.
Cross-seed training-time variance is reported separately in §\ref{app:seed-variance}.

Tables \ref{tab:boot-all-ssim}--\ref{tab:boot-all-nrmse} below report
bootstrap statistics for all three metrics used in the main paper.

\begin{table}[H]
\centering\scriptsize
\caption{\textbf{SSIM} bootstrap std (1000 resamples of the evaluation samples for the configuration specified above). Each row a method, each column a task.}
\label{tab:boot-all-ssim}
\resizebox{1.0\linewidth}{!}{%
\begin{tabular}{lccccccc}
\toprule
Method & NS Vort & Torus Adv & Ellips Flow & Maxwell & Wilson & Yang--Mills & Airfoil \\
\midrule
GCN & $\pm\,0.029$ & $\pm\,0.015$ & $\pm\,0.010$ & $\pm\,0.028$ & $\pm\,0.016$ & $\pm\,0.001$ & $\pm\,0.018$ \\
GAT & $\pm\,0.030$ & $\pm\,0.013$ & $\pm\,0.010$ & $\pm\,0.025$ & $\pm\,0.007$ & $\pm\,0.002$ & $\pm\,0.021$ \\
SchNet & $\pm\,0.029$ & $\pm\,0.013$ & $\pm\,0.010$ & $\pm\,0.028$ & $\pm\,0.002$ & $\pm\,0.002$ & $\pm\,0.018$ \\
EGNN & $\pm\,0.029$ & $\pm\,0.016$ & $\pm\,0.010$ & $\pm\,0.025$ & $\pm\,0.003$ & $\pm\,0.002$ & $\pm\,0.018$ \\
MPSN & $\pm\,0.029$ & $\pm\,0.017$ & $\pm\,0.009$ & $\pm\,0.026$ & $\pm\,0.017$ & $\pm\,0.005$ & $\pm\,0.019$ \\
SCCNN & $\pm\,0.029$ & $\pm\,0.017$ & $\pm\,0.009$ & $\pm\,0.026$ & $\pm\,0.010$ & $\pm\,0.005$ & $\pm\,0.018$ \\
GaugeEqCNN & $\pm\,0.005$ & $\pm\,0.005$ & $\pm\,0.002$ & $\pm\,0.020$ & $\pm\,0.000$ & $\pm\,0.001$ & $\pm\,0.005$ \\
GEM-CNN & $\pm\,0.003$ & $\pm\,0.003$ & $\pm\,0.001$ & $\pm\,0.021$ & $\pm\,0.000$ & $\pm\,0.001$ & $\pm\,0.005$ \\
CW Net & $\pm\,0.006$ & $\pm\,0.005$ & $\pm\,0.011$ & $\pm\,0.021$ & $\pm\,0.002$ & $\pm\,0.001$ & $\pm\,0.018$ \\
Clifford-SMPN & $\pm\,0.018$ & $\pm\,0.006$ & $\pm\,0.009$ & $\pm\,0.024$ & $\pm\,0.006$ & $\pm\,0.004$ & $\pm\,0.018$ \\
FNO & $\pm\,0.012$ & N/A & N/A & N/A & N/A & N/A & N/A \\
DeepONet & $\pm\,0.029$ & $\pm\,0.022$ & $\pm\,0.010$ & $\pm\,0.028$ & $\pm\,0.013$ & $\pm\,0.002$ & N/A \\
RHMP & $\pm\,0.001$ & $\pm\,0.001$ & $\pm\,0.000$ & $\pm\,0.003$ & $\pm\,0.000$ & $\pm\,0.000$ & $\pm\,0.000$ \\
\bottomrule
\end{tabular}%
}
\end{table}

\begin{table}[H]
\centering\scriptsize
\caption{\textbf{Pearson} bootstrap std (1000 resamples of the evaluation samples for the configuration specified above). Each row a method, each column a task.}
\label{tab:boot-all-pearson}
\resizebox{1.0\linewidth}{!}{%
\begin{tabular}{lccccccc}
\toprule
Method & NS Vort & Torus Adv & Ellips Flow & Maxwell & Wilson & Yang--Mills & Airfoil \\
\midrule
GCN & $\pm\,0.014$ & $\pm\,0.014$ & $\pm\,0.003$ & $\pm\,0.013$ & $\pm\,0.015$ & $\pm\,0.005$ & $\pm\,0.005$ \\
GAT & $\pm\,0.011$ & $\pm\,0.008$ & $\pm\,0.006$ & $\pm\,0.017$ & $\pm\,0.028$ & $\pm\,0.006$ & $\pm\,0.006$ \\
SchNet & $\pm\,0.012$ & $\pm\,0.014$ & $\pm\,0.008$ & $\pm\,0.005$ & $\pm\,0.016$ & $\pm\,0.006$ & $\pm\,0.014$ \\
EGNN & $\pm\,0.014$ & $\pm\,0.013$ & $\pm\,0.007$ & $\pm\,0.016$ & $\pm\,0.022$ & $\pm\,0.006$ & $\pm\,0.014$ \\
MPSN & $\pm\,0.014$ & $\pm\,0.014$ & $\pm\,0.006$ & $\pm\,0.023$ & $\pm\,0.027$ & $\pm\,0.011$ & $\pm\,0.004$ \\
SCCNN & $\pm\,0.004$ & $\pm\,0.014$ & $\pm\,0.007$ & $\pm\,0.027$ & $\pm\,0.023$ & $\pm\,0.011$ & $\pm\,0.004$ \\
GaugeEqCNN & $\pm\,0.006$ & $\pm\,0.004$ & $\pm\,0.001$ & $\pm\,0.020$ & $\pm\,0.014$ & $\pm\,0.001$ & $\pm\,0.013$ \\
GEM-CNN & $\pm\,0.003$ & $\pm\,0.003$ & $\pm\,0.001$ & $\pm\,0.020$ & $\pm\,0.009$ & $\pm\,0.001$ & $\pm\,0.012$ \\
CW Net & $\pm\,0.005$ & $\pm\,0.005$ & $\pm\,0.008$ & $\pm\,0.018$ & $\pm\,0.021$ & $\pm\,0.002$ & $\pm\,0.008$ \\
Clifford-SMPN & $\pm\,0.011$ & $\pm\,0.007$ & $\pm\,0.008$ & $\pm\,0.016$ & $\pm\,0.024$ & $\pm\,0.006$ & $\pm\,0.007$ \\
FNO & $\pm\,0.004$ & N/A & N/A & N/A & N/A & N/A & N/A \\
DeepONet & $\pm\,0.021$ & $\pm\,0.017$ & $\pm\,0.019$ & $\pm\,0.016$ & $\pm\,0.015$ & $\pm\,0.003$ & N/A \\
RHMP & $\pm\,0.001$ & $\pm\,0.001$ & $\pm\,0.000$ & $\pm\,0.002$ & $\pm\,0.000$ & $\pm\,0.000$ & $\pm\,0.001$ \\
\bottomrule
\end{tabular}%
}
\end{table}

\begin{table}[H]
\centering\scriptsize
\caption{\textbf{NRMSE} bootstrap std (1000 resamples of the evaluation samples for the configuration specified above). Each row a method, each column a task.}
\label{tab:boot-all-nrmse}
\resizebox{1.0\linewidth}{!}{%
\begin{tabular}{lccccccc}
\toprule
Method & NS Vort & Torus Adv & Ellips Flow & Maxwell & Wilson & Yang--Mills & Airfoil \\
\midrule
GCN & $\pm\,0.0054$ & $\pm\,0.0017$ & $\pm\,0.0015$ & $\pm\,0.0032$ & $\pm\,0.0010$ & $\pm\,0.0001$ & $\pm\,0.0010$ \\
GAT & $\pm\,0.0055$ & $\pm\,0.0014$ & $\pm\,0.0015$ & $\pm\,0.0032$ & $\pm\,0.0006$ & $\pm\,0.0001$ & $\pm\,0.0010$ \\
SchNet & $\pm\,0.0054$ & $\pm\,0.0017$ & $\pm\,0.0015$ & $\pm\,0.0032$ & $\pm\,0.0003$ & $\pm\,0.0001$ & $\pm\,0.0010$ \\
EGNN & $\pm\,0.0054$ & $\pm\,0.0017$ & $\pm\,0.0015$ & $\pm\,0.0031$ & $\pm\,0.0004$ & $\pm\,0.0001$ & $\pm\,0.0010$ \\
MPSN & $\pm\,0.0051$ & $\pm\,0.0017$ & $\pm\,0.0019$ & $\pm\,0.0028$ & $\pm\,0.0013$ & $\pm\,0.0002$ & $\pm\,0.0010$ \\
SCCNN & $\pm\,0.0050$ & $\pm\,0.0017$ & $\pm\,0.0019$ & $\pm\,0.0028$ & $\pm\,0.0007$ & $\pm\,0.0001$ & $\pm\,0.0010$ \\
GaugeEqCNN & $\pm\,0.0018$ & $\pm\,0.0009$ & $\pm\,0.0005$ & $\pm\,0.0026$ & $\pm\,0.0001$ & $\pm\,0.0001$ & $\pm\,0.0003$ \\
GEM-CNN & $\pm\,0.0014$ & $\pm\,0.0008$ & $\pm\,0.0004$ & $\pm\,0.0027$ & $\pm\,0.0001$ & $\pm\,0.0001$ & $\pm\,0.0003$ \\
CW Net & $\pm\,0.0020$ & $\pm\,0.0010$ & $\pm\,0.0014$ & $\pm\,0.0028$ & $\pm\,0.0003$ & $\pm\,0.0001$ & $\pm\,0.0010$ \\
Clifford-SMPN & $\pm\,0.0039$ & $\pm\,0.0012$ & $\pm\,0.0020$ & $\pm\,0.0028$ & $\pm\,0.0006$ & $\pm\,0.0002$ & $\pm\,0.0010$ \\
FNO & $\pm\,0.0044$ & N/A & N/A & N/A & N/A & N/A & N/A \\
DeepONet & $\pm\,0.0054$ & $\pm\,0.0025$ & $\pm\,0.0014$ & $\pm\,0.0032$ & $\pm\,0.0008$ & $\pm\,0.0001$ & N/A \\
RHMP & $\pm\,0.0008$ & $\pm\,0.0005$ & $\pm\,0.0000$ & $\pm\,0.0016$ & $\pm\,0.0000$ & $\pm\,0.0000$ & $\pm\,0.0001$ \\
\bottomrule
\end{tabular}%
}
\end{table}


\subsection{Cross-seed training variance}
\label{app:seed-variance}

Tables~\ref{tab:seed-all-ssim}--\ref{tab:seed-all-nrmse} report sample
standard deviations across training seeds $\{42,1,2\}$ for the configurations
evaluated here. They quantify variability from initialization and stochastic
training, complementing the test-set bootstrap analysis in
Appendix~\ref{app:bootstrap}.

\begin{table}[H]
\centering\scriptsize
\caption{\textbf{SSIM} cross-seed standard deviation (seeds 42, 1, 2).}
\label{tab:seed-all-ssim}
\resizebox{1.0\linewidth}{!}{%
\begin{tabular}{lccccccc}
\toprule
Method & NS Vort & Torus Adv & Ellips Flow & Maxwell & Wilson & Yang--Mills & Airfoil \\
\midrule
GCN & $\pm\,0.001$ & $\pm\,0.000$ & $\pm\,0.000$ & $\pm\,0.000$ & $\pm\,0.001$ & $\pm\,0.000$ & $\pm\,0.000$ \\
GAT & $\pm\,0.029$ & $\pm\,0.001$ & $\pm\,0.000$ & $\pm\,0.003$ & $\pm\,0.007$ & $\pm\,0.004$ & $\pm\,0.014$ \\
SchNet & $\pm\,0.002$ & $\pm\,0.001$ & $\pm\,0.000$ & $\pm\,0.001$ & $\pm\,0.001$ & $\pm\,0.008$ & $\pm\,0.010$ \\
EGNN & $\pm\,0.022$ & $\pm\,0.002$ & $\pm\,0.000$ & $\pm\,0.005$ & $\pm\,0.003$ & $\pm\,0.012$ & $\pm\,0.001$ \\
MPSN & $\pm\,0.004$ & $\pm\,0.001$ & $\pm\,0.000$ & $\pm\,0.001$ & $\pm\,0.001$ & $\pm\,0.001$ & $\pm\,0.004$ \\
SCCNN & $\pm\,0.002$ & $\pm\,0.003$ & $\pm\,0.002$ & $\pm\,0.010$ & $\pm\,0.017$ & $\pm\,0.002$ & $\pm\,0.001$ \\
GaugeEqCNN & $\pm\,0.002$ & $\pm\,0.003$ & $\pm\,0.006$ & $\pm\,0.001$ & $\pm\,0.000$ & $\pm\,0.000$ & $\pm\,0.008$ \\
GEM-CNN & $\pm\,0.002$ & $\pm\,0.002$ & $\pm\,0.002$ & $\pm\,0.000$ & $\pm\,0.000$ & $\pm\,0.001$ & $\pm\,0.009$ \\
CW Net & $\pm\,0.003$ & $\pm\,0.005$ & $\pm\,0.003$ & $\pm\,0.001$ & $\pm\,0.000$ & $\pm\,0.001$ & $\pm\,0.001$ \\
Clifford-SMPN & $\pm\,0.007$ & $\pm\,0.002$ & $\pm\,0.001$ & $\pm\,0.003$ & $\pm\,0.001$ & $\pm\,0.005$ & $\pm\,0.001$ \\
FNO & $\pm\,0.013$ & N/A & N/A & N/A & N/A & N/A & N/A \\
DeepONet & $\pm\,0.013$ & $\pm\,0.000$ & $\pm\,0.001$ & $\pm\,0.000$ & $\pm\,0.001$ & $\pm\,0.000$ & N/A \\
RHMP & $\pm\,0.000$ & $\pm\,0.003$ & $\pm\,0.000$ & $\pm\,0.000$ & $\pm\,0.000$ & $\pm\,0.000$ & $\pm\,0.000$ \\
\bottomrule
\end{tabular}%
}
\end{table}

\begin{table}[H]
\centering\scriptsize
\caption{\textbf{Pearson} cross-seed standard deviation (seeds 42, 1, 2).}
\label{tab:seed-all-pearson}
\resizebox{1.0\linewidth}{!}{%
\begin{tabular}{lccccccc}
\toprule
Method & NS Vort & Torus Adv & Ellips Flow & Maxwell & Wilson & Yang--Mills & Airfoil \\
\midrule
GCN & $\pm\,0.014$ & $\pm\,0.001$ & $\pm\,0.003$ & $\pm\,0.002$ & $\pm\,0.004$ & $\pm\,0.000$ & $\pm\,0.000$ \\
GAT & $\pm\,0.029$ & $\pm\,0.002$ & $\pm\,0.003$ & $\pm\,0.005$ & $\pm\,0.026$ & $\pm\,0.008$ & $\pm\,0.002$ \\
SchNet & $\pm\,0.016$ & $\pm\,0.001$ & $\pm\,0.012$ & $\pm\,0.009$ & $\pm\,0.004$ & $\pm\,0.016$ & $\pm\,0.010$ \\
EGNN & $\pm\,0.044$ & $\pm\,0.001$ & $\pm\,0.008$ & $\pm\,0.016$ & $\pm\,0.016$ & $\pm\,0.041$ & $\pm\,0.006$ \\
MPSN & $\pm\,0.014$ & $\pm\,0.000$ & $\pm\,0.008$ & $\pm\,0.008$ & $\pm\,0.021$ & $\pm\,0.001$ & $\pm\,0.009$ \\
SCCNN & $\pm\,0.010$ & $\pm\,0.002$ & $\pm\,0.026$ & $\pm\,0.006$ & $\pm\,0.033$ & $\pm\,0.004$ & $\pm\,0.000$ \\
GaugeEqCNN & $\pm\,0.002$ & $\pm\,0.002$ & $\pm\,0.005$ & $\pm\,0.001$ & $\pm\,0.000$ & $\pm\,0.000$ & $\pm\,0.003$ \\
GEM-CNN & $\pm\,0.002$ & $\pm\,0.001$ & $\pm\,0.002$ & $\pm\,0.002$ & $\pm\,0.002$ & $\pm\,0.001$ & $\pm\,0.007$ \\
CW Net & $\pm\,0.007$ & $\pm\,0.004$ & $\pm\,0.003$ & $\pm\,0.001$ & $\pm\,0.002$ & $\pm\,0.001$ & $\pm\,0.005$ \\
Clifford-SMPN & $\pm\,0.012$ & $\pm\,0.003$ & $\pm\,0.001$ & $\pm\,0.004$ & $\pm\,0.002$ & $\pm\,0.007$ & $\pm\,0.003$ \\
FNO & $\pm\,0.005$ & N/A & N/A & N/A & N/A & N/A & N/A \\
DeepONet & $\pm\,0.028$ & $\pm\,0.013$ & $\pm\,0.002$ & $\pm\,0.047$ & $\pm\,0.002$ & $\pm\,0.000$ & N/A \\
RHMP & $\pm\,0.000$ & $\pm\,0.002$ & $\pm\,0.000$ & $\pm\,0.000$ & $\pm\,0.000$ & $\pm\,0.000$ & $\pm\,0.000$ \\
\bottomrule
\end{tabular}%
}
\end{table}

\begin{table}[H]
\centering\scriptsize
\caption{\textbf{NRMSE} cross-seed standard deviation (seeds 42, 1, 2).}
\label{tab:seed-all-nrmse}
\resizebox{1.0\linewidth}{!}{%
\begin{tabular}{lccccccc}
\toprule
Method & NS Vort & Torus Adv & Ellips Flow & Maxwell & Wilson & Yang--Mills & Airfoil \\
\midrule
GCN & $\pm\,0.0000$ & $\pm\,0.0000$ & $\pm\,0.0000$ & $\pm\,0.0000$ & $\pm\,0.0001$ & $\pm\,0.0000$ & $\pm\,0.0000$ \\
GAT & $\pm\,0.0002$ & $\pm\,0.0001$ & $\pm\,0.0000$ & $\pm\,0.0000$ & $\pm\,0.0004$ & $\pm\,0.0001$ & $\pm\,0.0001$ \\
SchNet & $\pm\,0.0000$ & $\pm\,0.0000$ & $\pm\,0.0000$ & $\pm\,0.0000$ & $\pm\,0.0001$ & $\pm\,0.0002$ & $\pm\,0.0001$ \\
EGNN & $\pm\,0.0002$ & $\pm\,0.0000$ & $\pm\,0.0000$ & $\pm\,0.0002$ & $\pm\,0.0002$ & $\pm\,0.0002$ & $\pm\,0.0000$ \\
MPSN & $\pm\,0.0000$ & $\pm\,0.0000$ & $\pm\,0.0000$ & $\pm\,0.0000$ & $\pm\,0.0003$ & $\pm\,0.0000$ & $\pm\,0.0000$ \\
SCCNN & $\pm\,0.0000$ & $\pm\,0.0002$ & $\pm\,0.0001$ & $\pm\,0.0005$ & $\pm\,0.0006$ & $\pm\,0.0001$ & $\pm\,0.0000$ \\
GaugeEqCNN & $\pm\,0.0002$ & $\pm\,0.0004$ & $\pm\,0.0013$ & $\pm\,0.0001$ & $\pm\,0.0000$ & $\pm\,0.0000$ & $\pm\,0.0002$ \\
GEM-CNN & $\pm\,0.0001$ & $\pm\,0.0001$ & $\pm\,0.0004$ & $\pm\,0.0000$ & $\pm\,0.0001$ & $\pm\,0.0000$ & $\pm\,0.0003$ \\
CW Net & $\pm\,0.0001$ & $\pm\,0.0007$ & $\pm\,0.0001$ & $\pm\,0.0001$ & $\pm\,0.0001$ & $\pm\,0.0000$ & $\pm\,0.0000$ \\
Clifford-SMPN & $\pm\,0.0001$ & $\pm\,0.0002$ & $\pm\,0.0000$ & $\pm\,0.0001$ & $\pm\,0.0000$ & $\pm\,0.0001$ & $\pm\,0.0000$ \\
FNO & $\pm\,0.0009$ & N/A & N/A & N/A & N/A & N/A & N/A \\
DeepONet & $\pm\,0.0001$ & $\pm\,0.0000$ & $\pm\,0.0000$ & $\pm\,0.0000$ & $\pm\,0.0001$ & $\pm\,0.0000$ & N/A \\
RHMP & $\pm\,0.0000$ & $\pm\,0.0009$ & $\pm\,0.0001$ & $\pm\,0.0002$ & $\pm\,0.0000$ & $\pm\,0.0000$ & $\pm\,0.0000$ \\
\bottomrule
\end{tabular}%
}
\end{table}

\clearpage

\section{Additional Qualitative Comparisons}
\label{app:qualitative}

This appendix shows further per-task qualitative samples to complement the main-text figures.
For each fixed-mesh task we display four further test samples, taken in order
from the same ranking as the main-text example: test samples ranked by the margin
of RHMP's per-sample $R^2$ over the displayed baselines, among those on which RHMP
has the lowest error of all methods. For AirfRANS, two test simulations are shown
on two mesh sizes.

\begin{figure}[H]
\centering
\begin{subfigure}{0.49\textwidth}\centering
  \includegraphics[width=\linewidth]{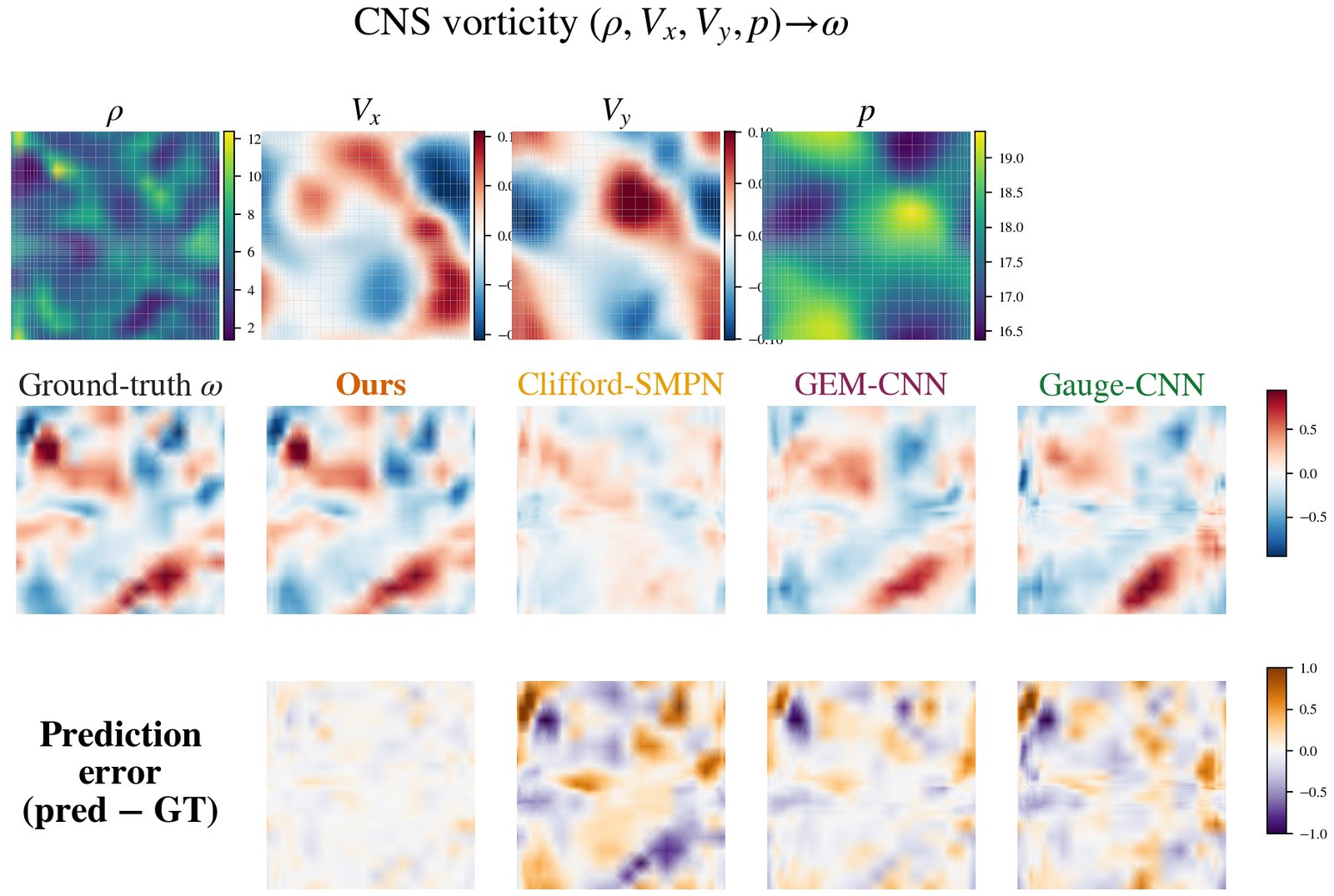}
  \caption{Sample 1.}
\end{subfigure}\hfill
\begin{subfigure}{0.49\textwidth}\centering
  \includegraphics[width=\linewidth]{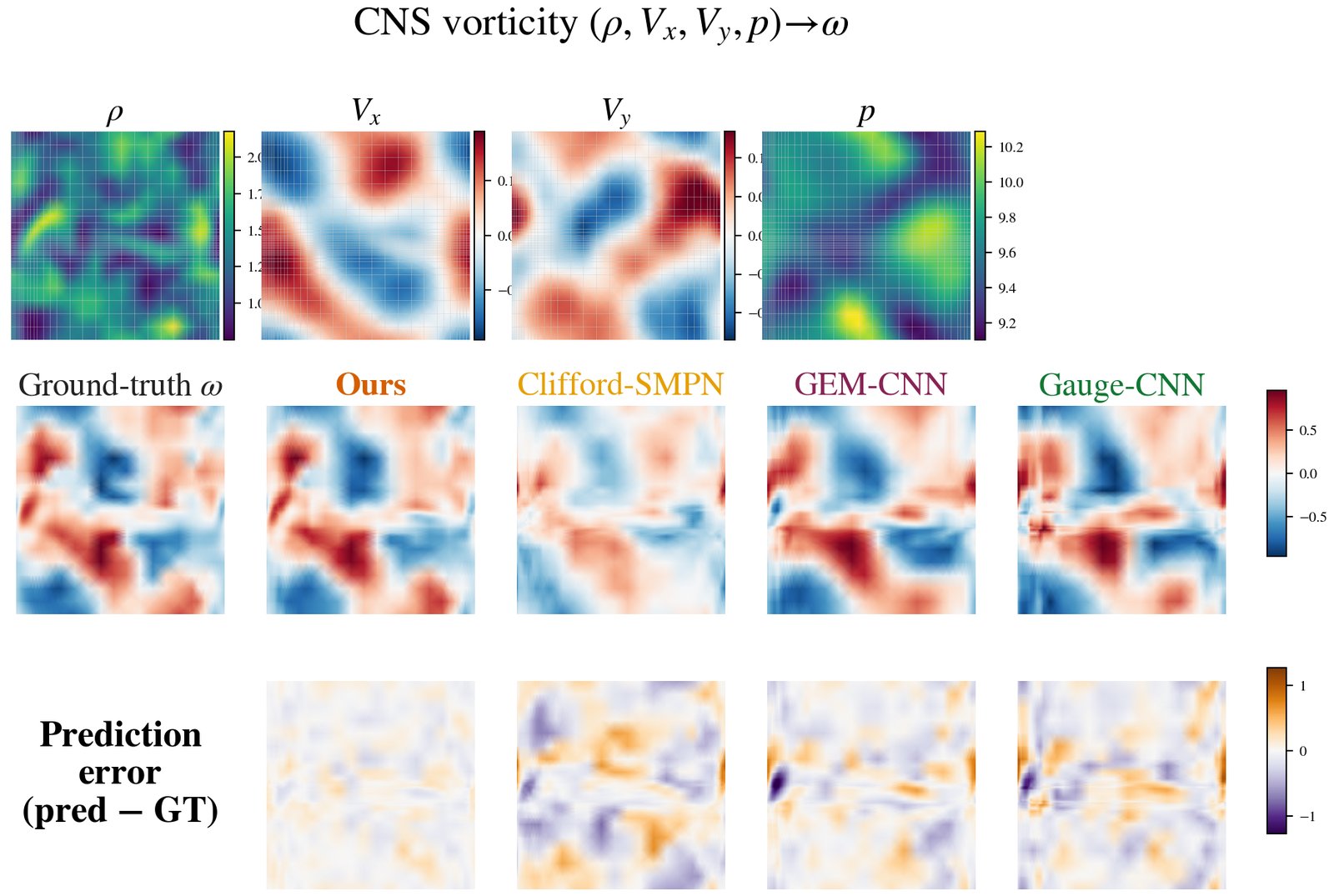}
  \caption{Sample 2.}
\end{subfigure}

\vspace{0.4em}
\begin{subfigure}{0.49\textwidth}\centering
  \includegraphics[width=\linewidth]{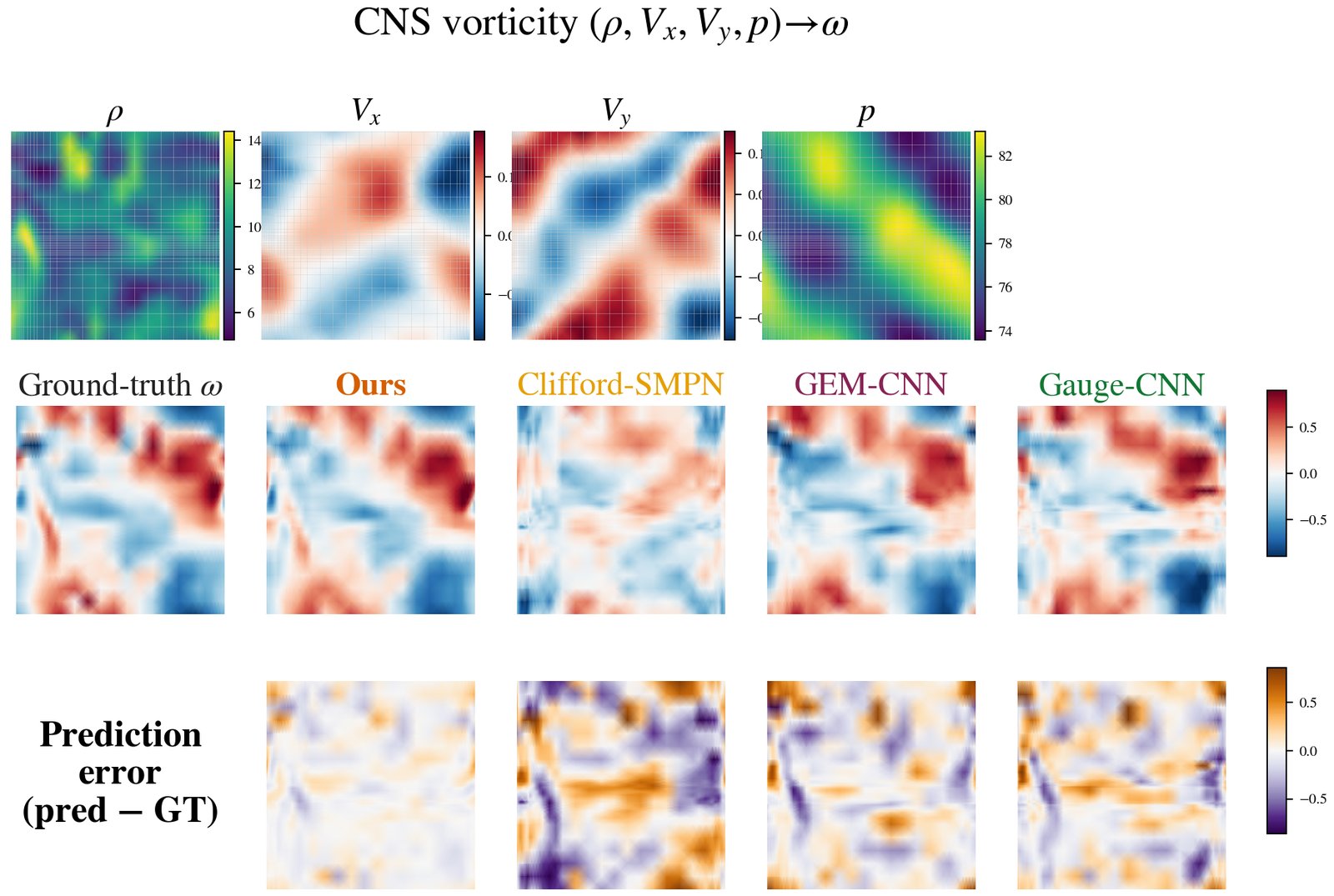}
  \caption{Sample 3.}
\end{subfigure}\hfill
\begin{subfigure}{0.49\textwidth}\centering
  \includegraphics[width=\linewidth]{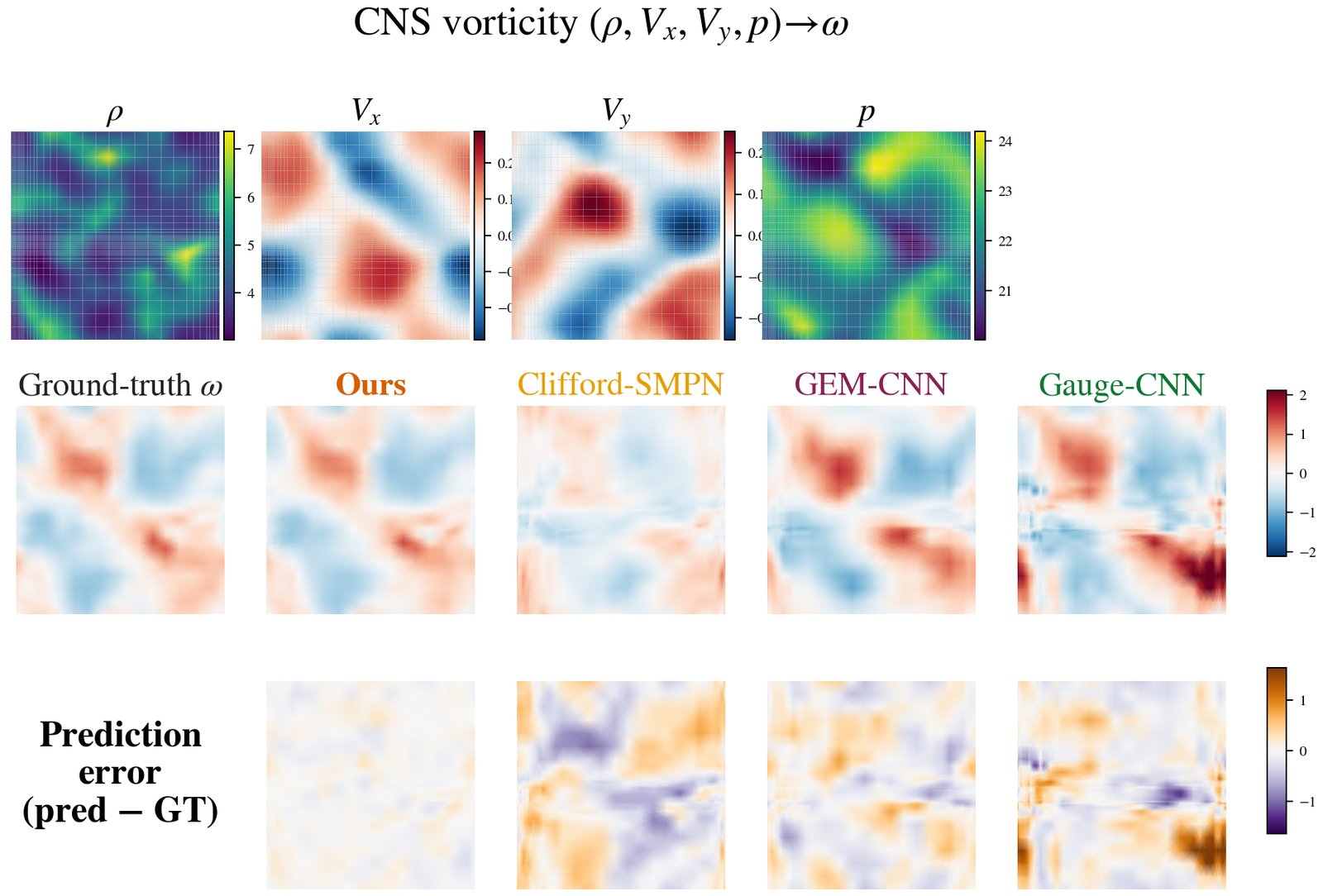}
  \caption{Sample 4.}
\end{subfigure}
\caption{Additional qualitative samples on Navier--Stokes vorticity (PDEBench).
Each row contrasts the ground-truth vorticity with predictions from
RHMP and the strongest cell-complex / manifold-gauge baselines.}
\label{fig:qual_app_ns}
\end{figure}

\begin{figure}[H]
\centering
\begin{subfigure}{0.49\textwidth}\centering
  \includegraphics[width=\linewidth]{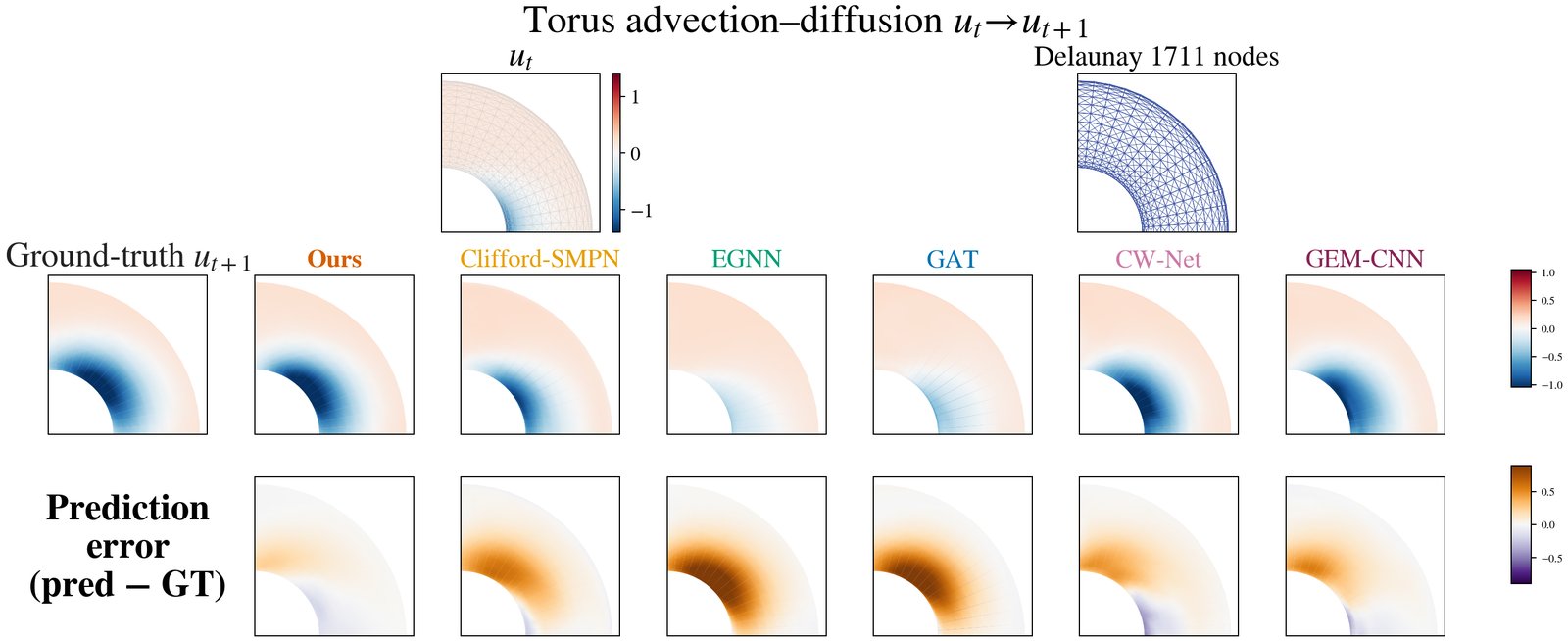}
  \caption{Sample 1.}
\end{subfigure}\hfill
\begin{subfigure}{0.49\textwidth}\centering
  \includegraphics[width=\linewidth]{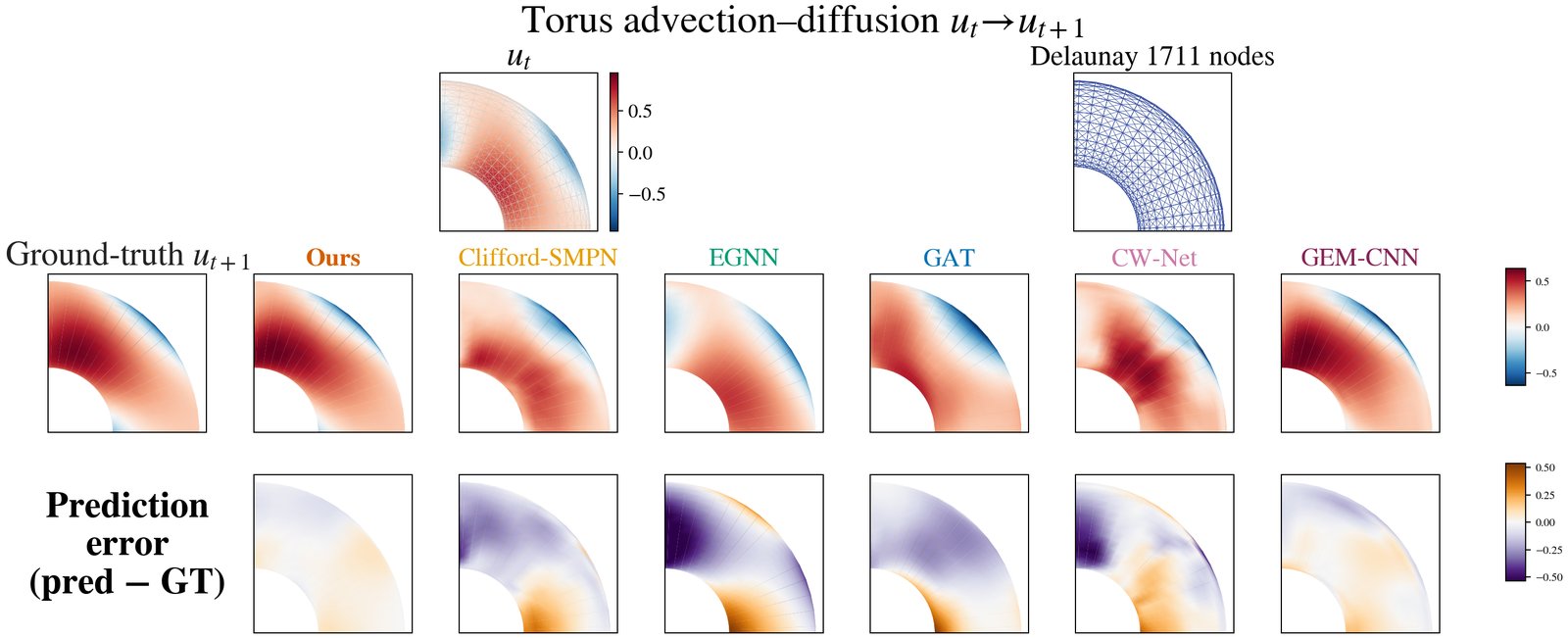}
  \caption{Sample 2.}
\end{subfigure}

\vspace{0.4em}
\begin{subfigure}{0.49\textwidth}\centering
  \includegraphics[width=\linewidth]{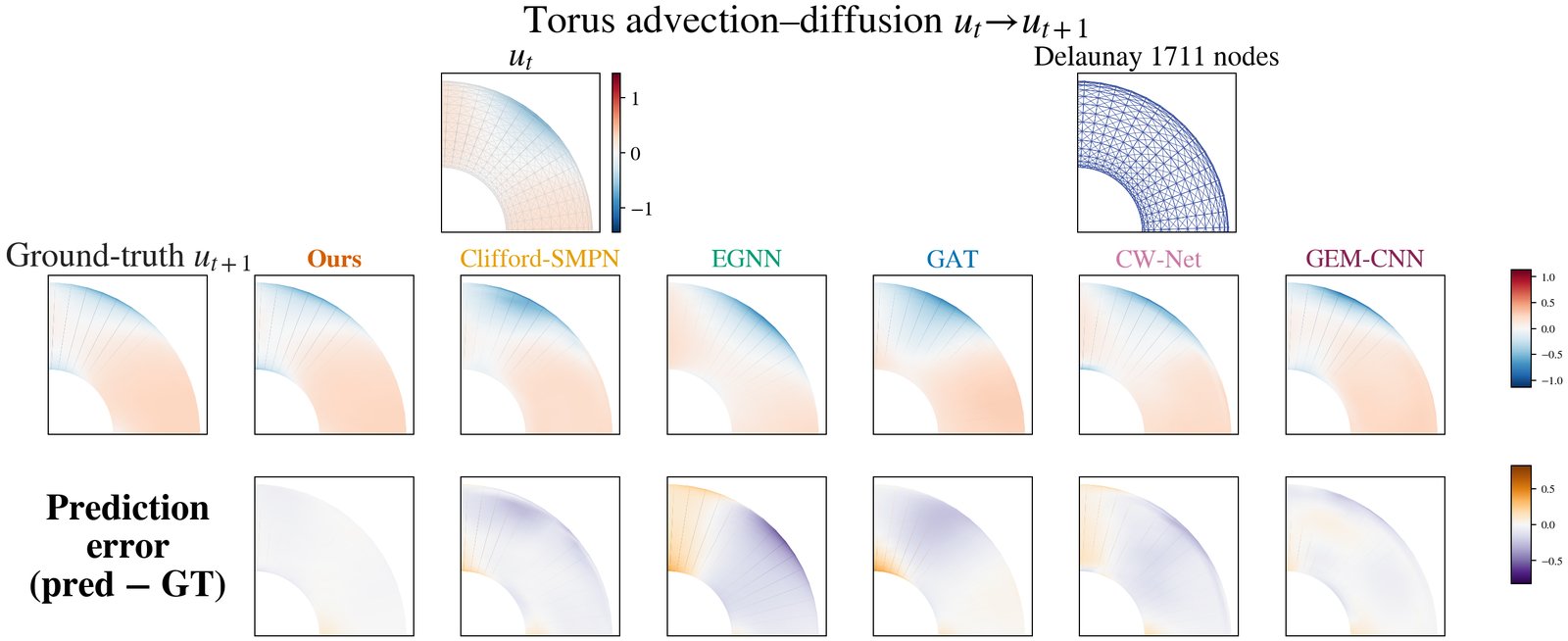}
  \caption{Sample 3.}
\end{subfigure}\hfill
\begin{subfigure}{0.49\textwidth}\centering
  \includegraphics[width=\linewidth]{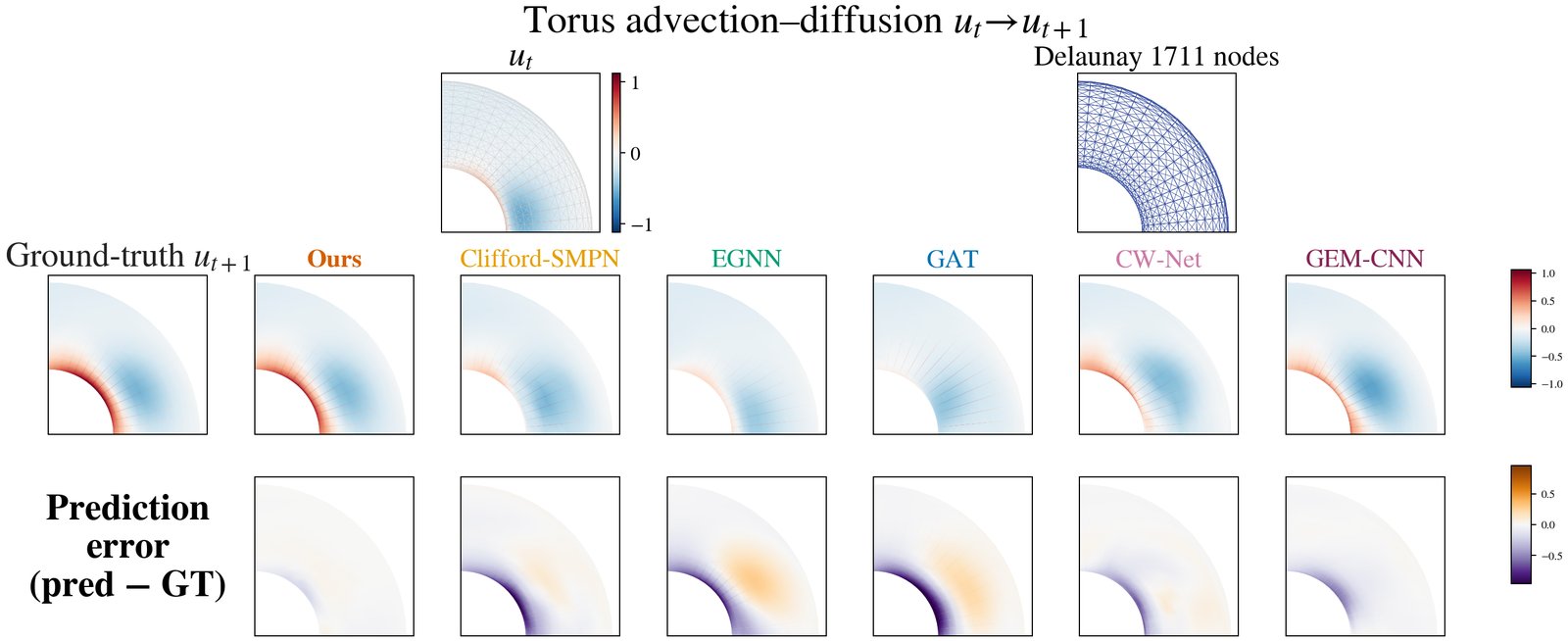}
  \caption{Sample 4.}
\end{subfigure}
\caption{Additional qualitative samples on torus advection--diffusion.
Each panel shows the scalar transport field on a genus-1 surface;
RHMP preserves the localized plumes and sharp gradients that
graph and cell-complex baselines smear out.}
\label{fig:qual_app_torus}
\end{figure}

\begin{figure}[H]
\centering
\begin{subfigure}{0.43\textwidth}\centering
  \includegraphics[width=\linewidth]{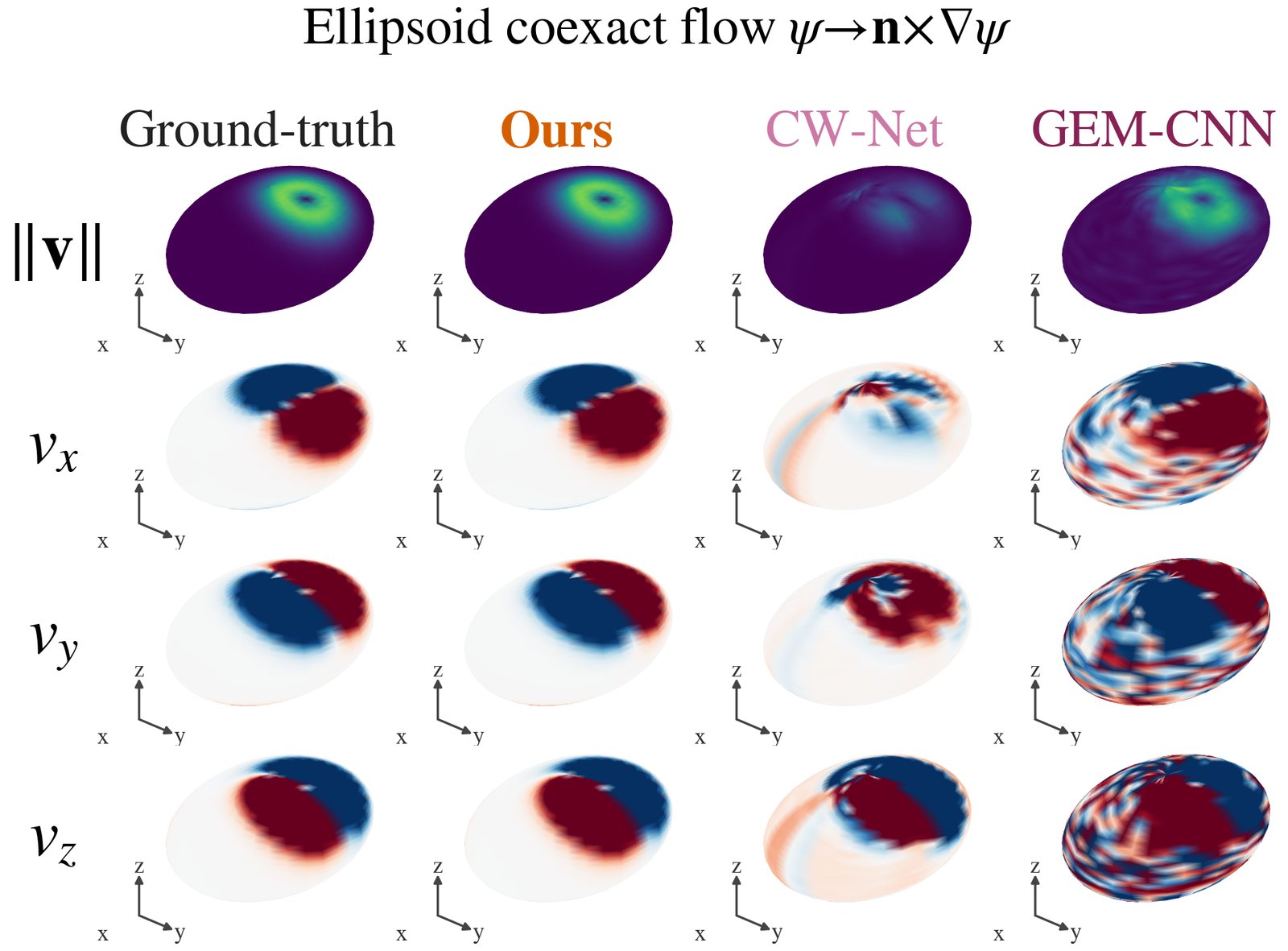}
  \caption{Sample 1.}
\end{subfigure}\hfill
\begin{subfigure}{0.43\textwidth}\centering
  \includegraphics[width=\linewidth]{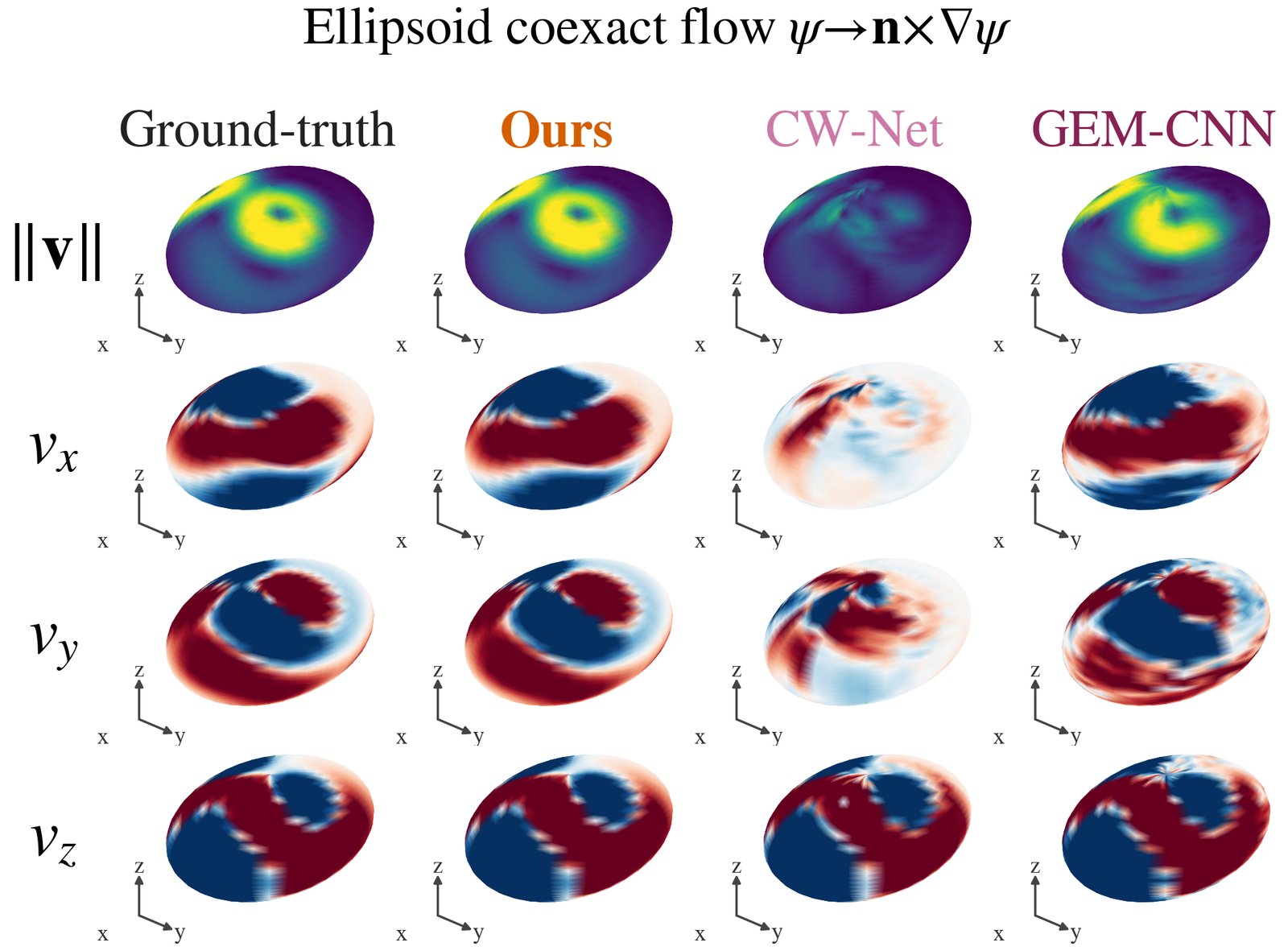}
  \caption{Sample 2.}
\end{subfigure}

\vspace{0.4em}
\begin{subfigure}{0.43\textwidth}\centering
  \includegraphics[width=\linewidth]{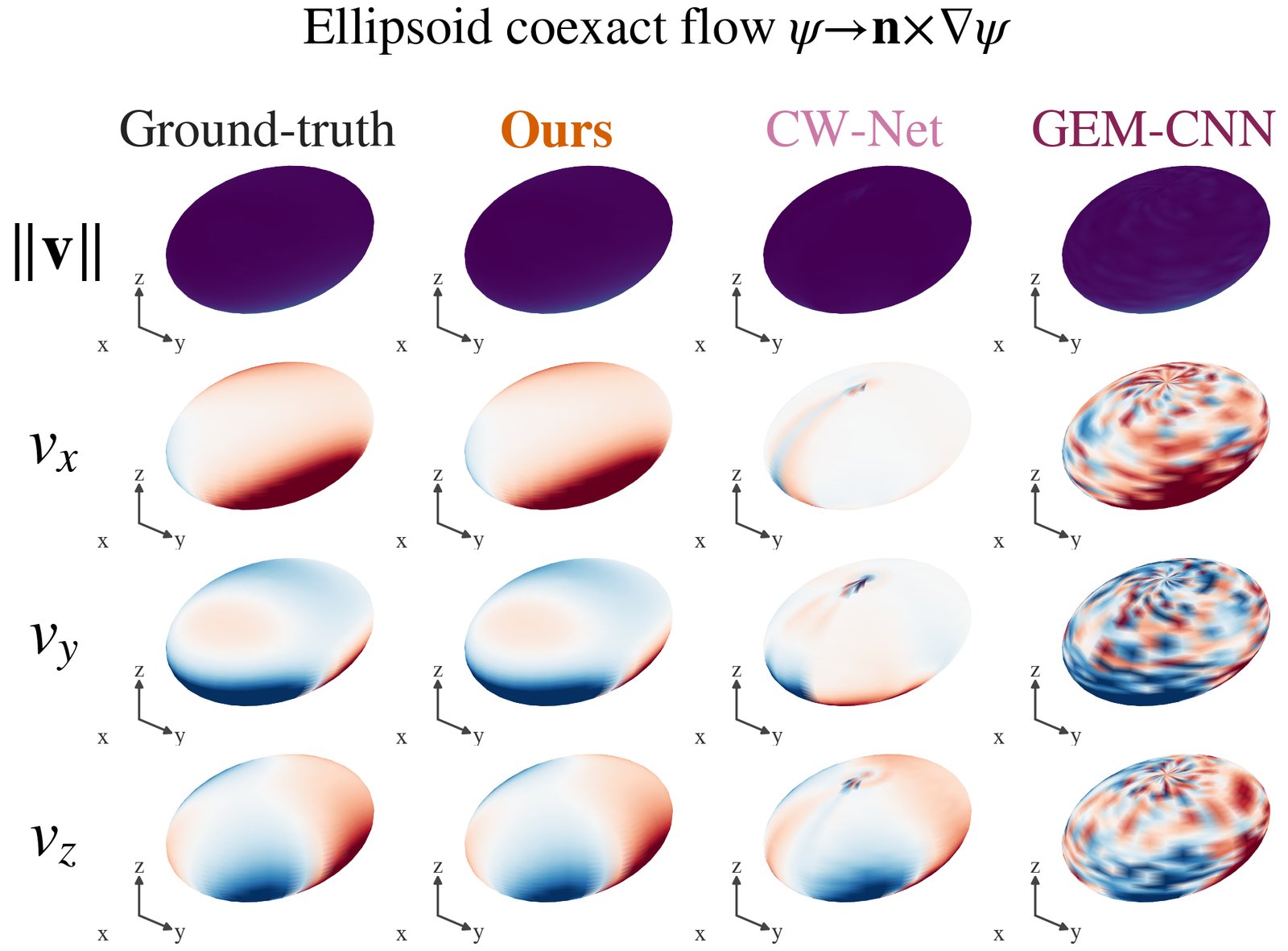}
  \caption{Sample 3.}
\end{subfigure}\hfill
\begin{subfigure}{0.43\textwidth}\centering
  \includegraphics[width=\linewidth]{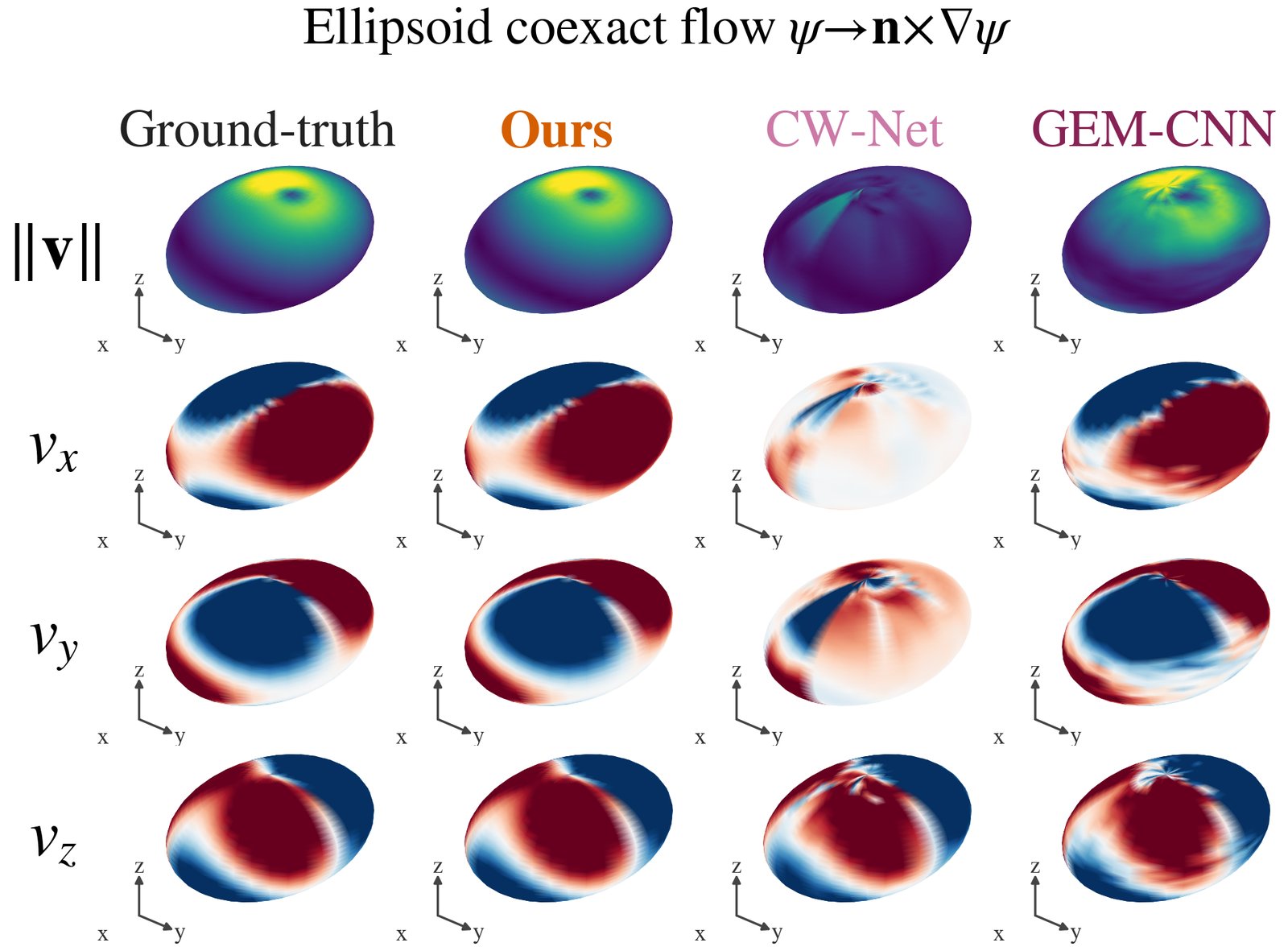}
  \caption{Sample 4.}
\end{subfigure}
\caption{Additional qualitative samples on ellipsoid coexact surface flow.
Rows show $\|\mathbf{v}\|$ and the three Cartesian components of
$\mathbf{v}_{\rm tan} = \mathbf{n}\times\nabla\psi$; CW-Net loses the
sign structure of the surface vector field while RHMP tracks the
ground truth.}
\label{fig:qual_app_ellipsoid}
\end{figure}

\begin{figure}[H]
\centering
\begin{subfigure}{0.49\textwidth}\centering
  \includegraphics[width=\linewidth]{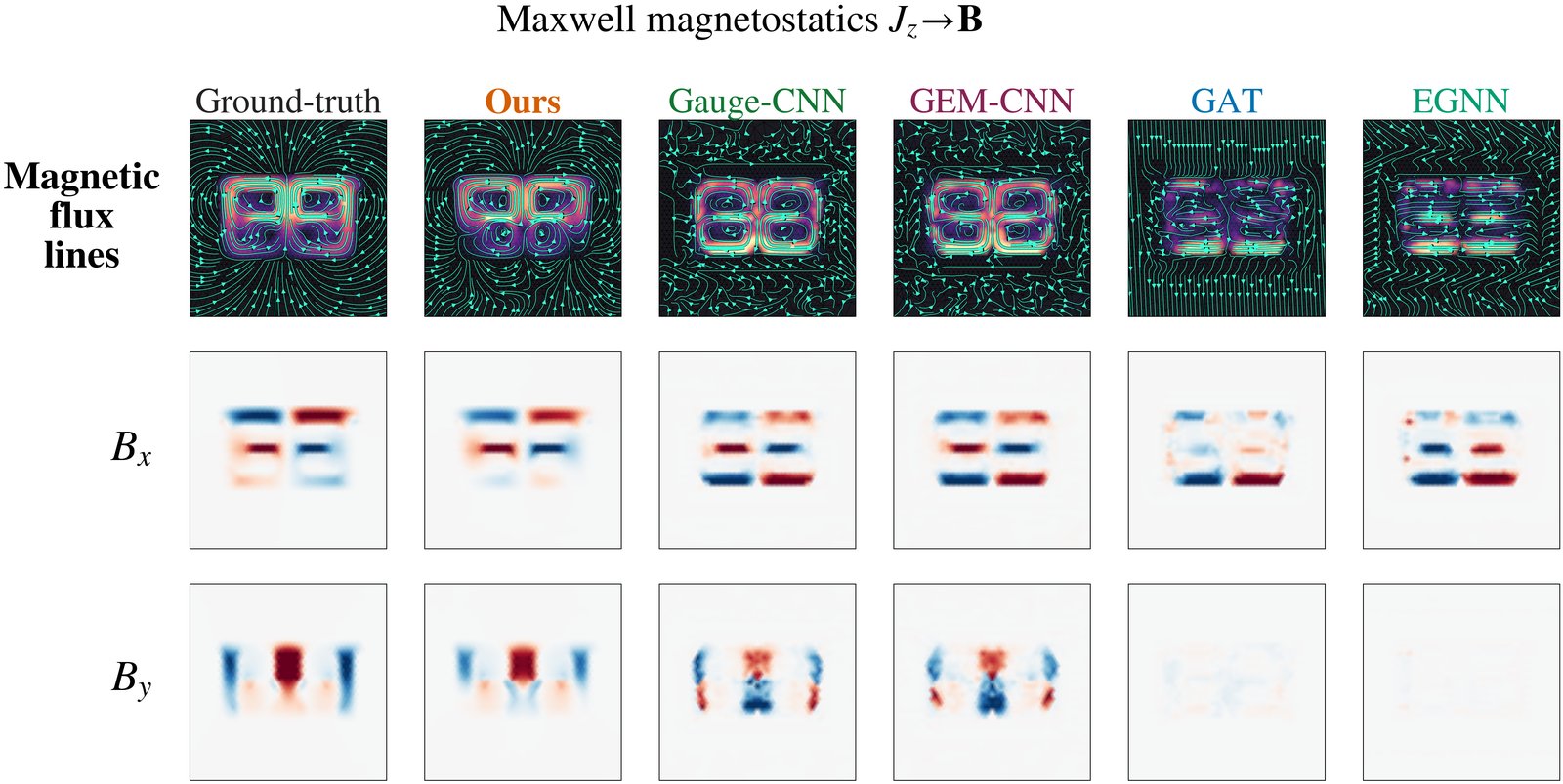}
  \caption{Sample 1.}
\end{subfigure}\hfill
\begin{subfigure}{0.49\textwidth}\centering
  \includegraphics[width=\linewidth]{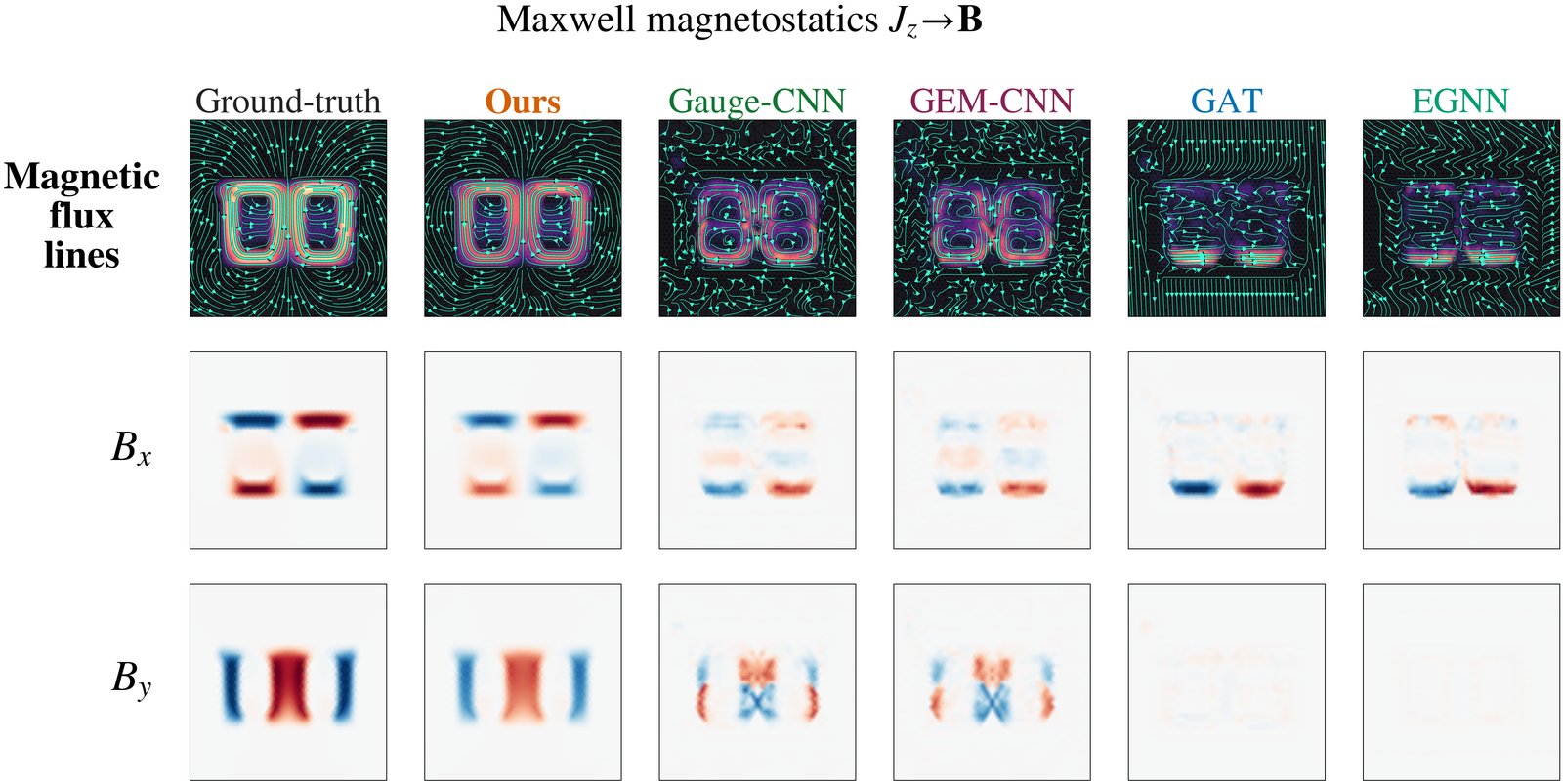}
  \caption{Sample 2.}
\end{subfigure}

\vspace{0.4em}
\begin{subfigure}{0.49\textwidth}\centering
  \includegraphics[width=\linewidth]{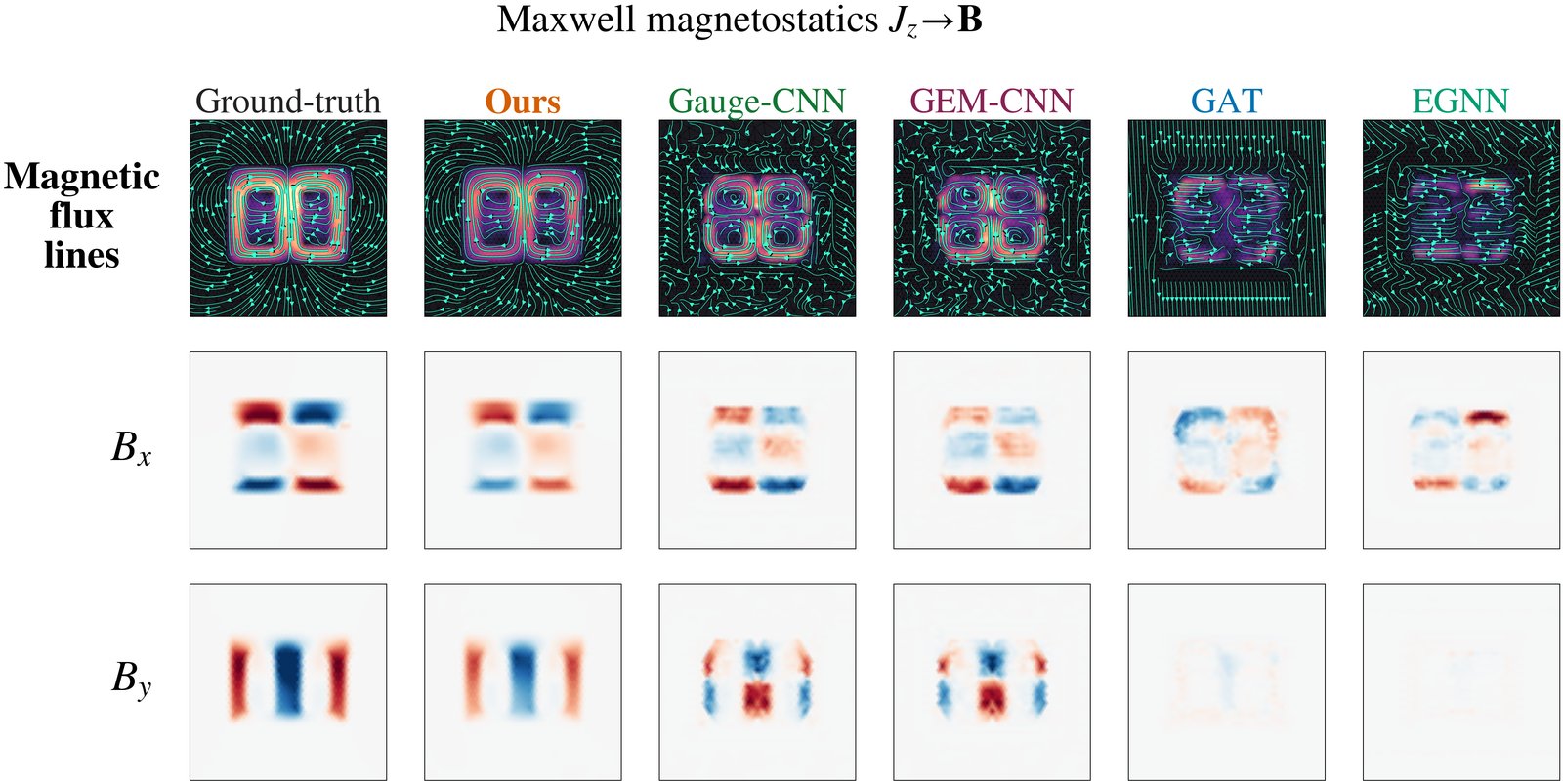}
  \caption{Sample 3.}
\end{subfigure}\hfill
\begin{subfigure}{0.49\textwidth}\centering
  \includegraphics[width=\linewidth]{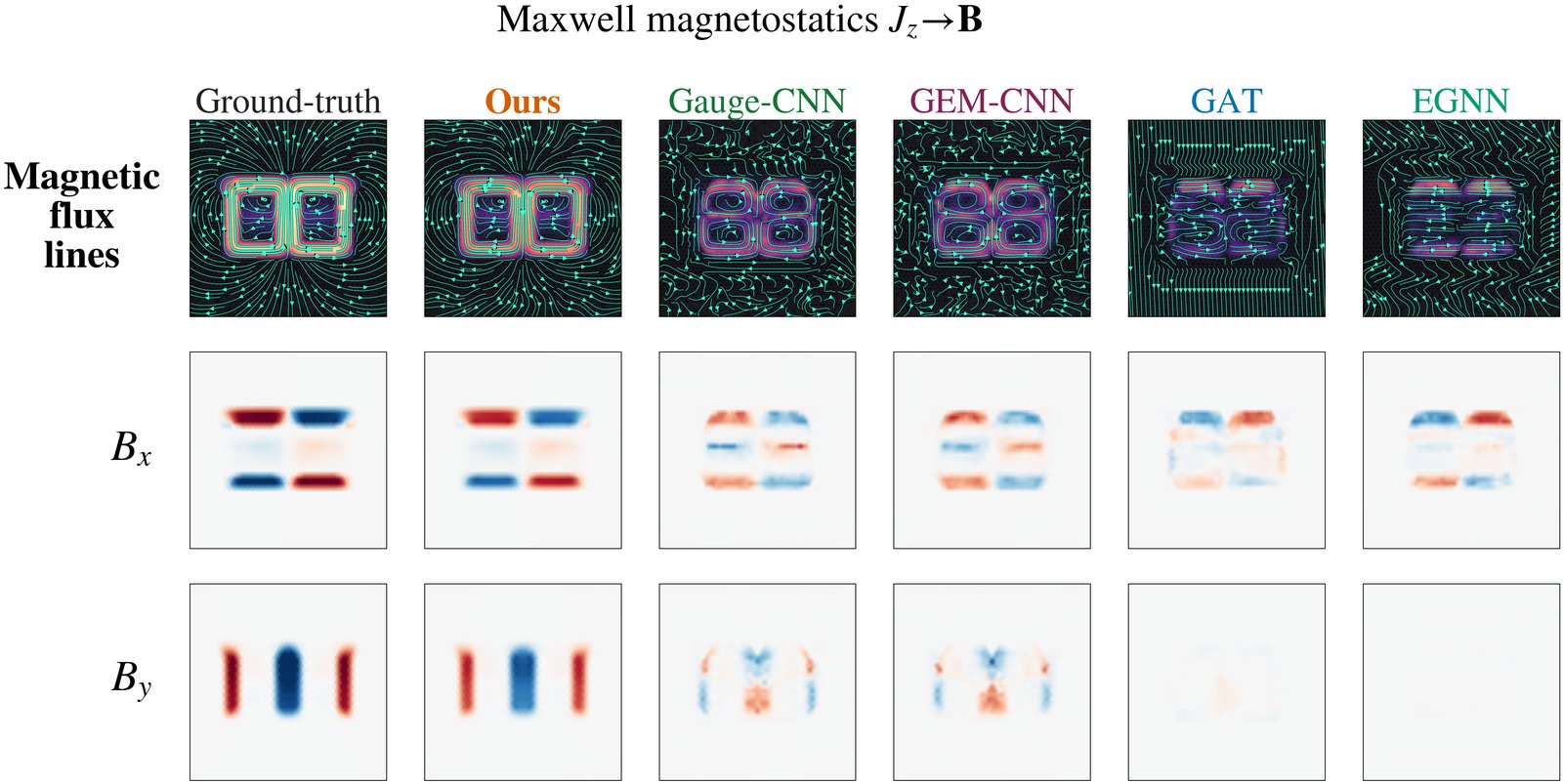}
  \caption{Sample 4.}
\end{subfigure}
\caption{Additional qualitative samples on Maxwell magnetostatics.
Heatmaps show $\|\mathbf{B}\|$ and streamlines show $\mathbf{B}$;
the divergence-free constraint $d^2\!=\!0$ is encoded structurally in
RHMP, whose flux lines close through the core and the air gaps as in the ground truth.}
\label{fig:qual_app_maxwell}
\end{figure}

\begin{figure}[H]
\centering
\begin{subfigure}{0.49\textwidth}\centering
  \includegraphics[width=\linewidth]{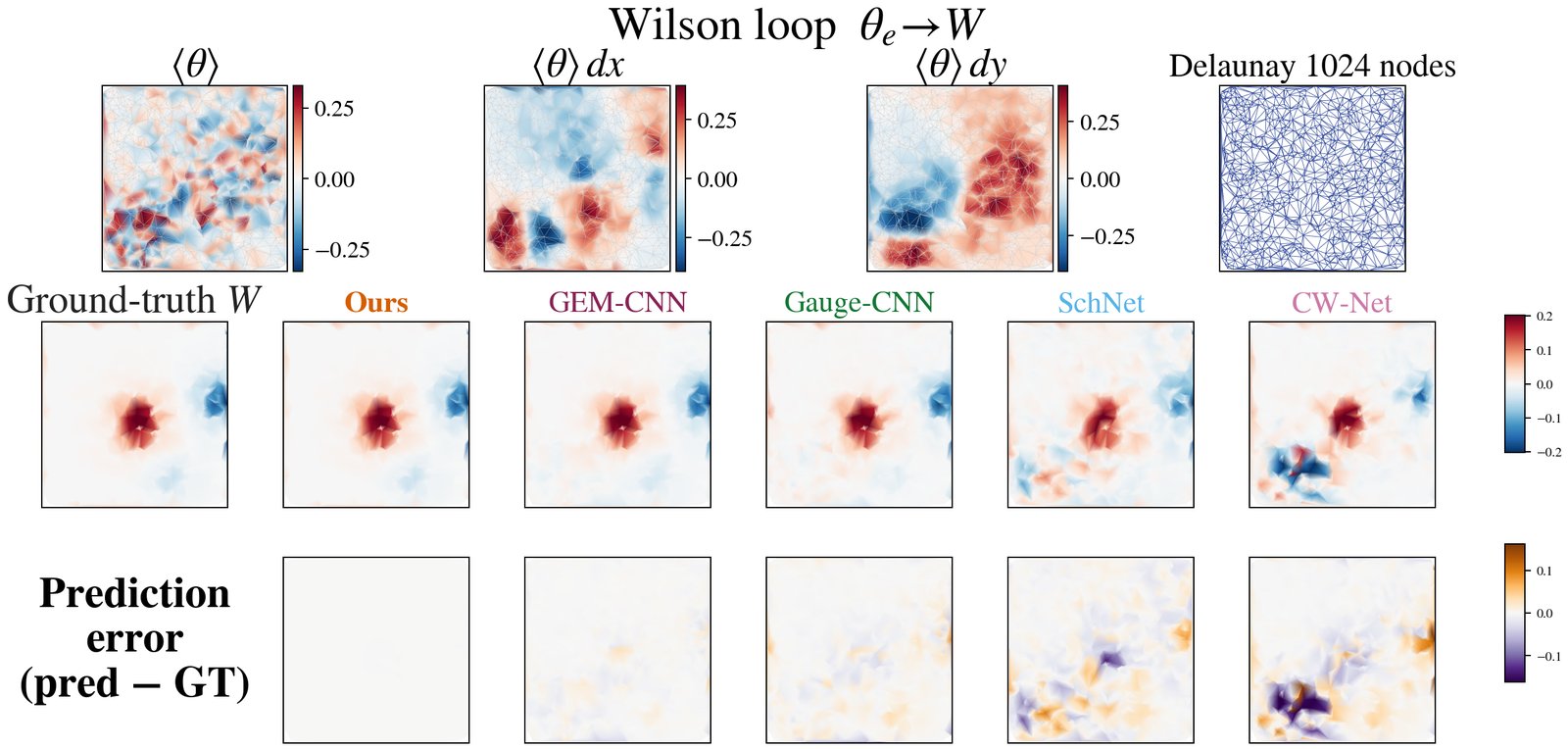}
  \caption{Sample 1.}
\end{subfigure}\hfill
\begin{subfigure}{0.49\textwidth}\centering
  \includegraphics[width=\linewidth]{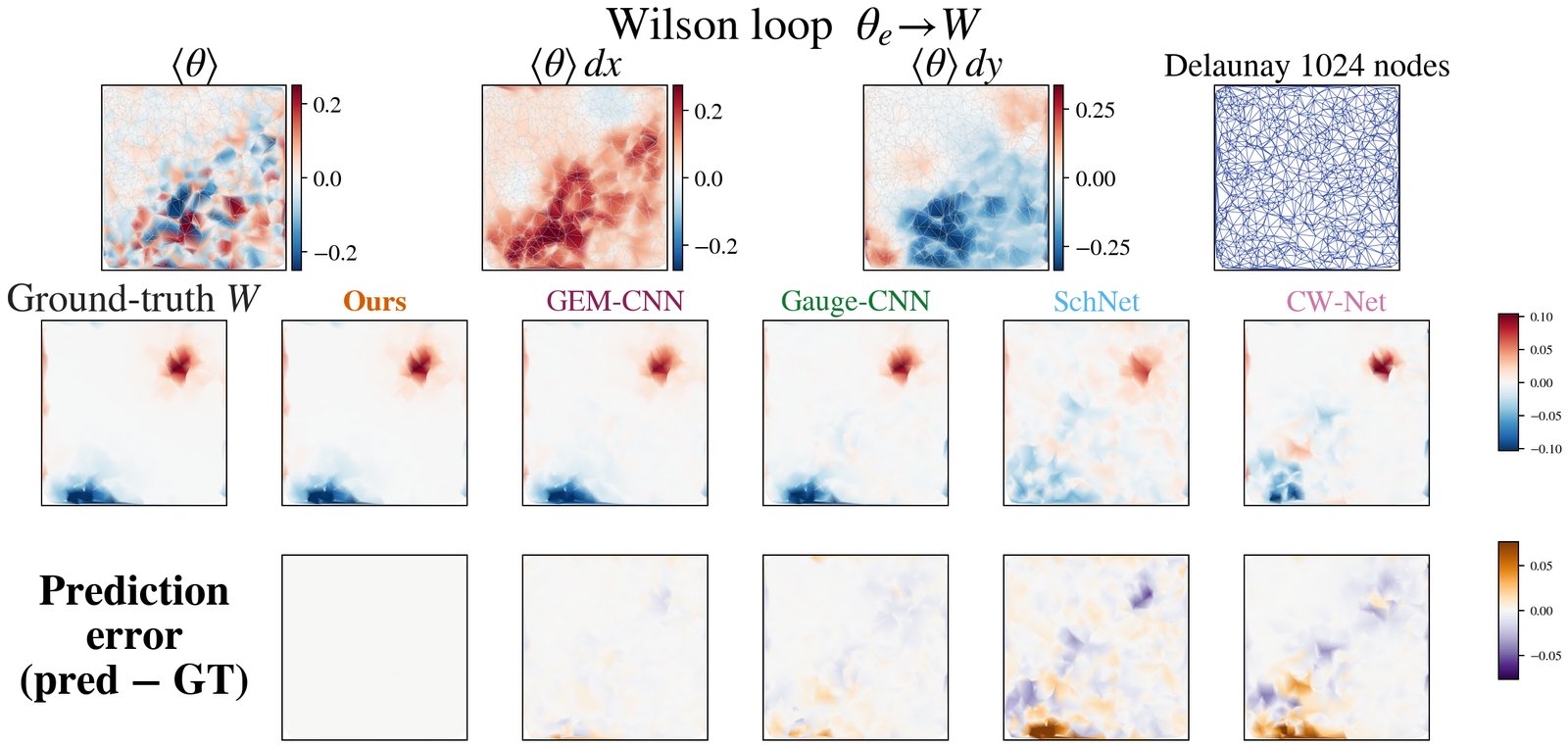}
  \caption{Sample 2.}
\end{subfigure}

\vspace{0.4em}
\begin{subfigure}{0.49\textwidth}\centering
  \includegraphics[width=\linewidth]{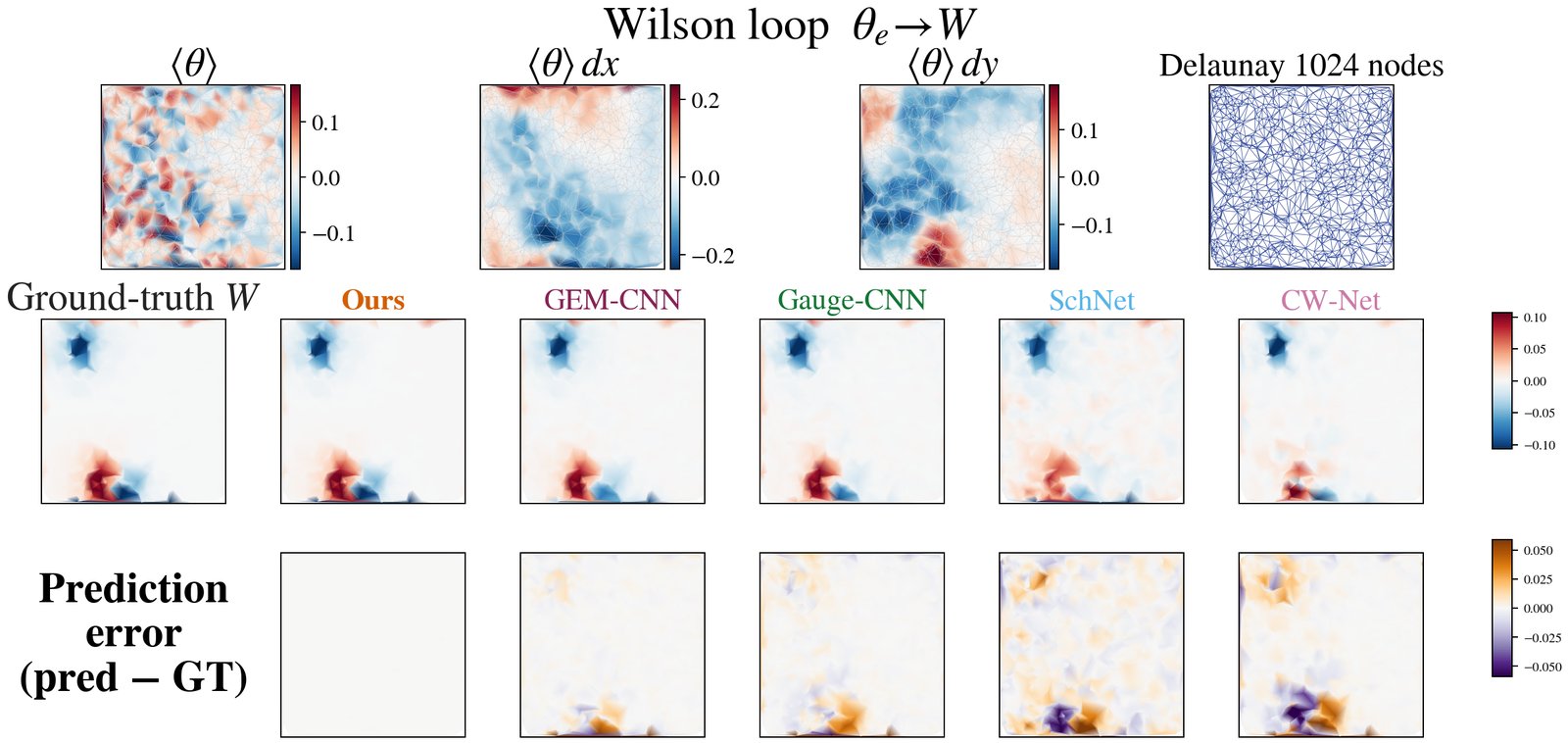}
  \caption{Sample 3.}
\end{subfigure}\hfill
\begin{subfigure}{0.49\textwidth}\centering
  \includegraphics[width=\linewidth]{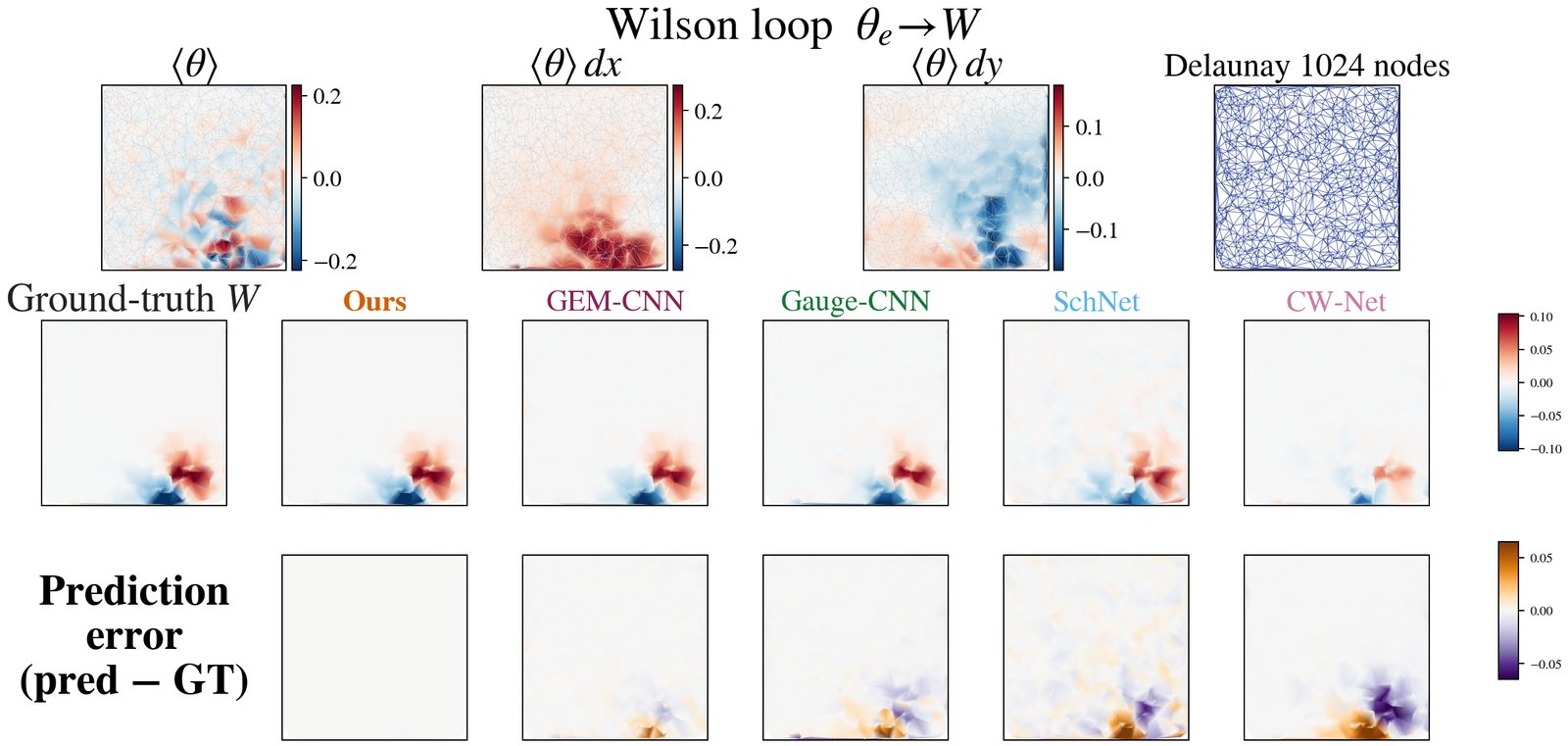}
  \caption{Sample 4.}
\end{subfigure}
\caption{Additional qualitative samples on $U(1)$ Wilson-loop curvature.
The localized $F=d\theta$ peaks are recovered with the correct
amplitude by RHMP; GEM-CNN and GaugeEquivCNN leave small
residuals, while SchNet and CW-Net add spurious structure away from the peaks.}
\label{fig:qual_app_wilson}
\end{figure}

\begin{figure}[H]
\centering
\begin{subfigure}{0.49\textwidth}\centering
  \includegraphics[width=\linewidth]{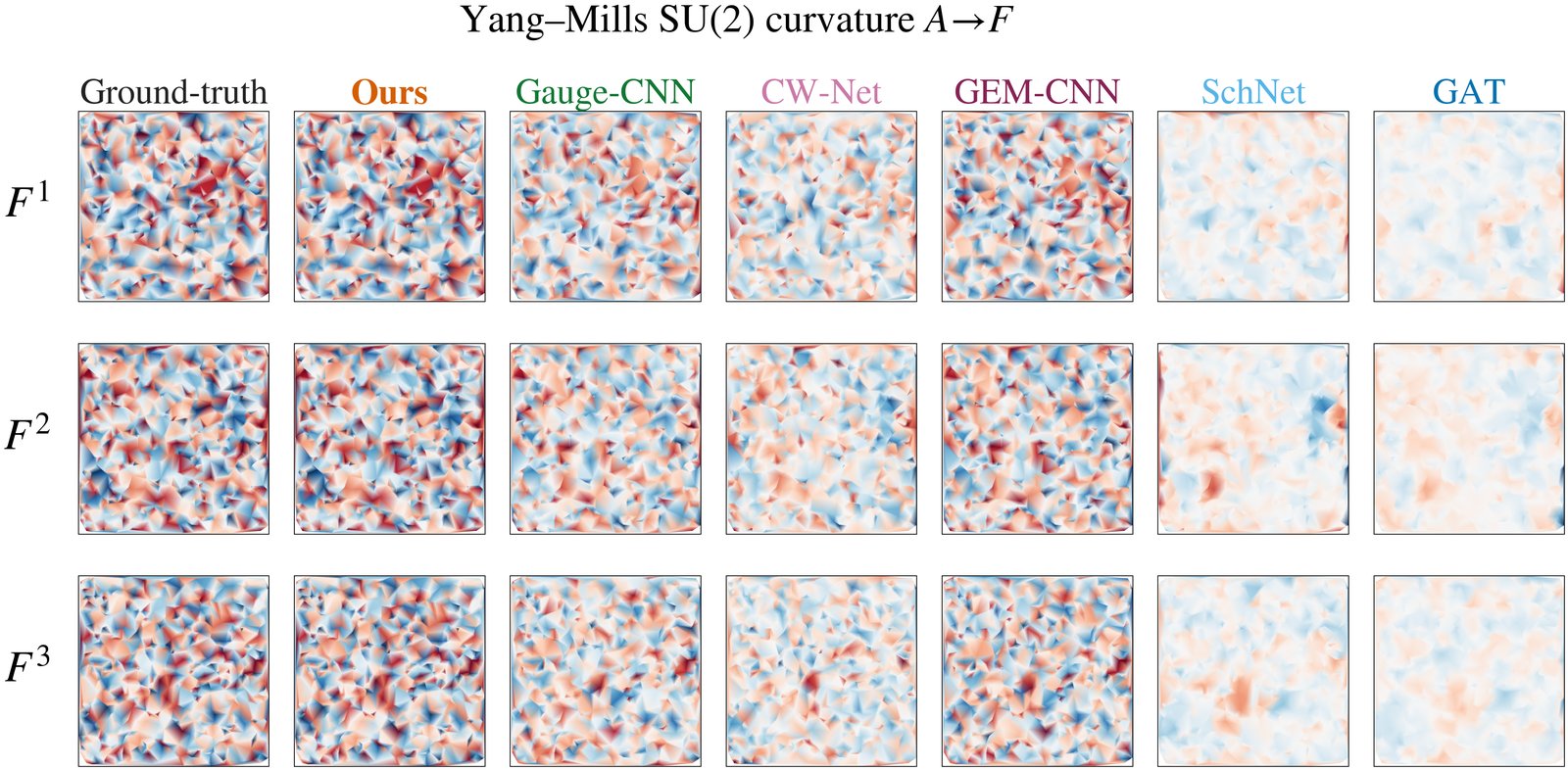}
  \caption{Sample 1.}
\end{subfigure}\hfill
\begin{subfigure}{0.49\textwidth}\centering
  \includegraphics[width=\linewidth]{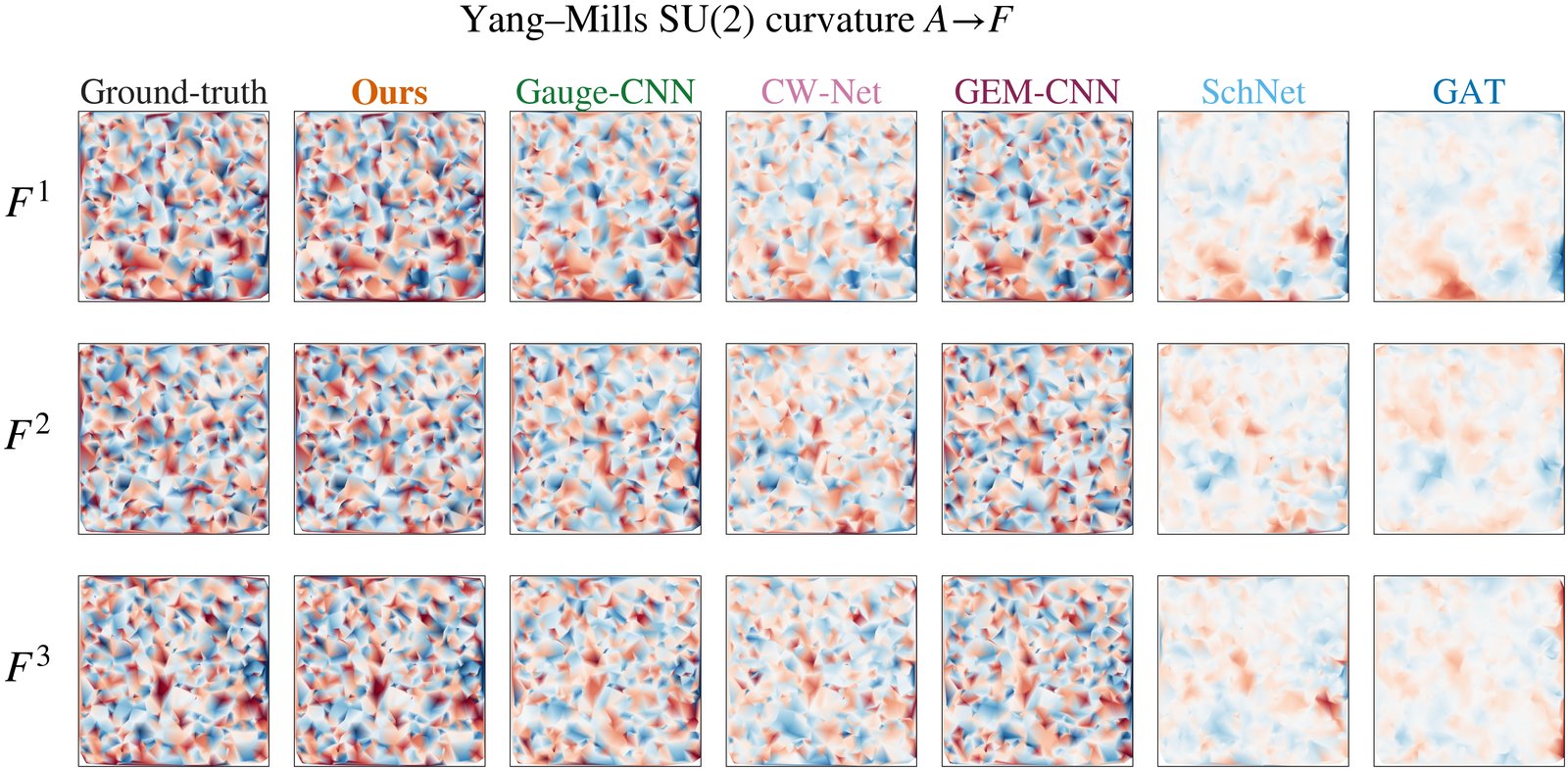}
  \caption{Sample 2.}
\end{subfigure}

\vspace{0.4em}
\begin{subfigure}{0.49\textwidth}\centering
  \includegraphics[width=\linewidth]{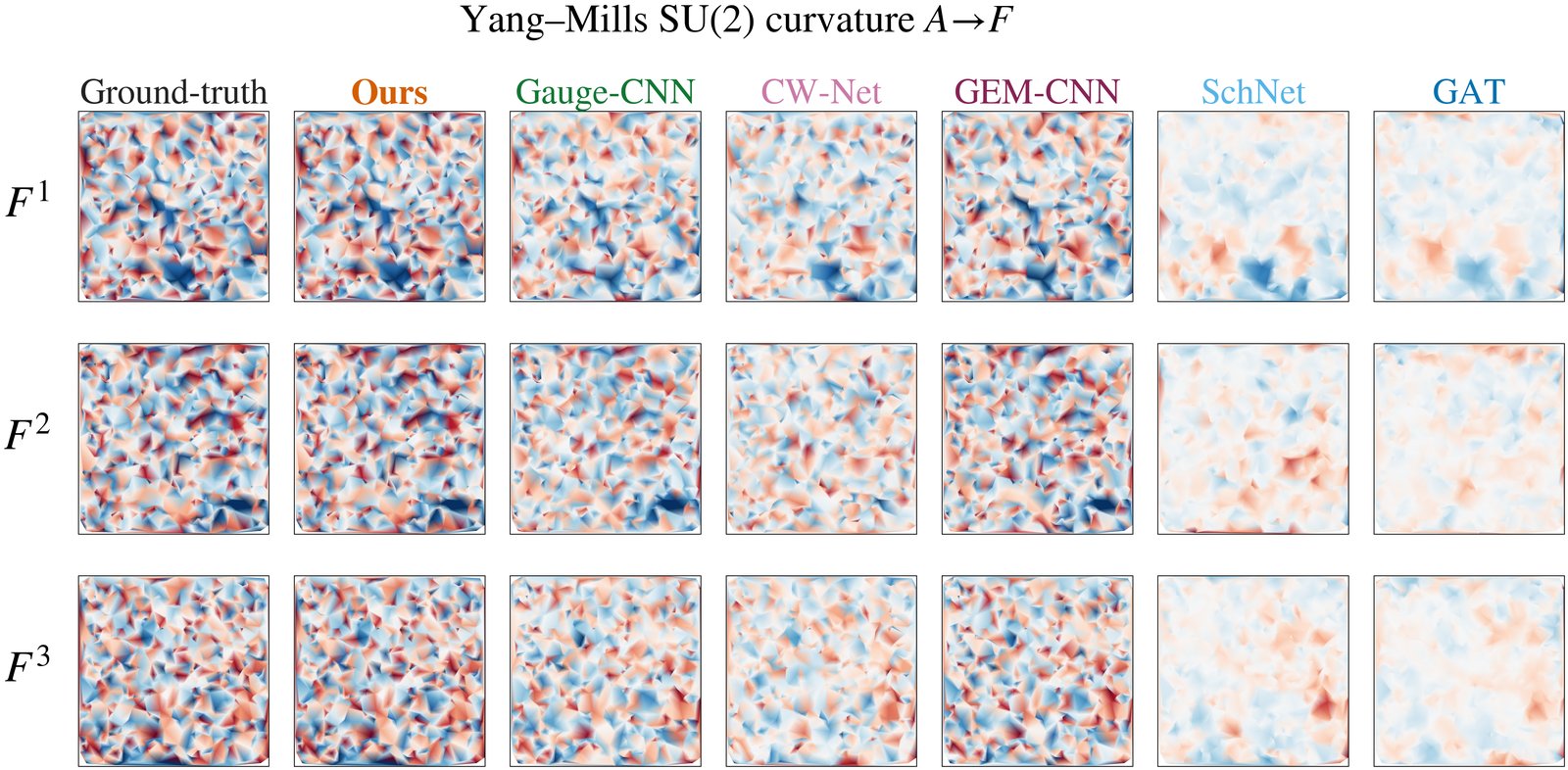}
  \caption{Sample 3.}
\end{subfigure}\hfill
\begin{subfigure}{0.49\textwidth}\centering
  \includegraphics[width=\linewidth]{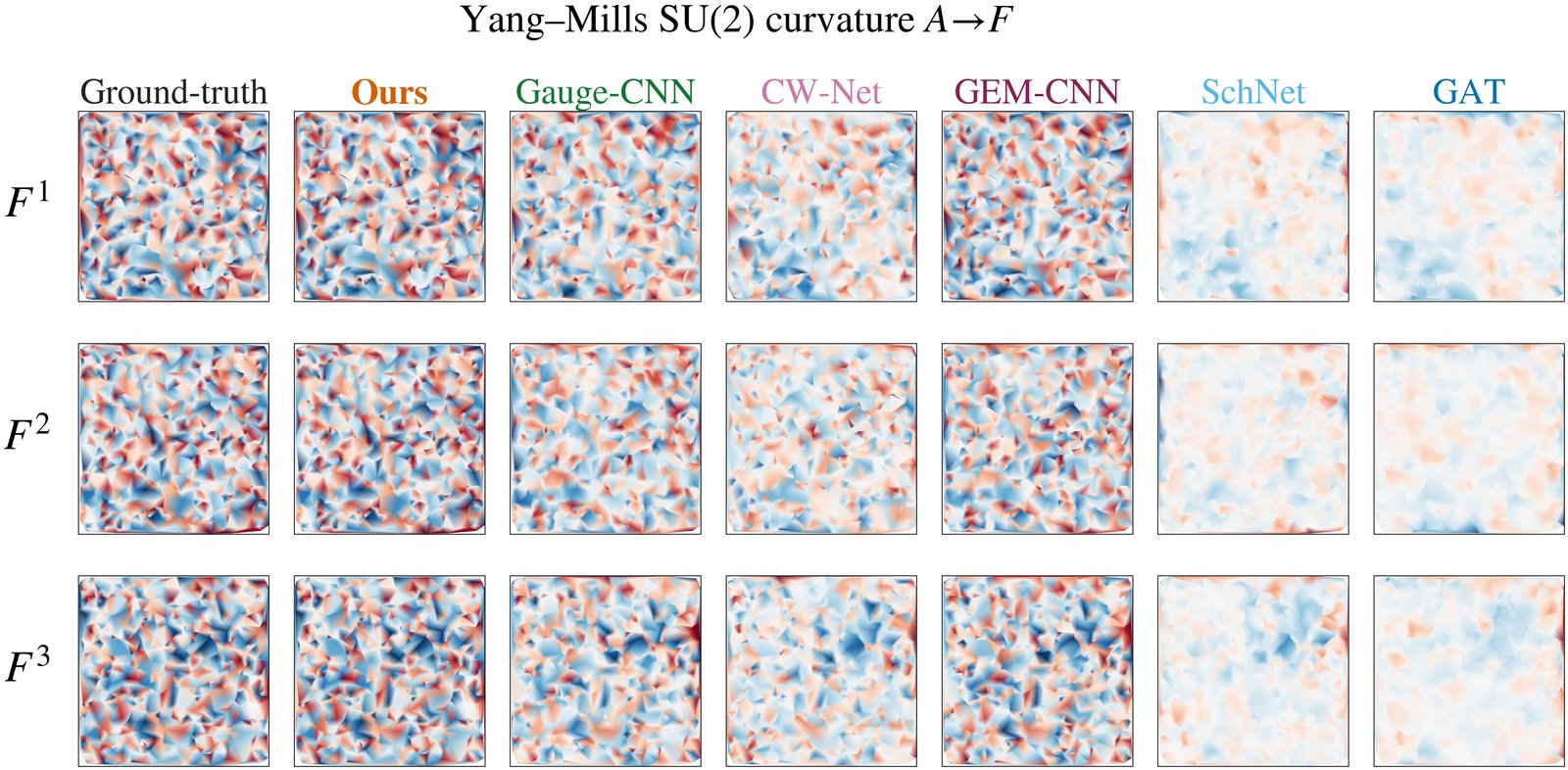}
  \caption{Sample 4.}
\end{subfigure}
\caption{Additional qualitative samples on $SU(2)$ Yang--Mills curvature.
Rows display the three Lie-algebra components $F^1, F^2, F^3$ of
$F = dA + [A,A]$; in these examples only RHMP captures both the differential
$dA$ structure and the algebraic $[A,A]$ coupling at the correct
amplitude.}
\label{fig:qual_app_ym}
\end{figure}

\begin{figure}[H]
\centering
\begin{subfigure}{0.49\textwidth}\centering
  \includegraphics[width=\linewidth]{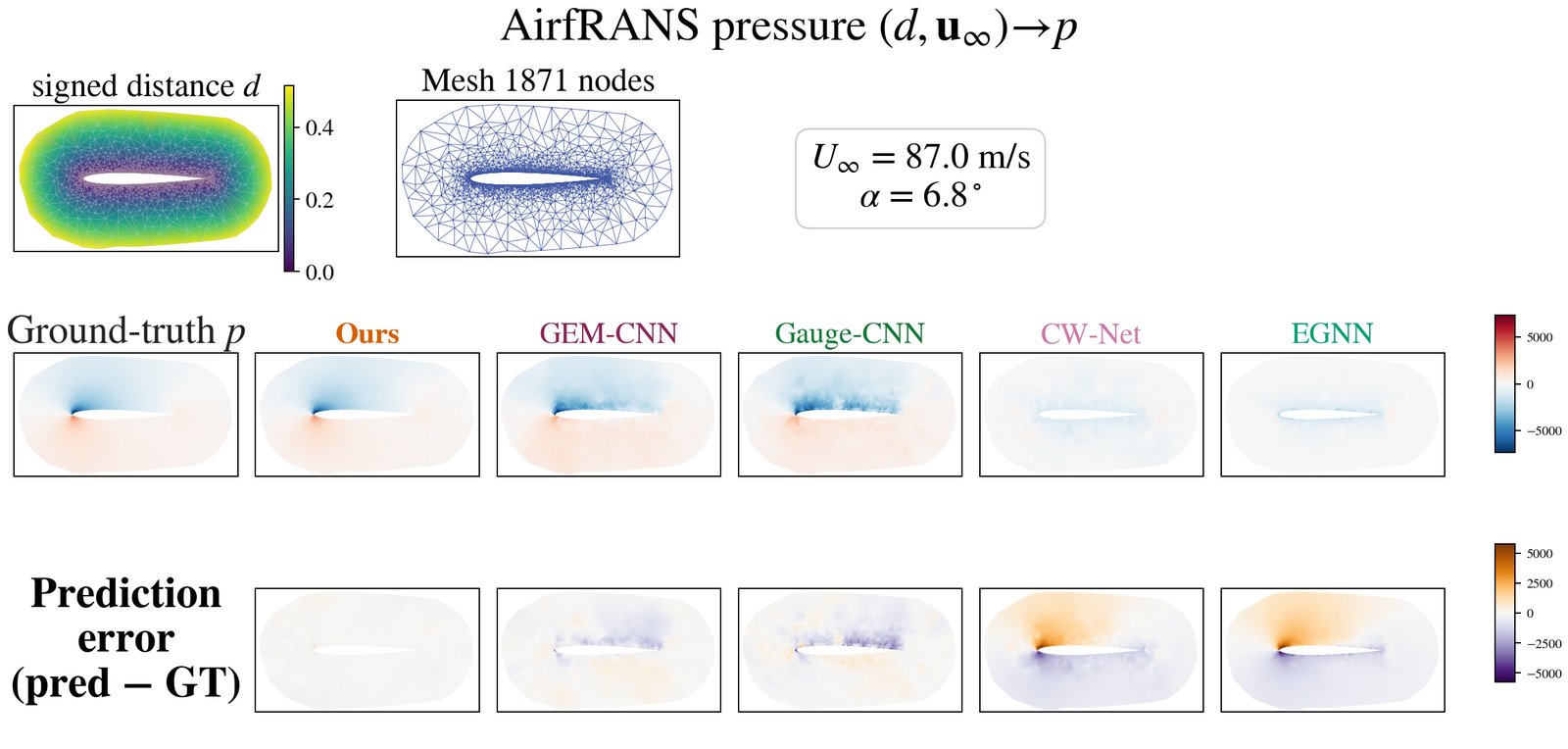}
  \caption{Sample 1, 1.9K nodes.}
\end{subfigure}\hfill
\begin{subfigure}{0.49\textwidth}\centering
  \includegraphics[width=\linewidth]{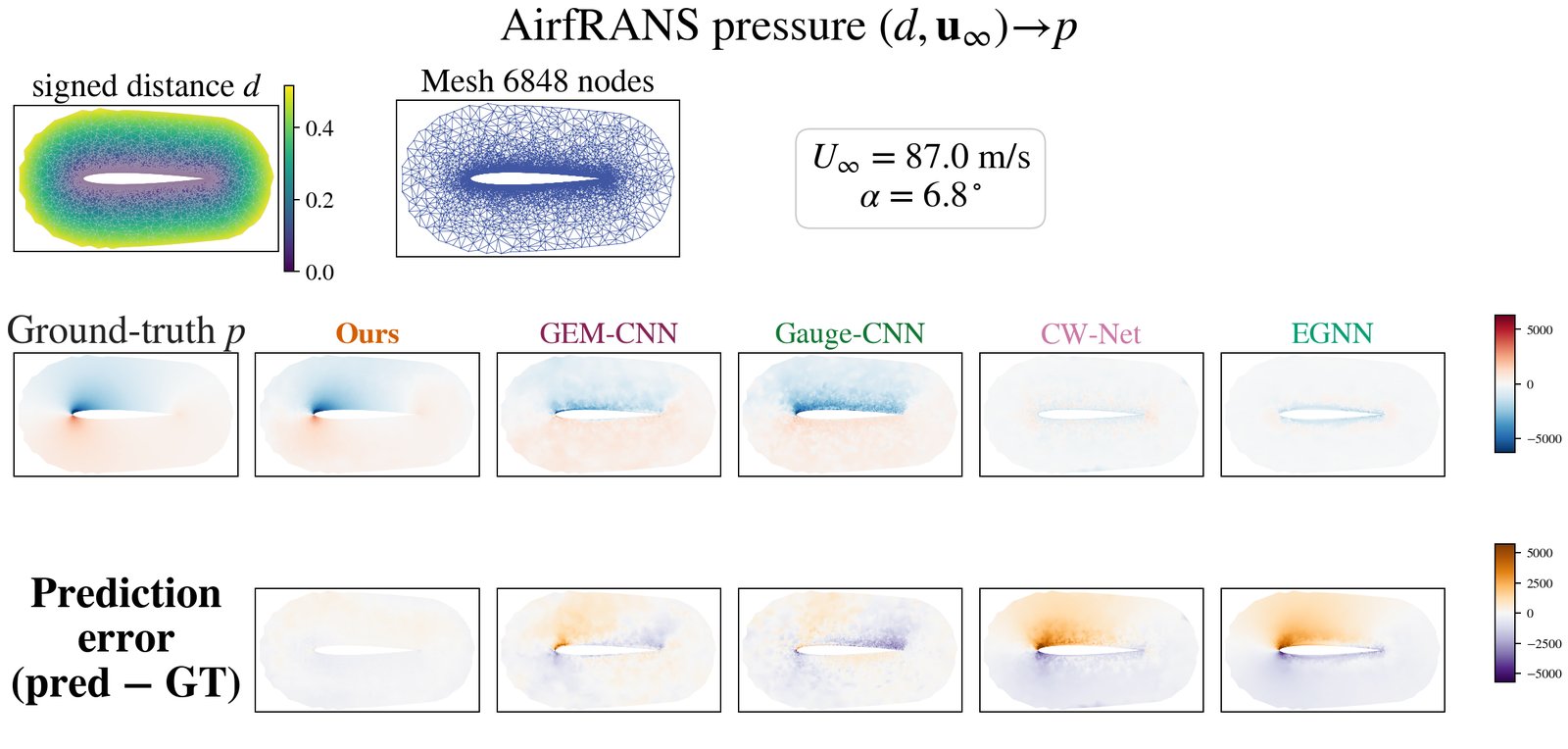}
  \caption{Sample 1, 6.9K nodes.}
\end{subfigure}

\vspace{0.4em}
\begin{subfigure}{0.49\textwidth}\centering
  \includegraphics[width=\linewidth]{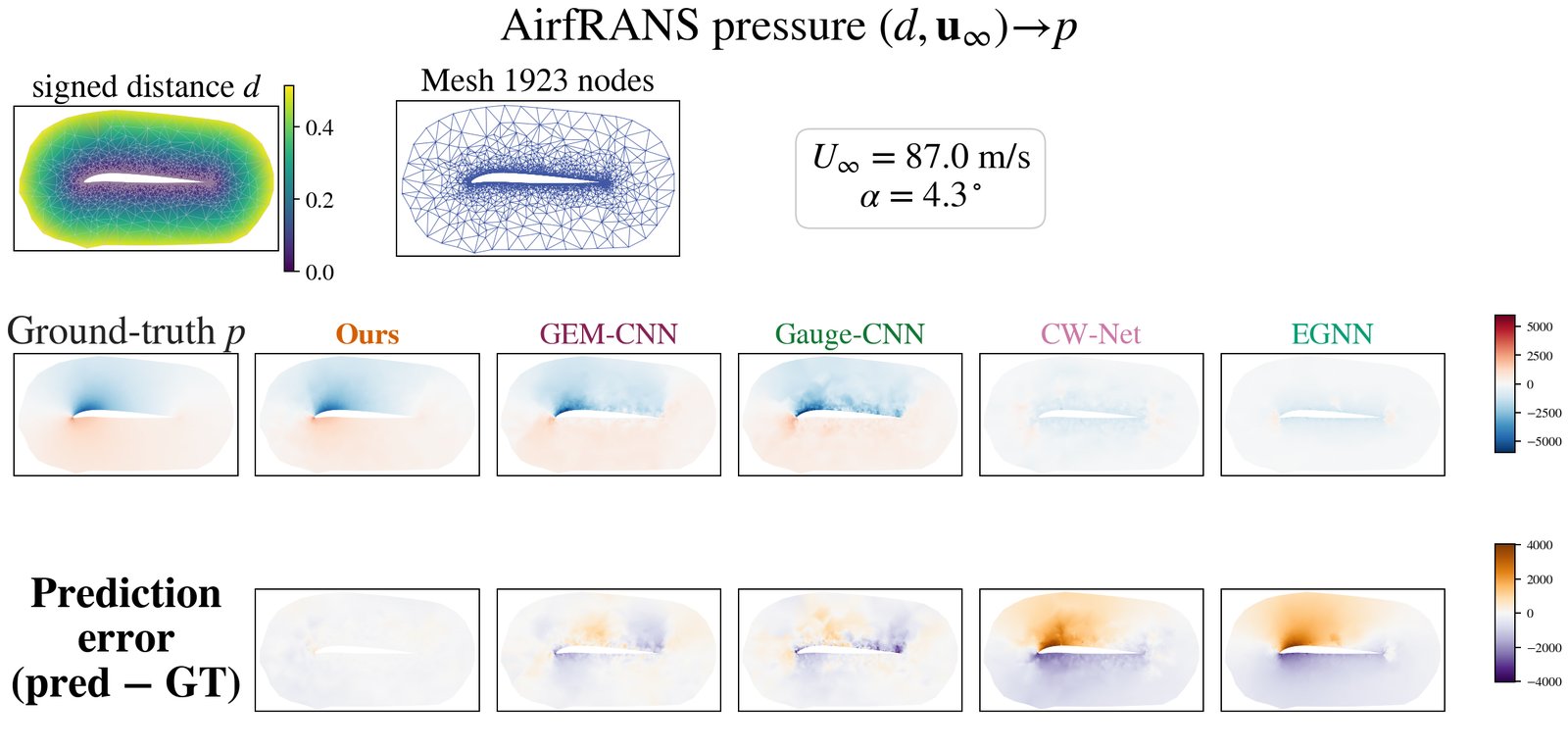}
  \caption{Sample 2, 1.9K nodes.}
\end{subfigure}\hfill
\begin{subfigure}{0.49\textwidth}\centering
  \includegraphics[width=\linewidth]{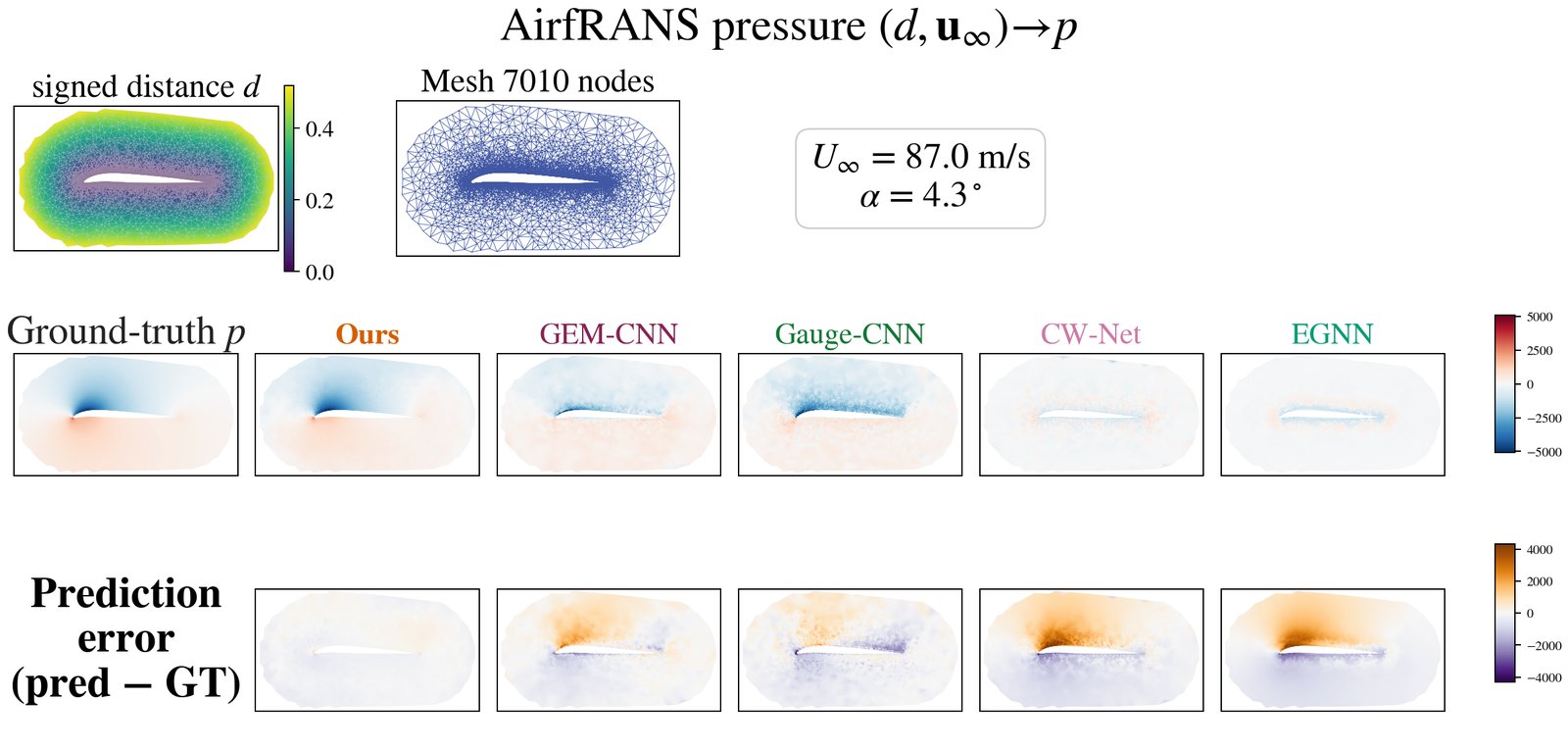}
  \caption{Sample 2, 6.9K nodes.}
\end{subfigure}
\caption{Additional qualitative samples on AirfRANS pressure. Each test simulation is shown on its
1.9K-node mesh (training size) and on its 6.9K-node mesh without retraining; rows display the
inputs, the pressure, and the prediction error.}
\label{fig:qual_app_airfoil}
\end{figure}

\clearpage



\end{document}